\documentclass{article}

 \usepackage[preprint,nonatbib]{neurips_2019}

\usepackage[utf8]{inputenc} % allow utf-8 input
\usepackage[T1]{fontenc}    % use 8-bit T1 fonts
\usepackage{hyperref}       % hyperlinks
\usepackage{url}            % simple URL typesetting
\usepackage{booktabs}       % professional-quality tables
\usepackage{amsfonts}       % blackboard math symbols
\usepackage{nicefrac}       % compact symbols for 1/2, etc.
\usepackage{microtype}      % microtypography

\usepackage{enumitem}
\usepackage{filecontents}
\usepackage{todonotes}
\usepackage{subcaption}
\usepackage{algorithm}
\usepackage[noend]{algorithmic}
\usepackage{mathtools, amsmath,amssymb, amsthm}

\usepackage{amsfonts,bm,dsfont,color}

\usepackage[numbers]{natbib}

\newtheorem{theorem}{Theorem}

\newtheorem{lemma}{Lemma}
\newtheorem{definition}{Definition}
\newtheorem{remark}{Remark}

\def\H{\mathcal{H}}

\def\calR{\mathcal{R}}

\def\I{\mathcal{I}}
\def\Q{\mathcal{Q}}
\def\S{\mathcal{S}}

\def\E{\mathbb{E}}
\def\1{\mathbf{1}}
\def\P{\mathbb{P}}
\def\R{\mathbb{R}}
\def\N{\mathbb{N}}

\newcommand{\mc}[1]{\mathcal{#1}}
\newcommand{\mb}[1]{\mathbb{#1}}

\def\bestH{\H_{\mathrm{best}}}
\def\fdrH{\H_{\mathrm{FDR}}}
\def\fdrHt{\tilde{\H}_{\mathrm{FDR}}}
\def\fwerH{\H_{\mathrm{FWER}}}

\title{The True Sample Complexity of Identifying Good Arms}

\author{%
  Julian Katz-Samuels \\
  Department of EECS \\
  University of Michigan \\
  Ann Arbor, MI  \\
 \texttt{jkatzsam@umich.edu} \\
   \And
Kevin Jamieson \\
Allen School of Computer Science \\
University of Washington \\
Seattle, WA \\
\texttt{jamieson@cs.washington.edu} \\
}

\begin{document}

\maketitle

\begin{abstract}
We consider two multi-armed bandit problems with $n$ arms: \emph{(i)} given an $\epsilon > 0$, identify an arm with mean that is within $\epsilon$ of the largest mean and \emph{(ii)} given a threshold $\mu_0$ and integer $k$, identify $k$ arms with means larger than $\mu_0$. Existing lower bounds and algorithms for the PAC framework suggest that both of these problems require $\Omega(n)$ samples. However, we argue that these definitions not only conflict with how these algorithms are used in practice, but also that these results disagree with intuition that says \emph{(i)} requires only $\Theta(\frac{n}{m})$ samples where $m =  |\{ i : \mu_i > \max_{i \in [n]} \mu_i - \epsilon\}|$ and \emph{(ii)} requires $\Theta(\frac{n}{m}k)$ samples where $m =  |\{ i : \mu_i >  \mu_0 \}|$. We provide definitions that formalize these intuitions, obtain lower bounds that match the above sample complexities, and develop explicit, practical algorithms that achieve nearly matching upper bounds.
\end{abstract}

\section{Introduction}

Collecting data sequentially and adaptively, so that the decision of what to measure next is based on all previous observations, is a powerful paradigm in which the same statistically significant conclusions can be made using far fewer total measurements than using a pre-defined experimental design. 
The pertinent metric of an adaptive data collection algorithm is its  sample complexity: the total number of measurements a procedure must make in order to achieve some objective with high probability.    
This paper is interested in contrasting \emph{verifiable} versus \emph{unverifiable} sample complexity, which we explain through an example.

% One of the first examples in the literature that studies (un)verifiable sample complexity
% The limitations of a definition like PAC that requires \emph{verifying} that the correct hypothesis is returned has been observed before in the active binary classification literature \cite{balcan2010true}.
% To illustrate, consider a ``noiseless'' bandit game where sampling the $i$th arm returns a deterministic value $\mu_i \in \{0,1\}$.
Consider a simple binary classification problem in one-dimension where we can query (i.e. measure) the locations $i \in \{1,\dots,n\}$ and observe $\mu_i \in \{0,1\}$.
The objective is to estimate the vector $\bm{\mu} \coloneqq (\mu_j)_{j \leq n}$.
Suppose that there exists $m \in \{1,\ldots, n\}$ such that $\mu_i=1$ for $i \leq m$ and $\mu_i =0$ for $i > m$. If the player knew that the $\mu_i$ were non-increasing but just not the number $m$, then binary search could identify the index of the transition (and uniquely determine $\bm{\mu}$) with just $\log_2(n)$ total measurements.
However, if the player \emph{didn't} know they were non-increasing, then to \emph{verify} that the first $m$ were indeed $1$ and the last $n-m$ means were $0$ would require $n$ samples.
If after each measurement, the player outputs her best guess for $\bm{\mu}$, then she could \emph{unverifiably} output the correct $\bm{\mu}$ after just $\log_2(n)$ measurements by performing binary search under a monotonicity assumption, but not be able to \emph{verify} that $\bm{\mu}$ was truly non-increasing with a certificate until $n$ measurements were taken. 
This example, inspired by \cite{balcan2010true}, generalizes beyond classification.
% The idea is that the player could run binary search, and in the favorable case in which the first $m$ locations were equal to $1$ and the remaining $0$, the player could start outputting the true value of $\mu \in \{0,1\}^n$ after just $\log_2(n)$ measurements, but not be able to \emph{verify} that $\mu$ was truly non-increasing it with a certificate until $n$ measurements were taken.
% Informally, the \emph{unverifiable} sample complexity would be the minimum time in which the agent started outputting the correct value of $\mu$, in this case after $\log(n)$ measurements, whereas the \emph{verifiable} sample complexity would be the minimum time in which she can verify that the estimate for $\mu$ is correct, in this case $n$.

We believe algorithms for adaptive data collection should be designed to optimize for \emph{unverifiable} sample complexity so that they can take advantage of favorable scenarios. 
While this difference between verifiable and unverifiable sample complexity exists in classification, regression, and general reinforcement learning, in this paper we choose to study an instance of multi-armed bandits that exemplifies the difference between verifiable and unverifiable sample complexity: $\epsilon$-good arm identification.
We propose a novel definition of unverifiable sample complexity and prove upper and lower bounds on the quantities of interest. As a corollary, we obtain results for the intimately related problem of identifying arms with means above a threshold.

\subsection{Multi-armed bandits}
Define a multi-armed \emph{bandit instance} $\rho$ as a collection of $n$ distributions over $\R$ where the $j$th sample from the $i$th distribution $\rho_i$ is an iid random variable $X_{i,j} \sim \rho_i$ with $\E[X_{i,j}] = \mu_i$. 
At round $t \in \mathbb{N}$ a player selects\footnote{We also say that a player pulls arm $I_t$.} an index $I_t \in [n]\coloneqq \{1,\dots,n\}$, immediately observes $X_{I_t,t}$, and then outputs a set $\widehat{S}_t \subseteq [n]$ before the next round. Formally, defining the filtration $(\mc{F}_t)_{t \in \N}$ where $\mc{F}_t = \{ (I_s, X_{I_s,s},\widehat{S}_s) : 1 \leq s \leq t \}$, we require that $I_t$ is $\mc{F}_{t-1}$ measurable while $\widehat{S}_t$ is $\mc{F}_t$ measurable, each with possibly additional external sources of randomness. 
% \footnote{With respect to the filtration $(\mc{F}_t)_{t \in \N}$ where $\mc{F}_t = \{ (I_s, X_{I_s,s},\widehat{S}_s) : 1 \leq s \leq t \}$, $I_t$ is $\mc{F}_{t-1}$ measurable while $X_{I_s,s}$ and $\widehat{S}_t$ are $\mc{F}_t$ measurable, each with possibly additional external sources of randomness.}

% Informed by past observations $\mc{F}_{t-1}$ up to time $t-1$, 
The player strategically chooses an arm $I_t$ at each time $t$ in order to accomplish a goal for $\widehat{S}_t$ as quickly as possible. Two important goals that arise in this setting are \emph{(i)} identifying an arm with the largest mean (commonly referred to as best arm identification) and \emph{(ii)} identifying all of the arms with means above a given threshold $\mu_0 \in \mathbb{R}$. Applications of \emph{(i)} include drug or material design in the presence of noisy experiments. Applications of \emph{(ii)} include genetic screens where individual genes are inhibited to infer a causal relationship with a particular phenotype; typically multiple genes are identified as influencing the phenotype. 
% Unfortunately, when the number of arms is large, solving these problems may require an impractical number of samples. 
% In practice, it is often sufficient to solve relaxations of these problems. 
% For example, it may be satisfactory for a material scientist to obtain merely a nearly-optimal material, and similarly, a sceientist may be content to discover only a subset of the important genes that influence a phenotype.
% For example, if it enables significant savings, identifying a nearly optimal material design is satisfactory  and, similarly, a scientist is typically content to discover a few of the important genes for a phenotype. 
% These relaxations can be formalized as follows:
In practice, one is often willing to trade the ``best'' for ``satisfactory'' if it means a smaller sample complexity. Define: 
\begin{enumerate}[leftmargin=.25in]
  \item \textbf{Identifying an $\boldsymbol{\epsilon}$-good mean}: for a given $\epsilon > 0$, minimize $\tau$ such that the index $\widehat{S}_t \in [n]$ satisfies $\mu_{\widehat{S}_t} > \max_{i \in [n]} \mu_i - \epsilon$ for all $t \geq \tau$ with high probability.
	\item \textbf{Identifying means above a threshold $\boldsymbol{\mu_0}$}: for a given threshold $\mu_0 \in \R$ and $k \in [n]$, minimize $\tau_k$ such that the set $\widehat{S}_t \subseteq [n]$ satisfies $|\widehat{S}_t \cap \{ i: \mu_i > \mu_0 \}| \geq \min(k,|\{i : \mu_i > \mu_0\}|)$ for every $t \geq \tau_k$ subject to $\widehat{S}_s \cap \{ i: \mu_i \leq \mu_0 \} = \emptyset$ for all times $s$ with high probability\footnote{The constraint $\widehat{S}_s \cap \{ i: \mu_i \leq \mu_0 \} = \emptyset$ is known as a family-wise error rate (FWER) condition. We will also consider a more relaxed condition known as false discovery rate (FDR) which controls  $\E[ |\widehat{S}_s \cap \{ i: \mu_i \leq \mu_0 \}| / |\widehat{S}_s| ]$}.
\end{enumerate}  
Note that in the second problem we require $\widehat{S}_t \subseteq \{i : \mu_i > \mu_0 \}$ for all times $t$ with high probability, whereas in the first problem we allow mistakes, $\widehat{S}_t \notin \{i : \mu_i > \mu_1 - \epsilon\}$, for some times $t$.

\textbf{Why study both objectives simultaneously?} 
Our proposed algorithms for each objective are extremely similar, and the fundamental difficulty of the objectives are closely related:
for a fixed set of means $\mu_1 \geq \dots \geq \mu_n$ and  any threshold $\mu_0$ there exists an $\epsilon = \mu_1 - \mu_0$ so that $\{\mu_i : \mu_i > \mu_1 - \epsilon\} = \{\mu_i : \mu_i > \mu_0\}$.
Thus, identifying $k$ arms above the threshold $\mu_0$ is equivalent to identifying $k$ $\epsilon$-good means for $\epsilon=\mu_1-\mu_0$. 
Consequently, if $m = |\{ i \in [n] : \mu_i > \mu_1 - \epsilon\}|$ then we can study \emph{lower bounds} on the sample complexity of both problems simultaneously by considering the necessary number of samples required to identify $k$ of the $m$ largest means (i.e.,  to have $\widehat{S}_t \subset [m]$ with $|\widehat{S}_t| = k$) for any value of $1 \leq k \leq m$. 
While considering $m$ is helpful for analysis, it should be stressed that ${m}$\emph{ is never known to the algorithm and must be adapted to} and, in fact, such knowledge would significantly simplify this problem \cite{chaudhuri2019pac}.

% We will define formally what a successful outcome is shortly. 
% These objectives will be made well-defined shortly, but suffice for now to illustrate a point. 
 % are not well-defined for now to illustrate a point, and formally defined shortly.
% Nevertheless, the objectives are very related and, in fact, our proposed algorithm is essentially identical for solving both.
% Thus, the following discussion focuses the first objective, but will address the second objective in due time. 

% If a child searches through $n$ Lego blocks looking for a particular block, she intuits that if there are $m$ blocks of that type, she will come across the desired block approximately $m$ faster than identifying a one-of-a-kind block. 
\textbf{Intuition for unverifiable sample complexity.} Verifiably identifying an $\epsilon$-good arm requires a sample from every arm because some unsampled arm may be much better than the sampled arms (made formal below), but unverifiably identifying an $\epsilon$-good arm does not necessarily require a sample from every arm. If there is just a single $\epsilon$-good distribution out of the total $n$ so that $m=1$, then any strategy will have to sample from about $n$ distributions before coming across the unique $\epsilon$-good distribution, and by sampling every arm we can also \emph{verify} that the suggested arm is $\epsilon$-good. 
However, if there are $1 < m \leq n$ means that are $\epsilon$-good, then if at least $n/m$ indices are drawn uniformly at random from $\{1,\dots,n\}$ then with constant probability at least one of them will be $\epsilon$-good.
Thus, when there are $m$ $\epsilon$-good distributions, one should expect the number of total measurements to identify an $\epsilon$-good distribution, at least unverifiably, to scale as $n/m$, not $n$.
In an extreme case, if $m = n/2$ so that half the distributions are $\epsilon$-good, then one should expect the number of samples to identify an $\epsilon$-good distribution to be \emph{constant} with respect to $n$.
The same argument applies to identifying arms with means above a threshold: if there are $m$ means above the threshold $\mu_0$, then one would expect that the number of samples required to identify at least $1 \leq k \leq m$ of them scales like $k \tfrac{n}{m}$, not $n$. 
Unfortunately, with few exceptions the literature has only been interested in verifiable sample complexity, and consequently, have sample complexities that scale with $n$, not $n/m$, and therefore disagree with the intuition of above.
% In the case of identifying arms above a threshold, only the case of identifying all $k=m$ arms above the threshold has been addressed, ignoring the reality that only rarely is it possible to identify all of $m$ as they require often impractical $\Omega(n)$ sample complexities.
% In the case of identifying an $\epsilon$-good arm, using the definitions accepted hitherto by the community, one can prove information theoretic lower bounds that conflict with the above intuition, as we discuss below.
This paper aims to develop definitions, algorithms, and lower bounds that confirm the necessary and sufficient conditions for obtaining the intuitive sample complexities expected in the two problems of interest.

\subsection{Revisiting identifying an $\epsilon$-good arm: an unverifiable sample complexity perspective}\label{sec:pac_defs}
Many works have studied the problem of identifying an $\epsilon$-good arm with high probability.
We begin by considering the standard definition under which verifiable sample complexity of identifying an $\epsilon$-good arm is characterized.
For an instance $\rho =(\rho_1,\dots,\rho_n)$ recall $\mu_i = \E_{X \sim \rho_i}[X]$.

\begin{definition}
Fix an algorithm $\mc{A} \equiv (I_t, \widehat{S}_t, \tau_{PAC})$ where $\tau_{PAC}$ is a stopping time with respect to the filtration $(\mc{F}_t)_{t \in \N}$. 
Then $\mc{A}$ is $\boldsymbol{(\epsilon,\delta)}$\textbf{-PAC (Probably Approximately Correct)} if $\forall \rho \in \mc{P}$ $\mc{A}$ terminates at $\tau_{PAC}$ and $\P_\rho( \mu_{\widehat{S}_{\tau_{PAC}}} \geq \max_i \mu_i - \epsilon) \geq 1-\delta$.
\end{definition}
This definition exemplifies verifiable sample complexity because it requires that the algorithm terminate and declare that the arm $\widehat{S}_{\tau_{PAC}}$ satisfies $\mu_{\widehat{S}_{\tau_{PAC}}} \geq \max_i \mu_i - \epsilon$.
One typically takes $\mc{P}$ to be all sub-Gaussian tailed distributions. 
For a given $\epsilon, \delta$, instance $\rho$, and $\mc{P} = \{ \mc{N}(\mu', \sigma^2 I) : \mu' \in \R^n \}$, one can show $\E_\rho[ \tau_{PAC} ] \gtrsim \log(1/\delta)[ \epsilon^{-2} m + \sum_{i=m+1}^n (\mu_1 - \mu_i)^{-2}]$ for any $(\epsilon,\delta)$-PAC algorithm \cite{kaufmann2016complexity,Mannor04thesample} (see Appendix \ref{sec:kaufmann_results} for a formal statement). 
That is, \emph{any ${(\epsilon,\delta)}$-PAC algorithm has an expected sample complexity $\E[\tau_{PAC}]$ of at least ${n}$, regardless of $m$} the number of $\epsilon$-good distributions among the $n$.
This is necessary because an $(\epsilon,\delta)$-PAC algorithm must prove that any output arm $\widehat{S}_{\tau_{PAC}}$ satisfies $\mu_{\widehat{S}_{\tau_{PAC}}} \geq \max_i \mu_i - \epsilon$, but the value of $\max_i \mu_i$ is not known a priori, so the algorithm must pull \emph{every} arm at least once to \emph{verify} that $\widehat{S}_{\tau_{PAC}}$ is indeed $\epsilon$-good
% \footnote{We note that if $\mc{P}$ is structured enough like in linear bandits, it is possible to identify and verify an $\epsilon$-good arm with a sample complexity that does not scale linearly with $n$.}
.
Contrast this with the above discussion where we were merely concerned with how quickly an algorithm could start outputting an $\epsilon$-good arm, with no condition of \emph{verifying} that it is $\epsilon$-good.
% While this result is counter-intuitive given the above discussion, it is not contradictory because above we merely asked how many samples are necessary before an algorithm could potentially start suggesting an $\epsilon$-good arm, while an $(\epsilon,\delta)$-PAC algorithm on $\mc{P}$ is required not only to suggest an $\epsilon$-good arm $\widehat{i} \in [n]$, but also provide a certificate that there exists no other arm $i\neq \widehat{i}$ such that $\mu_i > \mu_{\widehat{i}} + \epsilon$.
%
% It is this last qualification of \emph{verifying} that the suggested arm is $\epsilon$-good that requires a sample complexity linear in the number of arms $n$ because if there existed an arm that was never sampled, the algorithm could never certify that that arm was not actually the only $\epsilon$-good arm
% However, we note that sometimes $\mc{P}$ can be small enough so that pulling some arms can rule out hypotheses affecting other arms.
% For example, if $\mc{P}$ was just all permutations of a single known instance, then the value of $\mu_1$ would be known and this problem would become considerably easier since given any arm with mean $\mu_j$, one could verify with just this knowledge and samples from arm $j$ that $\mu_j > \mu_1 - \epsilon$.
% But in general, of course, the value of $\mu_1$ is not known to the algorithm which makes it impossible for an algorithm to be $(\epsilon,\delta)$-PAC for any non-trivial $\mc{P}$ without sampling every arm at least once.  
We now propose an alternative definition:
\begin{definition}
Fix an algorithm $\mc{A} \equiv (I_t, \widehat{S}_t)$. 
Then $\mc{A}$ is $\boldsymbol{(\epsilon,\delta)}$\textbf{-SimplePAC} if $\forall \rho \in \mc{P}$ there exists a stopping time $\tau_{simple, \rho}$ with respect to the filtration $(\mc{F}_t)_{t \in \N}$ such that $\P_\rho( \forall t \geq \tau_{simple,\rho}: \ \   \mu_{\widehat{S}_{t}} \geq \max_i \mu_i - \epsilon) \geq 1-\delta$.
\end{definition}
We emphasize $\tau_{simple, \rho}$ is for \emph{analysis purposes only} and \emph{is unknown to the algorithm}, and thus, the algorithm never terminates, stops taking samples, or recommending sets $\widehat{S}_t$.
The critical difference between the two definitions is that a {PAC} algorithm must declare when it has found an $\epsilon$-good arm, whereas a SimplePAC algorithm just needs to start outputting an $\epsilon$-good arm eventually.
Analogous to above, for a given $\mc{P}$ and $\rho \in \mc{P}$ we prove a lower bound on $\E_\rho[ \tau_{simple,\rho} ]$ for any $(\epsilon,\delta)$-SimplePAC algorithm.
Clearly, if an algorithm is $(\epsilon,\delta)$-PAC for an instance $\rho$ then it is also $(\epsilon,\delta)$-SimplePAC for $\rho$ since we may take $\tau_{simple,\rho} = \tau_{PAC}$ and output the arm identified at $\tau_{PAC}$ at all $t \geq \tau_{PAC}$.
However, as the above discussion suggests, $\tau_{simple,\rho}$ may be significantly smaller than $\tau_{PAC}$, even as small as $\tau_{simple,\rho}=O(1)$ while $\tau_{PAC} = \Omega(n)$. We have written $\tau_{simple,\rho}$ to emphasize that this stopping time may be a function of the underlying instance $\rho$, but in the interest of brevity we will often write $\tau_{simple}$ when the context makes it clear. 
 % since the former is just a time after which an algorithm only outputs $\epsilon$-good arms, and the latter is when the algorithm provides an additional guarantee that that recommended arm is actually $\epsilon$-good. 
% Indeed, when $m=n/2$ then we will see examples where $\tau_{simple}=O(1)$ while $\tau_{PAC} = \Omega(n)$. 
% To define such a stopping time $\tau_{simple}$, we will use knowledge of the behavior of the algorithm (e.g., how $\widehat{S}_t$ is constructed) and of the true means of the problem instance. 
% In particular, we will essentially define $\tau_{simple}$ to be the first time that some $\epsilon$-good arm $i$ is pulled on the order of $( \mu_1 - \epsilon - \mu_i)^{-2}$ times and we will show that whp for all $t \geq \tau_{simple}$ the algorithm only outputs $\epsilon$-good arms.
% The next section provides real-world case studies that illustrate how algorithms designed for PAC are inferior in practice relative to our proposed algorithm that is designed for SimplePAC.

\subsection{Motivating Examples}
% There are real-world consequences of the community accepting a definition like PAC over SimplePAC.
% Lower bounds play a critical role in a algorithmic field's progress by providing researchers with permission to stop searching for a better algorithm once one is found that nearly matches the lower bound. 
% However, if the definitions or assumptions under which lower bounds were proven are inconsistent with how these algorithms are applied in practice, using a near-optimal algorithm for a poor definition of optimality could have harmful real-world consequences. 
% We very briefly compare state-of-the-art algorithms to those proposed in this paper.

\noindent\textbf{Identifying a good arm.} The LUCB algorithm of \cite{DBLP:conf/icml/KalyanakrishnanTAS12} is an $(\epsilon,\delta)$-PAC algorithm whose sample complexity is within $\log(n)$ of the lower bound of any $(\epsilon,\delta)$-PAC algorithm and is known to have excellent empirical performance \cite{10.1109/CISS.2014.6814096}.
LUCB does not use $\epsilon$ as a sampling rule (only a stopping condition), and thus can be evaluated after any number of pulls using its empirical best arm. 
We compare its performance to our algorithm, denoted ``ILUCB"  (Infinite LUCB), designed for the SimplePAC definition.
We obtain a realistic bandit instance of 9061 Bernoulli arms with parameters defined by the empirical means from a recent crowd-sourced \emph{New Yorker Magazine} Caption Contest, where each caption was shown uniformly at random to a participant, and received on average 155 votes of funny/unfunny (see Appendix~\ref{sec:experiment_details} for details).
% Thus, in our replayed experiments, each sample from the $i$th arm corresponds to an individual rating of the $i$th caption. 
We run LUCB and ILUCB for 3 million iterations. Figure~1 depicts $\tau_{simple}$ for LUCB and ILUCB, where for a given $\epsilon > 0$, $\tau_{simple}$ is the last time that a non-$\epsilon$-good arm is outputted. We observe that our proposed algorithm latches onto the $\epsilon$-good arms at a rate that is several orders of magnitude better than LUCB for a large range of values of $\epsilon$. In addition, LUCB requires several million samples to provide a certificate for $\epsilon \approx 0.32$--an impractical sample complexity for an arm that is still very far from optimal. 
%Our theoretical results are actually derived for an analogous UCB algorithm (also shown for comparison) but our theory still holds for LUCB because LUCB pulls the UCB and just one additional arm per round. 
% We use an LUCB variant of our algorithm in this experiment. Whereas UCB algorithms pull an arm that maximizes an upper confidence bound on its true mean, LUCB algorithms pull this maximizer as well as an additional arm at each round.  Although our theory is developed for the UCB variant, it also applies to the LUCB variant since LUCB is essentially UCB with an additional arm pull at each round. Figure~2 depicts the performance of the UCB variant of our algorithm against UCB on the same dataset.

\begin{figure*}[tb]
  \begin{minipage}[b]{.45\linewidth}
    \centering
    \includegraphics[width=\linewidth]{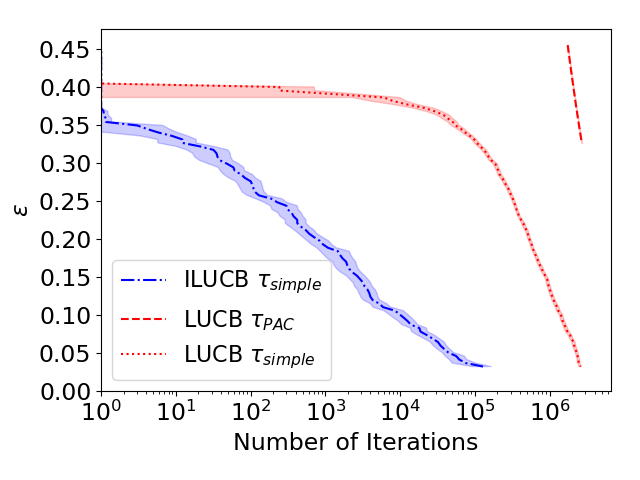}
    \vspace*{-5mm}
    \caption{$\epsilon$-good arm identification (LUCB)}
  \end{minipage}
  \hspace{1em}
%  \begin{minipage}[b]{.31\linewidth}
%    \centering
%    \includegraphics[width=\linewidth]{./figures/new_yorker_1_UCB.png}
%    \vspace*{-5mm}
%    \caption{$\epsilon$-good arm identification (UCB)}
%  \end{minipage}
  \hspace{1em}
  \begin{minipage}[b]{.45\linewidth}
    \centering
    \includegraphics[width=\textwidth]{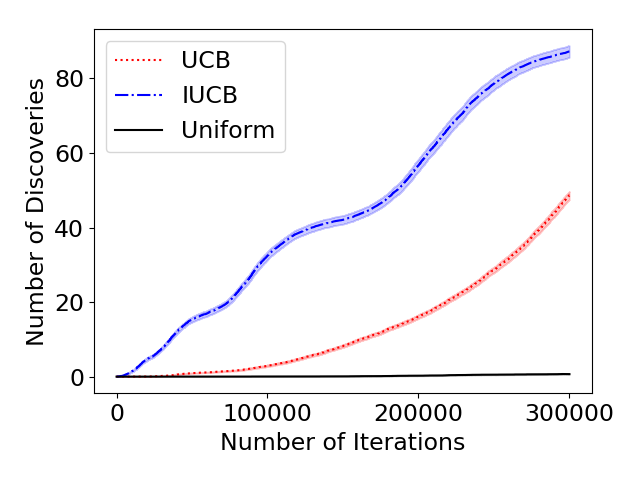}
    \vspace*{-5mm}
    \caption{Identifying means above a threshold}
  \end{minipage}
\end{figure*}

\noindent\textbf{Multiple identifications of arms above a threshold.}
We now turn our attention to multiple testing, a stalwart of science. 
The recent work \cite{jamieson2018bandit} proposed an algorithm that identifies nearly all $m$ arms above a threshold in a number of samples that is nearly optimal, but has a sample complexity that scales with $n$.
However, identifying just $k < m$ arms may require substantially fewer samples and is often sufficient in practice.
% However, in seeking to identify all $m$ arms above the threshold, it may use more samples than is necessary to find just $k \leq m$ arms and indeed, its sample complexity will always scale with $n$. This shortcoming is problematic for applications where samples are expensive and it is satisfactory to find a subset of the discoveries, especially if this can be done with substantially fewer samples.
Consider the experimental data of \cite{hao2008drosophila}, which aimed to discover genes in Drosophila that inhibit virus replication. 
Starting with 13,071 genes \cite{hao2008drosophila} measured each gene multiple times using an adaptive data collection procedure, ultimately taking about $38000$ measurements.
 % twice and eliminated all but the $1000$ most extreme observations, and then measured each of these $1000$ genes $12$ times. Finally, they reported the $100$ genes that were statistically significant of these $1000$ genes measured $12$ times; a total of about $38,000$ measurements. 
As was revealed in a meta-study \cite{hao2013limited}, there were inconsistencies in the identified genes when this procedure was duplicated by other labs; this provides strong motivation for sample efficient and provably reliable algorithms. 
% This raises the question: given a finite budget of total experiments, is it better to test many hypotheses with low signal to noise, or a smaller set with higher signal to noise and make more discoveries?
% Current metrics in prioritizing the identification of all $m$ arms often require too many samples to make discoveries with a realistic budget. 
% By contrast, our paper analyzes metrics that encourage the development of algorithms that can make more discoveries under a realistic budget.
Figure 2 depicts a simulation based on plug-in estimates of the experimental data of \cite{hao2008drosophila} (described in Appendix~\ref{sec:experiment_details}) and shows that our algorithm (IUCB) is able to make discoveries much more quickly than the algorithm from \cite{jamieson2018bandit} (UCB).
Our algorithm leverages the observation that to identify $k$ arms above the threshold $\mu_0$, it suffices to consider $\Theta(\frac{kn}{m})$ arms chosen uniformly at random, where $m = | \{ i : \mu_i > \mu_0\}|$. See Appendix \ref{sec:experiment_details} for more details on the experiments. 
% Since $m=360$,  we only expect gains for small $k$ (e.g., $k\leq 50$) and this explains why our algorithm is eventually overtaken by UCB. However, UCB  outperforms the algorithms from our paper only after around 75,000 samples and in many settings, this budget is far too large to be practical. 

\subsection{Related work}

In addition to the lower bounds for the $(\epsilon,\delta)$-PAC setting discussed in Section~\ref{sec:pac_defs} \cite{kaufmann2016complexity,Mannor04thesample}, a related line of work has studied the exact PAC sample complexity in the asymptotic regime as $\delta \rightarrow 0$ \cite{degenne2019pure,garivier2019non}.

Our definition of SimplePAC may be interpreted as a high probability version of the expected \emph{simple regret} metric (c.f. \cite{Bubeckal11}), however, neither definition subsumes the other.
The closest work to our own SimplePAC setting is that of \cite{chaudhuri2017pac,chaudhuri2019pac,aziz2018pure} that also aimed to identify multiple arms, but with the critical difference that $m$ is assumed to be \emph{known}.
Specifically, given a tolerance $\eta \geq 0$, they say an arm $i$ is $(\eta,m)$-optimal if $\mu_i \geq \mu_m - \eta$.
The objective, given $m$ and $\eta$ as inputs to the algorithm, is to identify $k$ $(\eta,m)$-optimal arms with probability at least $1-\delta$. 
The case when $\eta = 0$ and $m=|\{i : \mu_i > \mu_1 - \epsilon \}|$ coincides with our setting, with the critical difference that in our setting the algorithm never has knowledge of $m$. 
With just knowledge of $\epsilon$ but not $m$, as in our setting, there is no guide a priori to how many arms we need to consider in order to get just one $\epsilon$-good arm.
However, still relevant from a lower bound perspective, they prove \emph{worst-case}  results for $\eta >0$. 
In contrast, our work demonstrates instance-specific lower-bounds (i.e., those that depend on the particular means $\mu$) that directly apply to their setting, a contribution of its own.

% , that, unfortunately, does not extend beyond best-arm identification to the more general $\epsilon$-good identification problem.

% But we must pause for a moment and remember, 

% We emphasize that the lower bounds of the above paragraphs for $(\epsilon,\delta)$-PAC for both $\mc{P} = \{ \rho(\mu') : \mu' \in \R^n \}$ and $\mc{P} = \{ \rho(\pi(\mu)) : \pi \in \mb{S}^n \}$ for every problem described are $\Omega(n)$ regardless of the number of $\epsilon$-good arms $m$.
% If we wish algorithms to have sample complexities that scale more like $\Theta(n/m)$ then we need to adopt the $(\epsilon,\delta)$-SimplePAC definition. 

\textbf{Algorithms for $\epsilon$-good identification}
% While for lower bounds we can characterize the difficulty of identifying $k=1$ $\epsilon$-good arm or $1 \leq k \leq m$ arms above a threshold $\mu_0$ with a single unifying problem statement as the difficulty of identifying $k$ of the $m$ largest arms, for upper bounds we require explicit algorithms that differ slightly depending on the problem in question.  
% We first address the problem of identifying an $\epsilon$-good arm (i.e., $k=1$ arms within top $m$). 
For identifying an $\epsilon$-good arm, there have been many $(\epsilon,\delta)$-PAC algorithms proposed over the last few decades \cite{EvenDaral06,DBLP:conf/icml/KalyanakrishnanTAS12,NIPS2012_4640,COLT13,icml2013_karnin13,simchowitz2017simulator,garivier2019non}, the best sample complexity result being $\sum_{i=1}^n \Delta_{i,\epsilon}^{-2} \log( \log(\Delta_{i,\epsilon}^{-2}))$ shown by \cite{icml2013_karnin13} (unfortunately, the algorithm has poor empirical performance and is eclipsed by the LUCB-style algorithms of \cite{DBLP:conf/icml/KalyanakrishnanTAS12,simchowitz2017simulator}).
% The simplest $(\epsilon,\delta)$-PAC algorithm is the one that samples each arm the same number of times, and can identify any number $1 \leq k \leq m$ of $\epsilon$-good arms with just $n \epsilon^{-2} \log(n/\delta)$ total pulls.
% Remarkably, for $k=1$, Median Elimination \cite{EvenDaral06} requires just $n \epsilon^{-2} \log(1/\delta)$ pulls and Successive Elimination \cite{EvenDaral06} just $\sum_{i=1}^n \Delta_{i,\epsilon}^{-2} \log( n \log(\Delta_{i,\epsilon}^{-2}))$ where $\Delta_{i,\epsilon} = \max\{ \max_j \mu_j - \mu_i, \epsilon \}$.
% Later an algorithm was proposed by \cite{icml2013_karnin13} that obtained the best of both worlds: $\sum_{i=1}^n \Delta_{i,\epsilon}^{-2} \log( \log(\Delta_{i,\epsilon}^{-2}))$ (unfortunately, the algorithm has poor empirical performance).
% In the years that followed, for the best-arm identification scenario with $\epsilon=0$, numerous works proposed alternative schemes to \cite{icml2013_karnin13} that achieved equally favorable sample complexity upper bounds (up to constants) but were practically useful \cite{jamieson2014lil}, shaved off $\log\log$ factors \cite{chen2017nearly}, and optimized asymptotic constants \cite{GK16}.
A closely related problem is known as the \emph{infinite armed-bandit problem} where the player has access to an infinite pool of arms such that when a new arm is requested, its mean is drawn iid from a distribution $\nu$.
% Given an infinite armed bandit algorithm that can identify an $\epsilon$-good arm for arbitrary arm distribution $\nu$ could be used, 
In principle, an infinite armed bandit algorithm could solve the problem of interest of this paper by taking $\nu(x) = \tfrac{1}{n} \sum_{i=1}^n \1\{ \mu_i \leq x \}$. 
With the exception of \cite{li2017hyperband}, nearly all of the existing work makes parametric\footnote{For example, for a drawn arm with random mean $\mu$ it is assumed $\P(\mu \leq x) \geq c (x-\mu_*)^\beta$ for some fixed parameters $c, \mu_*, \beta$ that are known (or not).} assumptions about $\nu$ in some way \cite{berry1997,wang2008,CarpentierV15,chandrasekaran2014finding,jamieson2016power}.
 However, the algorithm of \cite{li2017hyperband} was designed for a much more general setting and therefore sacrifices both theoretical and practical performance, and was not designed to take a fixed confidence $\delta$ as input.
 % , thus could waste samples trying to drive its error probability to zero when just $\delta$ sufficed.
% Nevertheless, our proposed algorithms are inspired by Hyperband's hedging strategy of considering different-sized sets of randomly sampled arms. 

\textbf{Algorithms for identifying means above $\mu_0$.}
% As alluded to before, in the special case of $k=m$ this is known as the Top-$k$ problem. 
% However, algorithms for Top-$k$ require $k$ as input, but of course, $m$ is unknown so it is impossible to set $k=m$ a priori, making these algorithm inapplicable in practice.
% Thus, the relevant problem statement here is identifying $1 \leq k \leq m$ arms above the known threshold $\mu_0$.
Maximizing the probability of identifying \emph{every} arm as either above or below a threshold $\mu_0$ given a budget $T$ is known as the \emph{threshold bandit problem} \cite{locatelli2016optimal,mukherjee2017thresholding}.
% Specifically, these works take a time-horizon $T$ and threshold $\mu_0$ as input, and sample arms in such a way as to maximize the probability of classifying all arms correctly above or below the threshold $\mu_0$ (i.e., the $k=m$ setting).
These works explicitly assume no arms are \emph{equal} to $\mu_0$ and penalize incorrectly predicting a mean above or below the threshold equally.
% If an arm's mean is arbitrarily close to, or equal to, $\mu_0$ the required time horizon $T$ of these works to obtain a non-trivial probability of success is unbounded.
For our problem setting, the most related work is \cite{jamieson2018bandit} which proposes an algorithm that takes a confidence $\delta$ and threshold $\mu_0$ as input.
The authors characterizes the total number of samples the algorithm takes before all $k=m$ arms with means above the threshold are output with probability at least $1-\delta$ for all future times, that is, the $k=m$ SimplePAC setting.
While this sample complexity is nearly optimal for the $k=m$ case (witnessed by the lower bounds of \cite{simchowitz2017simulator,DBLP:conf/nips/ChenLKLC14}) this work is silent on the issue of identifying just a subset of size $k \leq m$ means above the threshold (and the algorithm does not generalize to this setting).
% This is a significant gap because if some means above the threshold are much larger than others, they will clearly be detected earlier, and it is not clear that the proposed algorithm identifies any $1 \leq k \leq m$ in a near-optimal way or just all $m$ at a particular time.
%{\color{blue} [cut?] The work of \cite{jamieson2018bandit} also notably relaxed the family-wise error rate (FWER) $\P( \cap_{t=1}^\infty \{ \widehat{S}_t \subseteq \{ i : \mu_i \leq \mu_0 \} \} ) \geq 1-\delta$ to merely a bounded false-discovery rate (FDR): $\max_{t} \E\left[ \frac{| \widehat{S}_t \cap \{ i : \mu_i \leq \mu_0 \}|}{|\widehat{S}_t|} \right] \leq \delta$.
%Using the more relaxed error criterion of FDR versus FWER, it was shown that nearly all $m$ arms above the threshold could be identified substantially faster.}

\section{Lower bounds}
% In the Simple-PAC definition, for a given instance $\rho$ we wish to lower bound $\E_\rho[ \tau_{simple}]$. 
To avoid trivial algorithms that deterministically output an index that happens to be the best arm, we adopt the random permutation model of \cite{simchowitz2017simulator,chen2017nearly}.
We say $\pi \sim \mb{S}^n$ if $\pi$ is drawn uniformly at random from the set of permutations over $[n]$, denoted $\mb{S}^n$.
For any $\pi \in \mb{S}^n$, $\pi(i)$ denotes the index that $i$ is mapped to under $\pi$. 
For a bandit instance $\rho = (\rho_1,\dots,\rho_n)$ let $\pi(\rho) = (\rho_{\pi(1)},\rho_{\pi(2)},\dots,\rho_{\pi(n)})$ so that $\E_{\pi(\rho)}[ T_{\pi(i)}(t) ]$ denotes the expected number of samples taken by the algorithm up to time $t$ from the arm with mean $\mu_i$ when run on instance $\pi(\rho)$. 
% At the start of the game the algorithm is given a permutation $\pi$ of the arms chosen uniformly at random from $\mb{S}^n$ which we denote as $\pi \sim \mb{S}^n$. 
The sample complexity of interest is the expected number of samples taken by the algorithm under $\pi(\rho)$ averaged over all possible $\pi \in \mb{S}^n$.
% , which we denote as $\E_{\pi \sim \mb{S}^n} [ \E_{\pi(\rho)} [ \tau_{Simple}] ] = \E_{\pi \sim \mb{S}^n} [ \E_{\pi(\rho)} [ \sum_{i=1}^n T_{\pi(i)}(\tau_{Simple})] ]$.

As pointed out in the introduction, for any threshold $\mu_0$ there exists an $\epsilon=\mu_1-\mu_0$. 
Thus, if $m=|\{i: \mu_i > \mu_1 - \epsilon\}|$ then a lower bound for identifying an $\epsilon$-good arm or $k$ arms above a threshold $\mu_0$ is implied by a lower bound to identify $k$ arms among the $m$ largest means for any $1 \leq k \leq m$.
% In appendix []\footnote{In fact, for the case when $\mu_1=\dots=\mu_m>\mu_{m+1}=\dots=\mu_n$ and $\mu_{m}-\mu_{m+1} = \epsilon$ one can very easily show a sample complexity of $\epsilon^{-2} \tfrac{n}{m} \log(\tfrac{n}{m})$.} we show that the method of \cite{chen2017nearly} that leverages Lemma~\ref{lem:kaufmann_lemma} can be used for the $k=1 < m$ case with a very simple proof. However, the method does not generalize to the $1 < k \leq m$ case.
The next theorem handles all $1 \leq k \leq m$ cases simultaneously for a specific instance (i.e., not worst-case as in \cite{chaudhuri2019pac}). 

\begin{theorem}\label{thm:k_of_m_lower_bound}
Fix $\epsilon >0$, $\delta \in(0,1/16)$, and a vector $\mu \in \R^n$.
Consider $n$ arms where rewards from the $i$th arm are distributed according to $\mathcal{N}(\mu_i,1)$, a Gaussian distribution with mean $\mu_i$ and variance $1$.
Assume without loss of generality that $\mu_1 \geq \mu_2 \geq \dots \geq \mu_n$ and let $m = |\{ i \in [n] : \mu_i > \mu_1 - \epsilon\}|$. 
For every permutation $\pi \in \mb{S}^n$ let $(\mc{F}_t^\pi)_{t \in \N}$ be the filtration generated by the algorithm playing on instance $\pi(\rho)$, and let $\tau_\pi$ be a stopping time with respect to $(\mc{F}_t^\pi)_{t \in \N}$ at which time the algorithm outputs a set $\widehat{S}_{\tau_\pi} \subseteq [n]$ with $|\widehat{S}_{\tau_\pi}|=k$. 
If $\P_{\pi(\rho)}( \widehat{S}_{\tau_\pi} \subset \pi([m]) ) \geq 1-\delta$, then 
\begin{align*}
\E_{\pi \sim \mb{S}^n} \E_{\pi(\rho)}\Big[\sum_{i=1}^n T_{\pi(i)}(\tau_\pi) \Big] \geq  \H_{\mathrm{low},k}(\epsilon) \coloneqq \frac{1}{64} \Big(-(\mu_1-\mu_{m+1})^{-2} + \frac{k}{m} \sum_{i=m+1}^n (\mu_1 - \mu_{i})^{-2} \Big).
\end{align*} 
\end{theorem}
\begin{remark}
By definition, $(\mu_1-\mu_{m+1})^{-2} \leq \epsilon^{-2}$ so aside from pathological cases such as $\mu_1 - \mu_i \gg \epsilon$ for all $i > m+1$ the lower bound will be positive and non-trivial.
For example, suppose $m$ arms have means equal to $\mu_0 + \epsilon$ while the remaining have means equal to $\mu_0$. Then Theorem~\ref{thm:k_of_m_lower_bound} implies that to identify any $k$ of the top $m$ arms requires about $k\frac{n}{m}\epsilon^{-2}$ samples, which exactly matches our intuition for the $n/m$ scaling when identifying $(i)$ an $\epsilon$-good arm in the SimplePAC setting, and $(ii)$ $k$ arms above a threshold in the multiple identifications setting.
% for some $\alpha \in (0,1)$ suppose exactly $m=\alpha n$ arms have means equal to $\mu_0 + \epsilon$ while the remaining have means equal to $\mu_0$. Then Theorem~\ref{thm:k_of_m_lower_bound} implies that to identify any $k$ of the $\mu_0+\epsilon$ arms requires about $(\frac{k}{m}(n-m) - 1)\epsilon^{-2} = (k \tfrac{1-\alpha}{\alpha} - 1) \epsilon^{-2}$ samples.
\end{remark}

The proof of Theorem~\ref{thm:k_of_m_lower_bound} employs an extension of the \emph{Simulator} argument \cite{simchowitz2017simulator}.
While the $k=1$ case can be proven using an argument similar to \cite{chen2017nearly}, we needed the Simulator strategy for the $k > 1$ case.   
The technique may be useful for proving lower bounds for other combinatorial settings where many outcomes are potentially correct (i.e., choose any $k$ of $m$) \cite{DBLP:conf/nips/ChenLKLC14,chen2017nearly}.

\section{Algorithm}

% We introduce some notation. Let $T_{i,r}(t)$ be the number of times arm $i$ has been pulled in bracket $r$ up to time $r$. Let $\widehat{\mu}_{i,r, t}$ be the empirical mean of arm $i$ in bracket $r$ after $t$ pulls. Let $U(t,\delta) \propto \sqrt{\frac{1}{t} \log(\log(t)/\delta)}$ be an anytime confidence bound based on LIL {\color{red} [cite]} 

Algorithm~\ref{fdr_alg} simultaneously handles both the identification of an $\epsilon$-good arm (Line~\ref{lin:best_arm_set}) and the identification of multiple arms above a threshold $\mu_0$ (Line~\ref{lin:fdr_tpr_set}).  It opens progressively larger \emph{brackets}--subsets of the arms--over time. Each bracket $\ell$ is opened after $(\ell-1)2^{\ell-1}$ rounds and is drawn uniformly at random from $\binom{[n]}{M_{\ell}}$, where $M_{\ell} := n \wedge 2^{\ell}$ and $\binom{[n]}{M_{\ell}}$ denotes all subsets of $[n]$ of size $M_{\ell}$. Algorithm~\ref{fdr_alg} cycles through the open brackets, at each round pulling an arm $I_t$ in the chosen bracket $R_t$ that maximizes an upper confidence bound $ \widehat{\mu}_{i,R_t, T_{i,R_t}(t)} + U(T_{i,R_t}(t),\delta)$ on its mean. Here, $\widehat{\mu}_{i,r, t}$ denotes the empirical mean of arm $i$ in bracket $r$ after $t$ pulls, $T_{i,r}(t)$ denotes the number of times arm $i$ has been pulled in bracket $r$ up to time $t$, and finally $U(t,\delta) = c \sqrt{\frac{1}{t} \log(\log(t)/\delta)}$ denotes an anytime confidence bound (thus, satisfiying for any $r \in \N$ and $i \in [n]$ $\P(\cap_{t=1}^\infty |\widehat{\mu}_{i,r,t} - \mu_i| \leq U(t,\delta)) \geq 1 - \delta$) based on the law of the iterated logarithm (LIL) \cite{jamieson2014lil,kaufmann2016complexity}. We note that for the purposes of simplifying the analysis of the algorithm, observations from arms are not shared across brackets, but they should be shared in practice.

For $\epsilon$-good arm identification, the algorithm outputs a maximizer $O_t$ of its lower confidence bound (Line~\ref{lin:best_arm_set}), ensuring that once the lower confidence bound of some $\epsilon$-good arm is greater than $\mu_1 - \epsilon$, the algorithm will never output a non-$\epsilon$-good arm again with high probability. For the problem of multiple identifications above a threshold, various suggested sets are possible depending on the desired guarantees. In the main body of the paper, we focus on a guarantee called \emph{FDR-TPR} (false discovery rate-true positive rate) \citep{jamieson2018bandit} that guarantees approximate identification of the arms (see Theorem \ref{thm_fdr_paper} for a precise statement). For this goal, the algorithm  builds a set $\S_t$ (Line~\ref{lin:fdr_tpr_set}) based on the Benjamini-Hochberg procedure developed for multi-armed bandits in \cite{jamieson2018bandit}.

We briefly remark on a connection between our proposed algorithms and best arm identification. While $\S_t = \emptyset$, both algorithms essentially act identically to a nearly optimal best arm identification algorithm \cite{jamieson2014lil} since the same arm $I_t$ is pulled for all objectives. Furthermore, once $\S_t \neq \emptyset$, then the multiple identifications variant of the algorithm continues to act as a nearly optimal best arm identification algorithm on the remaining arms. These similarities reflect the deep connections between $\epsilon$-good arm identification, identifying means above a threshold, and best arm identification.

%The similarities between the algorithms for $\epsilon$-good arm identification and identifying arms with means above a threshold reflect the deep connection between these two problems. Until the first arm is determined to be above a threshold, both algorithms are essentially acting identically to a near-optimal best-arm identification algorithm \cite{jamieson2014lil} since the same arm $I_t$ is pulled for all objectives. Furthermore, if the multiple identifications variant of the algorithm has identified a subset of the arms as exceeding the threshold $\mu_0$, then it continues to act as a  near optimal best arm identification algorithm on the remaining arms. 

\begin{algorithm}[t]
\caption{Infinite UCB Algorithm: $\epsilon$-good arm identification and FDR-TPR}
\label{fdr_alg}
\footnotesize
\begin{algorithmic}[1]
\STATE $\delta_r = \frac{\delta}{r^2}$, $\delta^\prime_r = \frac{\delta_r}{6.4 \log(36/\delta_r)}$, $\ell = 0$, $R_0 = 0$, $\mathcal{S}_{0} = \emptyset$
\FOR{ $t =1,2, \ldots$}
\IF{$t \geq 2^\ell \ell$}
\STATE Draw a set $A_{\ell+1}$ uniformly at random from $\binom{[n]}{M_{\ell+1}}$, where $M_{\ell} := n \wedge 2^{\ell}$
\STATE $\ell = \ell + 1$
\ENDIF
\STATE $R_t = 1+  R_{t-1} \cdot \1\{R_{t-1} < \ell \} $
\IF{there exists $i \in A_{R_t} \setminus \mathcal{S}_{t}$ such that $T_{i,R_t}(t) = 0$}
\STATE  Pull an arm $I_t$ belonging to $\{i \in A_{R_t} \setminus \mathcal{S}_{t} : T_{i,R_t}(t) = 0 \}$
\ELSE
\STATE Pull arm $I_t = \text{argmax}_{i \in A_{R_t} \setminus \mathcal{S}_{t}} \widehat{\mu}_{i,R_t, T_{i,R_t}(t)} + U(T_{i,R_t}(t),\delta)$
\ENDIF
\IF{Best Arm Identification}
\STATE $O_t = {\text{argmax}}_{i \in A_r \text{ for some } r \leq \ell}   \widehat{\mu}_{i,r,T_{i,r}(t)} - U(T_{i,r}(t), \frac{\delta}{|A_{r}| r^2} )$ \hfill \texttt{\small\color{blue} \% Best-arm Thm.\ref{eps_good_theorem_paper}} \label{lin:best_arm_set}
\ELSIF{FDR-TPR}
\STATE $\forall p \in [|A_{R_t}|]$ set $s(p) = \{i \in A_{R_t} : \widehat{\mu}_{i,R_t, T_{i,R_t}(t)} - U(T_{i,R_t}(t), \frac{p}{|A_{R_t}| } \delta^\prime_{R_t}) \geq \mu_0 \}$
\STATE $\mathcal{S}_{t+1} = \mathcal{S}_{t} \cup s(\widehat{p})$ where $\widehat{p} = \text{max}\{p \in [|A_{R_t}|] : |s(p)| \geq p \}$ \hfill \texttt{\small\color{blue} \% FDR Thm.\ref{thm_fdr_paper}} \label{lin:fdr_tpr_set}
\ENDIF
\ENDFOR
\end{algorithmic}
\end{algorithm}

\section{Upper Bounds}

Our upper bounds all have a similar form. They are characterized in terms of $\Delta_{i,j} = \mu_i - \mu_j$, the \emph{gap} between the $i$th arm and the $j$th arm, where we henceforth assume without loss of generality that $\mu_1 \geq \mu_2 \geq \ldots \geq \mu_n$. For $\epsilon$-good arm identification we let $m = |\{i : \mu_i > \mu_1 - \epsilon \}|$ and for identifying means above a threshold $\mu_0$ we let $m = |\{i : \mu_i > \mu_0 \}|$. As argued in the introduction, in the SimplePAC setting if $m \gg 1$, one would expect the sample complexity to be about $m$ times smaller than if $m=1$. 
It turns out that it is more subtle than that because many of the top $m$ arms may be nearly indistinguishable from $\mu_{m+1}$, the largest mean that is not acceptable. Our upper bounds take into account this phenomenon, showing for example in $\epsilon$-good arm identification that Algorithm \ref{fdr_alg} may use a larger bracket, say of size $\Theta(\frac{n}{j})$ for some $j \leq m$, in the hopes to obtain an arm $l \leq j$ that is easier to identify as $\epsilon$-good. 

In Appendix~\ref{sec:upper_bound_proofs} we state our theorems including all factors, but for the purposes of exposition, here we use ``$\lesssim$" to hide constants and doubly logarithmic factors. For simplicity, we assume that the distributions are $1$-sub-Gaussian and that $\mu_0,\mu_1, \ldots, \mu_n \in [0,1]$. We also define $\log(x) \coloneqq \max(\ln(x),1)$. 

\subsection{Upper Bound for Identifying an $\epsilon$-good mean}

For all $j \in [m]$, let 
\begin{align*}
\bestH(\epsilon;j)  & \coloneqq \frac{1}{j}\left( \underbrace{\left\{\sum_{i=1}^j\Delta_{i,m+1}^{-2} + \sum_{i=j+1}^{m}(\Delta_{j,i} \vee \Delta_{i,m+1})^{-2}\right\}}_{\text{top arms}}\log(\frac{n}{j\delta}) + \underbrace{\sum_{i=m+1}^{n}\Delta_{j,i}^{-2}}_{\text{bottom arms}}\log (\frac{1}{\delta})  \right).
\end{align*}
$\bestH(\epsilon;j) $ bounds the expected number of samples required by a bracket of size $\Theta(\frac{n}{j})$ to identify an $\epsilon$-good arm when \emph{(i)} one of its arms is at least as large as $\mu_j > \mu_1 - \epsilon$ and \emph{(ii)} the empirical means of the arms in the bracket concentrate well enough. We note that as $j$ decreases from $m$ to $1$, the gaps of the bottom $n-m$ arms (given by $\Delta_{j,i}$) and the gaps of the top $m$ arms (which can be expressed as $\Delta_{j,i} \vee \Delta_{i,m+1}$) are both non-decreasing. Moreover, the gaps of the arms in $\{j+1, \ldots, m\}$ are at least $\tfrac{1}{2} (\mu_j - \mu_{m+1})$, so that for suitably small $j \in [m]$, the arms in $\{j+1, \ldots, m\}$ have large gaps even if their means are very close to $\mu_{m+1}$.
%\begin{align*}
%\frac{\mu_{|G_\gamma|} - \mu_{m+1}}{2} \leq \max(\mu_{|G_\gamma|} - \mu_i, \mu_i-\mu_{m+1}) \leq \mu_{|G_\gamma|} - \mu_{m+1}.
%\end{align*}
Theorem \ref{eps_good_theorem_paper} gives our upper bound for $\epsilon$-good arm identification.

 % We will discuss the gaps after presenting the Theorem. In the following, we use ``$\lesssim$" to hide constants and doubly logarithmic factors. Later in the paper, we will restate the the theorems with doubly logarithmic factors. 
% Interpretation of Theorem \ref{eps_good_theorem_paper} is provided directly following the theorem.
 % contains two upper bounds: line \eqref{eps_good_theorem_res_sim}, which is very accessible and line \eqref{eps_good_theorem_res}, which is much tighter but requires some discussion to understand. 
\begin{theorem}[$\boldsymbol{\epsilon}$-good identification]
\label{eps_good_theorem_paper}
Let $\delta \leq 0.025$ and $\epsilon > 0$. Let $(\mc{F}_t)_{t \in \N}$ be the filtration generated by playing Algorithm \ref{fdr_alg} on problem $\rho$. Then, there exists a stopping time $\tau$ wrt $(\mc{F}_t)_{t \in \N}$ such that
%\begin{align}
%\E[\tau] &\lesssim  \min_{\gamma \in(0,\epsilon)} \widetilde{U}_\epsilon(\gamma)\log(\widetilde{U}_\epsilon(\gamma) + \Delta_{m, \epsilon, \gamma}^{-2} ) \label{eps_good_theorem_res_paper} \\
%& \lesssim  \widetilde{U}_\epsilon \log(\widetilde{U}_\epsilon + (\mu_m - \mu_{m+1})^{-2} ) \label{eps_good_theorem_res_sim_paper}
%\end{align}
\begin{align}
\E[\tau] & \lesssim  \min_{j \in [m]} \bestH(\epsilon;j) \log(\bestH(\epsilon;j) + \Delta_{j,m+1}^{-2}) \label{eq:eps_good_bound}
\end{align}
and $\P( \exists s \geq \tau: \mu_{O_s} \leq \mu_1 - \epsilon) \leq  2\delta$. 
\end{theorem}

\begin{remark}
Assume the setting of Theorem \ref{eps_good_theorem_paper}. If $m$ arms have means equal to $\mu_0 + \epsilon$ while the remaining have means equal to $\mu_0$ then defining $\bar{\H} \coloneqq \frac{n}{m} \epsilon^{-2} \log(1/\delta)$, we have that
% \begin{align*}
$\E[\tau] \lesssim \bar{\H} \log(\bar{\H})$. 
This matches the lower bound given by $\H_{\mathrm{low},1}(\epsilon)$ up to $\log$ factors.
% \end{align*}
\end{remark}

Consider putting $j =m$ in \eqref{eq:eps_good_bound}. The term $\bestH(\epsilon;m)$ bounds the expected number of rounds required by a bracket of size $\Theta(\frac{n}{m})$ to identify an $\epsilon$-good arm when one of its arms is $\epsilon$-good. The extra logarithmic factor reflects the cost of adapting to unknown $m$.
In many situations setting $j =m$ in \eqref{eq:eps_good_bound} is woefully loose because while a bracket of size $\Theta(\frac{n}{m})$ is sufficiently large to contain an $\epsilon$-good arm with constant probability, it may be advantageous to use a much larger bracket in hopes of getting an $\epsilon$-good arm that is much easier to identify as $\epsilon$-good unverifiably. 
Figure~\ref{fig:doodles} illustrates a bandit instance that demonstrates this tradeoff for a particular $j \in [m]$. 
Informally, if one randomly chooses $\frac{n}{m}$ arms then one expects the highest mean amongst these to have an index $I$ uniformly distributed in $\{1,\dots,m\}$. The means of many of these arms may be very close to the means of the bottom $n-m$ arms so that on average an enormous number of samples is required to distinguish $I$ from the bottom $n-m$ arms and, therefore, to unverifiably identify $I$ as $\epsilon$-good.
On the other hand, if one randomly chooses $\frac{n}{j}$ arms then one expects the highest mean amongst these to have an index $I^\prime$ uniformly distributed in $\{1,\dots,j \}$. The means of these arms may be substantially larger than the means of the bottom $n-m$ arms so that on average far fewer samples are required to distinguish $I^\prime$ from the bottom $n-m$ arms.
Thus, there is a problem-dependent tradeoff in the number of arms to consider and an effective strategy must naturally adapt to it. 
%
% , which depicts a bandit problem in which blue distributions have $\mu_i=0.9$ ($i \in \{1,\ldots,.1n\}$), the orange distributions have $\mu_i =0.81$ ($i \in \{0.1n +1,\ldots, 0.5n \}$) and the red distributions have $\mu_i =0.8$ ($i \in \{ 0.5n+1,\ldots,n \}$).The goal is to identify an arm with $\mu_i > 0.8$. Although $m = 0.5n$ it is dramatically easier to identify a blue distribution as having a mean above $0.8$ than it is to to do so for an orange distribution, which makes the effective size of $m$ $0.1n$. 
% Informally, if one randomly chooses $\frac{n}{m} = 2$ arms, then on a typical case one of these arms is an orange distribution with $\mu_i = 0.81$ and the other is a red distribution with mean $0.8$ and information theoretic bounds say that it takes at least $10000$ samples to determine that the orange distribution has a mean greater than $0.8$, resulting in a sample complexity of at least $20,000$. On the other hand, if one randomly chooses $10$ arms, in typical case at least one of these arms is a blue distribution and it takes on the order of $100$ samples to determine it has a mean above point eight, resuliting in a sample complexity of at most $1,000$. In short, there is a trade-off between considering  fewer arms and potentially needing to identify an arm with a mean barely above $\mu_1-\epsilon$ as an $\epsilon$-good arm and considering more arms but having a higher likelihood being able to use an arm with a larger mean. This trade-off depends on the problem instance and an effective strategy must naturally adapt to it. 
%
The minimization problem of inequality~(1) says that Algorithm \ref{fdr_alg} uses a bracket with the optimal number of arms to identify an $\epsilon$-good arm. 

\begin{figure} % [tb]
    \centering
    \includegraphics[width=.75\textwidth]{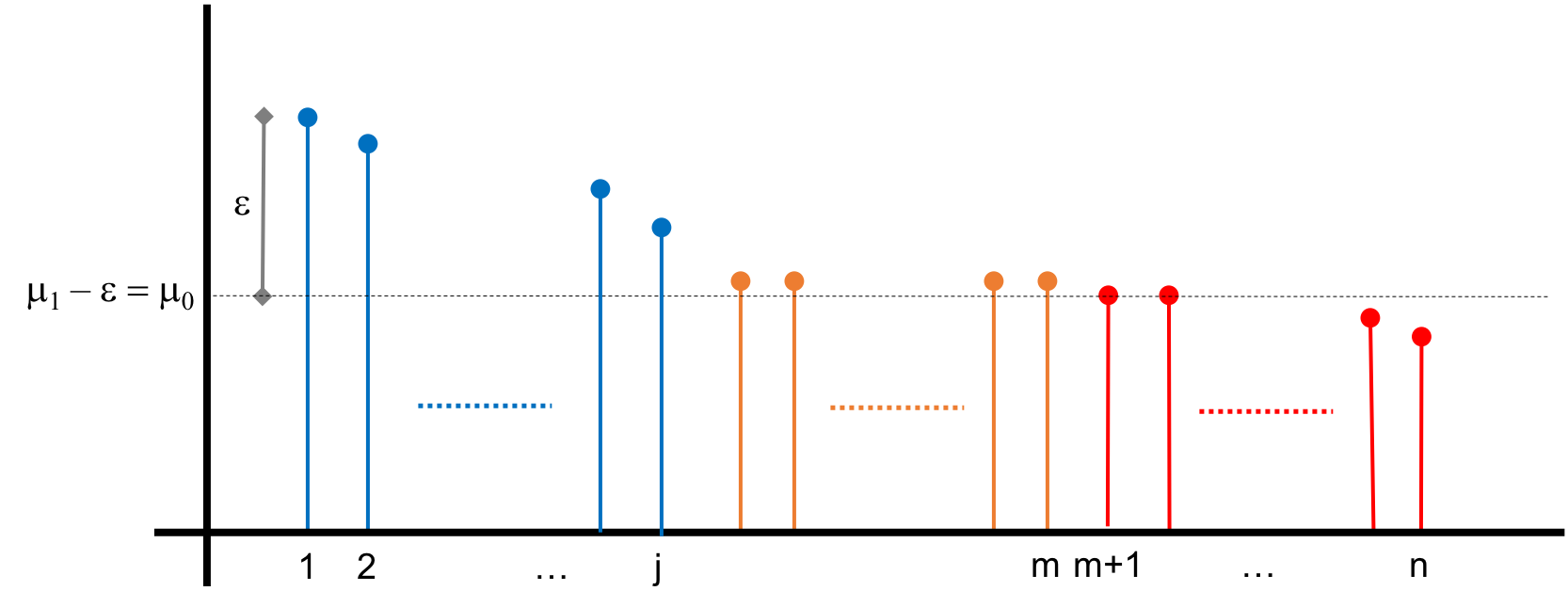}
    % \vspace*{-5mm}
    \caption{Our sample complexity results rely on picking a bracket of an appropriate size: $\frac{n}{m}$ is too small, $n$ is too large, and $\frac{n}{j}$ appears to be about a good size. }
    \label{fig:doodles}
\end{figure}

\subsection{Upper Bound for Identifying means above a threshold $\mu_0$}

Define 
% \begin{align*}
$\H_1 = \{i \in [n] : \mu_i > \mu_0 \} \,  \, \, \text{ and } \, \,  \,   \H_0 = \{i \in [n] : \mu_i \leq \mu_0 \}.$
% \end{align*}
$\H_1$ consists of the arms that we wish to identify and $\H_0$ all the other arms. Let $m = |\H_1|$ and define $\Delta_{j,0} \coloneqq \mu_j -\mu_0$. For all $j \in [m]$, define
\begin{align*}
\fdrH(\mu_0;j)  & \coloneqq \frac{k}{j}\left( \underbrace{\left\{\sum_{i=1}^m \Delta_{i \vee j, 0}^{-2}\right\}}_{\text{top arms}}\log(\frac{nk}{j\delta}) + \underbrace{\sum_{i=m+1}^{n}(\Delta_{j,i})^{-2}}_{\text{bottom arms}}\log(\frac{1}{\delta})  \right) \\
\fdrHt(\mu_0;j) & \coloneqq  \frac{n}{j}k \Delta_{j,0}^{-2} \log\left( 1/\delta\right).
\end{align*}
$\fdrH(\mu_0;j)$ bounds the expected number of samples required by a bracket of size $\Theta(\frac{nk}{j})$ to identify $k$ arms satisfying $\mu_i > \mu_0$ when \emph{(i)} at least $k$ of its arms have means greater than $\mu_j > \mu_0$ and \emph{(ii)} the empirical means of the arms in the bracket concentrate well. $\fdrHt(\mu_0;j)$ plays a similar role but removes a logarithmic factor on the arms in $\H_1$ at the cost of losing the dependence on the individual gaps. Theorem \ref{thm_fdr_paper} gives a FDR-TPR guarantee (see Appendix \ref{sec:additional_algorithms} for less practical algorithms with stronger theoretical guarantees).
\begin{theorem}[FDR-TPR]
\label{thm_fdr_paper}
Let $\delta \leq .025$. Let $k \leq |\H_1|$. Let $(\mc{F}_t)_{t \in \N}$ be the filtration generated by playing Algorithm \ref{fdr_alg} on problem $\rho$. Then, for all $t \in \N$, $\E[ \frac{|\S_t \cap \H_0|}{|\S_t| \wedge 1}] \leq 2 \delta$ and there exists a stopping time $\tau_k$ wrt $(\mc{F}_t)_{t \in \N}$ such that
%\begin{align}
%\E[\tau_k] & \lesssim \min_{\epsilon \geq \Delta : |\H_{1,\epsilon}| \geq k } \widetilde{S}_k(\epsilon) \log(\widetilde{S}_k(\epsilon) +  \epsilon^{-2} ) \label{main_thm_FDR_res_paper} \\
%& \lesssim  \widetilde{S}_k \log( \widetilde{S}_k +  \Delta^{-2} ), \text{ and} 
% \label{main_thm_FDR_res_simple_paper} \\
%\E[\tau_k] & \lesssim \min_{\epsilon \geq \Delta : |\H_{1,\epsilon}| \geq k }  \widetilde{T}_k(\epsilon) \log(\widetilde{T}_k(\epsilon))  \label{main_thm_res_ind_paper}  \\
%& \lesssim \widetilde{T}_k \log(\widetilde{T}_k) \label{main_thm_res_ind_simple_paper} 
%\end{align}
\begin{align}
\E[\tau_k] & \lesssim \min_{k \leq j \leq m } \fdrH(\mu_0;j)  \log(\fdrH(\mu_0;j) +  \Delta_{j,0}^{-2} ), \quad  \text{ and} \label{eq:fdr_bound_dep} \\
\E[\tau_k] & \lesssim \min_{k \leq j \leq m}  \fdrHt(\mu_0;j) \log(\fdrHt(\mu_0;j)) \label{eq:fdr_bound_indep} 
\end{align}
% \begin{align}
% \E[\tau_k] & \lesssim  \min_{\epsilon >0 : |\H_{1,\epsilon}| \geq k } {S}_k(\epsilon) \log[{S}_k(\epsilon) +  \epsilon^{-2} ] \label{main_thm_FWER_res} \\
% & \leq  \widetilde{S}_k \log[\widetilde{S}_k +  \Delta^{-2}  ]\label{main_thm_FWER_res_simple} 
% \end{align}
and for all $t \geq \tau_k$, $\E[|\S_t \cap \H_1|] \geq (1-\delta)k$. 
\end{theorem}

\begin{remark}
Assume the setting of Theorem \ref{thm_fdr_paper}. If $m$ arms have means equal to $\mu_0 + \epsilon$ while the remaining have means equal to $\mu_0$ then for any $k \leq m$, defining $\bar{\H}_k \coloneqq \frac{n}{m} k \epsilon^{-2} \log(1/\delta)$, we have that
% \begin{align*}
$\E[\tau_k] \lesssim \bar{\H} _k \log(\bar{\H}_k)$. 
This matches the lower bound given by $\H_{\mathrm{low},k}(\epsilon)$ up to $\log$ factors.
% \end{align*}
\end{remark}

\eqref{eq:fdr_bound_dep} gives a gap-dependent bound, while \eqref{eq:fdr_bound_indep} sacrifices the dependence on the individual gaps to remove an additional logarithmic factor on the arms in $\H_1$. We discuss inequality \eqref{eq:fdr_bound_dep}, but similar remarks apply to \eqref{eq:fdr_bound_indep}. Plugging $j =m$ into inequality \eqref{eq:fdr_bound_dep} gives the performance of a bracket of size $\Theta(\tfrac{n}{|\H_1|}k)$ while the minimization problem in inequality \eqref{eq:fdr_bound_dep} shows that the algorithm uses a bracket of optimal size.
Paralleling $\epsilon$-good arm identification, it may be very useful to consider more than $\Theta(\frac{n}{|\H_1|}k)$ arms in order to find $k$ that are easier to identify as being larger than the threshold. 
Indeed, the example in and discussion concerning Figure~\ref{fig:doodles} apply directly to this problem as well. 

\begin{remark} Our upper bounds scale as $\log(1/\delta)$ which arises due to requiring concentration of measure on subsets of the observations of the arms.
We do \emph{not} rely on a high probability event that a particular bracket includes some number of good arms which would result in a $\log^2(1/\delta)$ that is common in other related results \cite{chaudhuri2017pac,aziz2018pure,chaudhuri2019pac}; this more careful analysis could also be applied to these other related problem settings. We note that we could obtain a high probability bound at the cost of losing the dependence on the individual gaps for arms with $i \leq j_0$, where $j_0$ is the minimizer of either \eqref{eq:eps_good_bound} or \eqref{eq:fdr_bound_dep}. For example, in $\fdrH(\mu_0;j_0) $, $\Delta_{j_0,0}$ would replace $\Delta_{i \vee j_0,0}$ for all $i \leq j_0$. 
\end{remark}

\section*{Acknowledgements}

The authors would like to thank Max Simchowitz for very helpful feedback that substantially improved the clarity of the paper, as well as Jennifer Rogers and Andrew Wagenmaker for their very useful comments. We also thank Horia Mania for inspiring the proof of Lemma~\ref{lem:min_prob_in_set}. Julian Katz-Samuels is grateful to Clay Scott for his very generous support during the writing of this paper, which relied on NSF Grants No. 1422157 and 1838179 and funding from the Michigan Institute for Data Science.

% \begin{remark}
% We stress that the algorithm itself never stops and that the stopping time $T_\pi$ is for analysis purposes only. One possible choice could be $T_\pi = \min\{ t: \widehat{S}\subset \pi([m]), |\widehat{S}|=k \}$. While the algorithm would not be able to implement it since it has no knowledge of $\pi([m])$, it is still a valid stopping time because an oracle does know $\pi([m])$.
% \end{remark}

 \clearpage
 \appendix

\section{Related work: $(\epsilon,\delta)-PAC$ for identifying $k$ $\epsilon$-good arms}\label{sec:kaufmann_results}
\cite{kaufmann2016complexity} proved the following theorem which characterizes the sample complexity for $\epsilon$-good arm identification $k=1, m\geq 1$ and multiple identifications above a threshold $\mu_0$ in the special case of $k=m$ (in general, we are interested in any $1 \leq k \leq m$) in the  $(\epsilon,\delta)$-PAC setting.
% Define
% \begin{align*}
% \Delta_i = 
% \begin{cases}
% \mu_i - \mu_{m+1} & \text{ if } i \leq m \\
% \mu_m - \mu_i & \text {if } i > m 
% \end{cases}
% \end{align*}
\begin{theorem}[\cite{kaufmann2016complexity}]\label{thm:kaufmann_eps_good}
Fix $\epsilon, \delta >0$, and a vector $\mu \in \R^n$.
Fix a bandit instance $\rho$ of $n$ arms where the $i$th distribution equals $\rho_i(\mu) = \mathcal{N}(\mu_i,1)$, a Gaussian distribution with mean $\mu_i$ and variance $1$.
Assume without loss of generality that $\mu_1 \geq \mu_2 \geq \dots \geq \mu_n$ and let $m = |\{ i \in [n] : \mu_i \geq \mu_1 - \epsilon\}|$ so that $\mu_i \geq \mu_1 - \epsilon$ for all $i \in [m]$. 
If algorithm $\mathcal{A}$ returns $k=1$ arms of the top $m$ arms and is $(\epsilon,\delta)$-PAC on $\mc{P} = \{ \mc{N}(\mu',I) : \mu' \in \R^n \}$ then 
\begin{align*}
% \E_{\rho}[\tau_{PAC}] = 
\E_{\rho}\Big[\sum_{i=1}^n T_i(\tau_{PAC}) \Big] \geq  \tfrac{1}{2}\log(1/2.4\delta) \Big( (m-1)\epsilon^{-2} + \sum_{i=m+1}^n (\mu_1 - \mu_{i})^{-2} \Big) \qquad (k=1)
\end{align*} 
Under the same conditions, if $\mc{A}$ returns $k=m$ arms then
\begin{align*}
% \E_{\rho}[\tau_{PAC}] = 
\E_{\rho}\Big[\sum_{i=1}^n T_i(\tau_{PAC}) \Big] \geq  2\log(1/2.4\delta) \Big(\sum_{i=1}^m (\mu_i - \mu_{m+1})^{-2} + \sum_{i=m+1}^n (\mu_m - \mu_{i})^{-2} \Big)  \quad (k=m)
\end{align*} 
\end{theorem}
% \begin{proof}
% Let $m = |\{i : \mu_i \geq \mu_1 -\epsilon\}|$.
% Let $\rho_i = \mathcal{N}(\mu_i,1)$ and $\rho_i^{(j)} = \mathcal{N}(\mu_i^{(j)},1)$ for $i \in [n]$ where
% \begin{align*}
%     \mu_i^{(j)} = \begin{cases} \mu_1 + \epsilon & \text{ if } i = j \leq m \\
%     \mu_i & \text{ otherwise.}
%     \end{cases}
% \end{align*}
% Then $KL(\P_\rho, \P_{\rho^{(j)}} ) = (\mu_1 - \mu_i + \epsilon)^2/2$.
% If $E_j = \{\widehat{i} = j\}$ then $\P_\rho(E_j) \leq \delta$ and $\P_{\rho^{(j)}}(E_j) \geq 1-\delta$.
% Thus,
% \begin{align*}
%     \E_\rho[T_i] \geq \frac{ 2 d(\delta,1-\delta) }{(\mu_1-\mu_i+\epsilon)^2} \geq 2 \log(1/2.4\delta) (\mu_1-\mu_i + \epsilon)^{-2}. 
% \end{align*}
% Summing over all $i > m$ completes the proof.
% \end{proof}
\noindent Note that by the definition of $m$ we have that $\mu_m - \mu_{m+1} > 0$.
We emphasize that the sample complexity of Theorem~\ref{thm:kaufmann_eps_good} for both $k=1$ or $k=m$ is necessarily $\Omega(n)$ regardless of the number of $\epsilon$-good arms $m$.
As discussed below, the $k=1$ lower bound is achievable up to $\log\log$ factors \cite{icml2013_karnin13}.
The special case of $k=m$ is notably the TOP-$k$ identification problem where lower bounds were recently sharpened with additional log factors independently by \cite{simchowitz2017simulator,chen2017nearly}.
In particular, if for some $\mu_0$ we have $\mu_i = \mu_0+\epsilon$ for $i \leq m$ and $\mu_i = \mu_0$ for $i < m$ then their lower bounds on the expected sample complexity scale like $k \epsilon^{-2} \log(n-k) + (n-k) \epsilon^{-2} \log(k)$, which is always larger than $n\epsilon^{-2}$ that is predicted by the above theorem.

% In the PAC setting, Theorem~\ref{thm:kaufmann_eps_good} addresses the $k=1$ case whereas the $k > 1$ case is at least as hard as identifying just a single $\epsilon$-good arm.
% Thus, identifying $k$ $\epsilon$-good arms for any $1 \leq k \leq m$ is at least $\Omega(n)$. 

% For a given instance $\mu_1 \geq \mu_m > \mu_{m+1} \geq \dots \geq \mu_n$ in the $(\epsilon,\delta)$-PAC setting where $\mc{P} = \{ \mc{N}(\mu',I) : \mu' \in \R^n \}$, \cite{DBLP:conf/icml/KalyanakrishnanTAS12,kaufmann2016complexity} proved a sample complexity of $\sum_{i=1}^n \Delta_{i}^{-2} \log(1/\delta)$ for $\Delta_i = \max\{ \mu_m - \mu_i, \mu_{i} - \mu_{m+1}\}$.

% And in the more instance-specific setting when $\mc{P} = \{ \rho(\pi(\mu)) : \pi \in \mb{S}^n \}$, \cite{simchowitz2017simulator,chen2017nearly} independently showed a sample complexity of at least $\sum_{i=1}^n \Delta_{i}^{-2}$. 

\section{Proof of lower bounds}
We now briefly provide some intuition behind the proof. 
Suppose $m>1$ and $k=1$ and consider the easier problem where the permutation set averaged over is just the identity permutation $\pi_1 = (1,2,\dots,n)$ and the permutation $\pi_2$ that swaps $\{1,\dots,m\}$ and some fixed $\sigma \subset [n]\setminus [m]$ with $|\sigma| = m$.
That is, the algorithm knows the instance it is playing is either $\pi_1(\rho)=\rho$ or $\pi_2(\rho)$ where $\rho$ is known but the permutation $\pi_1$ or $\pi_2$ is not.
Information theoretic arguments say that at least $\tau \approx \min_{i \in \sigma} (\mu_1 - \mu_i)^{-2}$ observations from $[m] \cup \sigma$ are necessary in order to determine whether the underlying instance is $\pi_1(\rho)$ versus $\pi_2(\rho)$.
But if the algorithm cannot distinguish between $\pi_1$ and $\pi_2$ with fewer than $\tau$ samples, then we can also argue that if $\pi_1$ and $\pi_2$ are chosen with equal probability, then taking nearly $\tau$ samples from the arms in $\sigma$ with sub-optimal means is unavoidable in expectation. The choice of $\sigma$ was arbitrary and there are $\frac{n}{m}-1$ disjoint choices (e.g., $\{m+1,\dots,2m\},\{2m+1,\dots,3m\},\dots$) resulting in a lower bound of about $\frac{1}{m} \sum_{i=m+1}^n (\mu_1-\mu_i)^{-2}$.

The $k >1$ case is trickier because if we used just $\pi_1$ and $\pi_2$ as above, as soon as we found just one $\epsilon$-good arm (and thus being able to accurately discern whether the instance is $\pi_1(\rho)$ or $\pi_2(\rho)$) the algorithm would immediately know of $m-1$ other $\epsilon$-good arms. 
To overcome this, we choose a large enough set $\sigma \subset [m]$ such that $\sigma \cap \widehat{S}$ is non-empty with constant probability on the identity permutation. 
This way, if we swap this set $\sigma \subset [m]$ with some other set in $[n]\setminus [m]$ of size $|\sigma|$, then the algorithm would error with constant probability on this alternative permutation.
The next lemma guarantees the existence of such a set of size $\lceil m/k \rceil$ and the final result follows from the fact that there are about $\frac{n}{\lceil m /k \rceil}$ such disjoint choices in $[n] \setminus [m]$.

% $\pi_k$ for $k=2,\dots,n/m$ be the permutation that swaps $(1,\dots,m)$ with $((k-1)*m+1,\dots,km)$. 
% Now suppose the adversary chooses $\widehat{k} \in \{1,\dots,n/m\}$ uniformly at random and uses permutation $\pi_{\widehat{k}}$ to answer queries from the algorithm.
% If the algorithm only 
We introduce the following notation: for any $j \leq m$ let $\binom{[m]}{j}$ denote all subsets of $\{1,\dots,m\}$ of size $j$.
\begin{lemma}\label{lem:min_prob_in_set}
Fix $m \in \mathbb{N}$ and let $S$ be a random subset of size $k \leq m$ drawn from an arbitrary distribution over $\binom{[m]}{k}$.
For any $\ell \leq m-k$ there exists a subset $\sigma \subset [m]$ with $|\sigma| = \ell$ such that 
\begin{align*}\P( \sigma \cap S \neq \emptyset ) \geq 1 - \binom{m-k}{\ell}/\binom{m}{\ell} \geq 1 - e^{-\ell k/m}
\end{align*} 
If $\ell > m-k$ then $\P( \sigma \cap S \neq \emptyset )=1$.
\end{lemma}
\begin{proof}
Because the max of a set of positive numbers is always at least the average, we have
\begin{align*}
    \max_{\sigma \in \binom{[m]}{\ell}} \P( \sigma \cap S \neq \emptyset ) &\geq \frac{1}{\binom{m}{\ell}} \sum_{\sigma \in \binom{[m]}{\ell}} \P( \sigma \cap S \neq \emptyset ) \\
    &= \frac{1}{\binom{m}{\ell}} \sum_{\sigma \in \binom{[m]}{\ell}} \sum_{s \in \binom{[m]}{k}} \P( S = s ) \1\{ \sigma \cap s \neq \emptyset ) \\
    &= \frac{1}{\binom{m}{\ell}} \sum_{s \in \binom{[m]}{k}} \P( S = s ) \sum_{\sigma \in \binom{[m]}{\ell}}  \1\{ \sigma \cap s \neq \emptyset ) \\
    &= \frac{1}{\binom{m}{\ell}} \sum_{s \in \binom{[m]}{k}} \P( S = s ) \left( \binom{m}{\ell} - \binom{m-k}{\ell} \right) \\
    &= 1 - \binom{m-k}{\ell}/\binom{m}{\ell}
\end{align*}
where the last line follows from the fact that $\sum_{s \in \binom{[m]}{k}} \P( S = s )=1$ because it is a probability distribution.
Now
\begin{align*}
    \binom{m-k}{\ell}/\binom{m}{\ell} &= \frac{(m-k)! \, (m-\ell)!}{(m-k-\ell)! \, m!} \\
    &= \prod_{i=0}^{k-1} \frac{m-i-\ell}{m-i} = \prod_{i=0}^{k-1} \left( 1 - \frac{\ell}{m-i} \right) \leq \prod_{i=0}^{k-1} \left( 1 - \frac{\ell}{m} \right) \leq e^{- \ell k/m}.
\end{align*}
\end{proof}

Fix any $\sigma \subset [m]$ with $|\sigma|=\lceil m/k \rceil$ that satisfies $\P_\rho\left( \widehat{S} \cap \sigma \neq \emptyset \right) \geq 1- e^{-1}$ (which must exist by the above lemma).
Now fix any $\sigma' \subset [n]\setminus [m]$ with $|\sigma'|=|\sigma|$ and define $\rho'$ as swapping the arms of $\sigma$ and $\sigma'$, maintaining their relative ordering of the indices within the sets. 
Note that by the correctness assumption at the relative stopping times of $\rho$ and $\rho'$ we have 
\begin{align*}
\P_\rho( \widehat{S} \subset [m]) \geq 1-\delta, \qquad \P_{\rho'}( \widehat{S} \cap \sigma \neq \emptyset ) \leq \delta, \qquad \P_{\rho}(\widehat{S} \cap \sigma \neq \emptyset) \geq 1- e^{-1}
\end{align*}
which implies 
% $\P_{\rho'}( \widehat{S} \subset [m] ) = \P_{\rho'}( \widehat{S} \subset [m] \setminus \sigma ) + \P_{\rho'}( \widehat{S} \subset [m] , \widehat{S} \cap \sigma \neq \emptyset ) \leq e^{-kl/m} + \delta$.
% It follows that 
\begin{align}\label{eqn:TV_bound}
\mathrm{TV}(\P_\rho, \P_{\rho'}) = 
\sup_{\mathcal{E}} |\P_\rho(\mathcal{E})-\P_{\rho'}(\mathcal{E})| \geq |\P_\rho(\widehat{S} \cap \sigma \neq \emptyset)- \P_{\rho'}(\widehat{S} \cap \sigma \neq \emptyset)| \geq 1-\delta - e^{-1}.
\end{align}

\begin{remark}
Given \eqref{eqn:TV_bound}, one is tempted to apply Pinsker's inequality to obtain the right-hand-side of Lemma~1 from \cite{kaufmann2016complexity} and then provide a lower bound on $\E_\rho[\sum_{i \in \sigma \cup \sigma'} T_i]$. The difficulty here is that once we cover $[n] \setminus [m]$ with alternative $\sigma'$ sets, they would all share the same $\sigma$ in this lower bound, which suggests putting all samples on $\sigma$ and a trivial lower bound.
Alternatively, one could consider using the technique of \cite{chen2017nearly} which compares a given instance to a degenerate instance where the means of $\sigma'$ would be copied to $\sigma$ and argue that the probability of error is at least $1/2$ since there truly is no difference. This strategy is successful if $k=1$ so that $|\sigma|=m$ but breaks down when $k >1$ because one cannot reason about what the algorithm would have to do if the means of $\sigma$ were changed like one could if $k=1$. 
Consequently, we employ the use of the Simulator argument from \cite{simchowitz2017simulator} that is much more powerful at the cost of the introduction of some machinery. 
\end{remark}

\subsubsection*{The Simulator (background)}
The simulator argument is a kind of thought experiment where the player is playing against a non-stationary distribution. 
In the real game when the player pulls arm $I_t = i$ arm at time $t$ she observes a sample from the $i$th distribution of instance $\rho$: $X_{i,t} \sim \rho_i$. However, when playing against the simulator she observes a sample form the $i$th distribution of an instance denoted $\mathrm{Sim}(\rho, \{I_1,\dots,I_{t}\} )$ that depends on all past requests: ${X}_{i,t} \sim \mathrm{Sim}(\rho, \{I_1,\dots,I_{t}\} )_{i}$ with probability law $Q$ given $\rho$, $\{I_s=i_s \}_{s=1}^t$. 
That is, instead of receiving rewards from a stationary distribution $\rho$ at each time $t$, the simulator is an instance that depends on all the indices of past pulls (but not their values). 
For any set $A \subset \R$ define
\begin{align*}
\P_{\mathrm{Sim}(\rho,(i_1,\dots,i_t))}\left( X_{i_t,t} \in A \right) \coloneqq Q\left( X_{i_t,t} \in A | \rho, \{I_s=i_s \}_{s=1}^t \right) .
\end{align*}
We allow the algorithm to have internal randomness with probability law $P$ so that for $B \subset [n]$ define
\begin{align*}
\P_{\mathrm{Alg}((i_1,x_1,\dots,i_{t-1},x_{t-1}))}\left(I_t \in B \right) \coloneqq P\left(I_t \in B | \{I_s=i_s, X_{I_s}=x_s \}_{s=1}^{t-1} \right)
\end{align*}
so that for any event $E \in \mathcal{F}_T$ we define
\begin{align*}
&\P_{\mathrm{Alg},\mathrm{Sim}(\rho)}( E )\\
 &\coloneqq \sum_{i_1,\dots,i_T} \int_{x_1,\dots,x_T} \1_E \prod_{t=1}^T Q\left( X_{I_t} = x_t | \rho, \{I_s=i_s \}_{s=1}^t \right) P\left(I_t = i_t | \{I_s=i_s, X_{I_s}=x_s \}_{s=1}^{t-1} \right) dx_1\dots dx_T \\
&= \sum_{i_1,\dots,i_T} \int_{x_1,\dots,x_T} \1_E \prod_{t=1}^T \P_{\mathrm{Sim}(\rho,(i_1,\dots,i_t))}\left( X_{I_t} = x_t  \right) \P_{\mathrm{Alg}((i_1,x_1,\dots,i_{t-1},x_{t-1}))}\left(I_t = i_t \right)dx_1\dots dx_T
\end{align*}
so that for any $T$ we have $KL\left( \P_{\mathrm{Alg},\mathrm{Sim}(\rho)}, \P_{\mathrm{Alg},\mathrm{Sim}(\rho')} \right)  =$
\begin{align*}
&  \sum_{i_1,\dots,i_T} \int_{x_1,\dots,x_T} \P_{\mathrm{Alg},\mathrm{Sim}(\rho)}( \{I_s=i_s, X_{I_s}=x_s \}_{s=1}^{T} ) \log\left( \frac{\P_{\mathrm{Alg},\mathrm{Sim}(\rho)}( \{I_s=i_s, X_{I_s}=x_s \}_{s=1}^{T} )}{\P_{\mathrm{Alg},\mathrm{Sim}(\rho')}( \{I_s=i_s, X_{I_s}=x_s \}_{s=1}^{T} )} \right) dx_1\dots dx_T \\
&=  \sum_{i_1,\dots,i_T} \int_{x_1,\dots,x_T} \P_{\mathrm{Alg},\mathrm{Sim}(\rho)}( \{I_s=i_s, X_{I_s}=x_s \}_{s=1}^{T} ) \log\left( \frac{\prod_{t=1}^T \P_{\mathrm{Sim}(\rho,(i_1,\dots,i_t))}\left( X_{I_t} = x_t \right) }{\prod_{t=1}^T \P_{\mathrm{Sim}(\rho',(i_1,\dots,i_t))}\left( X_{I_t} = x_t \right) } \right) dx_1\dots dx_T \\
&=  \sum_{t=1}^T\sum_{i_1,\dots,i_T} \int_{x_1,\dots,x_T}  \P_{\mathrm{Alg},\mathrm{Sim}(\rho)}( \{I_s=i_s, X_{I_s}=x_s \}_{s=1}^{T} ) \log\left( \frac{ \P_{\mathrm{Sim}(\rho,(i_1,\dots,i_t))}\left( X_{I_t} = x_t \right) }{ \P_{\mathrm{Sim}(\rho',(i_1,\dots,i_t))}\left( X_{I_t} = x_t  \right) } \right) dx_1\dots dx_T\\
&=  \sum_{t=1}^T\sum_{i_1,\dots,i_T}  \P_{\mathrm{Alg},\mathrm{Sim}(\rho)}( \{I_s=i_s \}_{s=1}^{T} ) \int_{x_t} \P_{\mathrm{Sim}(\rho,(i_1,\dots,i_t))}\left( X_{I_t} = x_t \right) \log\left( \frac{ \P_{\mathrm{Sim}(\rho,(i_1,\dots,i_t))}\left( X_{I_t} = x_t \right) }{ \P_{\mathrm{Sim}(\rho',(i_1,\dots,i_t))}\left( X_{I_t} = x_t  \right) } \right) dx_t \\
&=  \sum_{t=1}^T\sum_{i_1,\dots,i_T}   \P_{\mathrm{Alg},\mathrm{Sim}(\rho)}( \{I_s=i_s \}_{s=1}^{T} )KL\left(\P_{\mathrm{Sim}(\rho,(i_1,\dots,i_t))} , \P_{\mathrm{Sim}(\rho',(i_1,\dots,i_t))} \right) \\
&= \sum_{i_1,\dots,i_T}   \P_{\mathrm{Alg},\mathrm{Sim}(\rho)}( \{I_s=i_s \}_{s=1}^{T} )\sum_{t=1}^T KL\left(\P_{\mathrm{Sim}(\rho,(i_1,\dots,i_t))} , \P_{\mathrm{Sim}(\rho',(i_1,\dots,i_t))} \right) \\
&\leq \max_{i_1,\dots,i_T} \sum_{t=1}^T KL\left(\P_{\mathrm{Sim}(\rho,(i_1,\dots,i_t))} , \P_{\mathrm{Sim}(\rho',(i_1,\dots,i_t))} \right)
\end{align*}
The simulator will be defined so that the right hand side is always finite for any $T$.
When it is clear from context we will simply write $\P_{\rho}( E )$ or $\P_{\mathrm{Sim}(\rho)}( E )$ to represent $\P_{\mathrm{Alg},\rho}( E )$ or $\P_{\mathrm{Alg},\mathrm{Sim}(\rho)}( E )$, respectively.
Let $\Omega_t = \{I_1,\dots,I_t\}$ denote the history of all arm pulls requested by the player up to time $t$. Note that $\Omega_t$ is a multi-set so that $|\Omega_t|=t$.
\begin{definition}
We say an event $W$ is \emph{truthful} under a simulator $\mathrm{Sim}$ with respect to instance $\rho$ if for all events $E \in \mathcal{F}_T$
\begin{align*}
\P_\rho(E \cap W) = \P_{\mathrm{Sim}(\rho, \Omega_T)}(E \cap W).
\end{align*}
\end{definition}
\begin{lemma}[\cite{simchowitz2017simulator}]
Let $\rho^{(1)}$ and $\rho^{(2)}$ be two instances, $\mathrm{Sim}( \cdot, \cdot)$ be a simulator, and let $W_i$ be two truthful $\mathcal{F}_T$-measureable events under $\mathrm{Sim}( \rho^{(i)}, \Omega_T)$ for $i=1,2$ where $\Omega_T$ is the history of pulls up to a stopping time $T$. Then
\begin{align*}
\P_{\rho^{(1)}}(W_1^c) + \P_{\rho^{(2)}}(W_2^c) \geq \mathrm{TV}(\rho^{(1)}, \rho^{(2)}) - Q\left( KL\left( \P_{\mathrm{Alg},\mathrm{Sim}(\rho^{(1)})}, \P_{\mathrm{Alg},\mathrm{Sim}(\rho^{(2)})} \right) \right)
% \geq \mathrm{TV}(\rho^{(1)}, \rho^{(2)}) - Q\left( \lim_{\tau \rightarrow \infty} \max_{i_1,\dots,i_\tau \in [n]} \sum_{t=1}^\tau KL\left(\mathrm{Sim}( \rho^{(1)},\{i_s\}_{s=1}^t), \mathrm{Sim}( \rho^{(2)},\{i_s\}_{s=1}^t)\right)\right)
\end{align*} 
where $Q(\beta) = \min\{1-\tfrac{1}{2}e^{-\beta}, \sqrt{\beta/2}\}$. 
\end{lemma}

\subsubsection*{Constructing the Simulator}

% Typically the simulator is tailored to a specific set of events, which are fated to be truthful.
Recall the definitions of $\rho,\rho'$ and $\sigma,\sigma'$ from above.
For some $\tau \in \mathbb{N}$ and multiset $\Omega$ of requested arm pulls, define $W_\sigma(\Omega)=\{ \sum_{i \in \Omega} \1\{i \in \sigma\} \leq \tau \}$ and $W_{\sigma'}(\Omega)=\{ \sum_{i \in \Omega} \1\{i \in \sigma'\} \leq \tau \}$.
For these events, an instance $\nu \in \{\rho,\rho'\}$, and any multiset $\Omega_t$ denoting the indices the player has played up to the current time $t$, define a simulator 
\begin{align*}
\mathrm{Sim}(\nu,\Omega_t)_i = \begin{cases}
\nu_i & \text{ if } i \notin \sigma \cup \sigma' \\
\nu_i & \text{ if } i \in \sigma \cup \sigma', \ W_\sigma(\Omega_t) \cap W_{\sigma'}(\Omega_t)  \\ 
\rho_{i} & \text{ if } i \in \sigma, \ W_\sigma^c(\Omega_t) \cup W_{\sigma'}^c(\Omega_t) \\
\rho_{\sigma(\sigma'^{-1}(i))} & \text{ if } i \in \sigma', \ W_\sigma^c(\Omega_t) \cup W_{\sigma'}^c(\Omega_t)
\end{cases}
\end{align*}
where $\sigma(i)$ denotes the $i$th element of $\sigma$ and $\sigma^{-1}(i) \in \{1,\dots,|\sigma|\}$ so that $\sigma( \sigma'^{-1}(i) ) \in \sigma$ for any $i \in \sigma'$. 
Note that $\mathrm{Sim}(\nu,\Omega_t)_i$ and $\nu_i$ potentially differ only on those arms $i \in \sigma'$, and only if $W_\sigma^c(\Omega_t) \cup W_{\sigma'}^c(\Omega_t) \coloneqq \max\{ \sum_{j \in \Omega_t} \1\{j \in \sigma\}, \sum_{j \in \Omega_t} \1\{j \in \sigma'\} \} > \tau$. 
That is, if $\max\{ \sum_{j \in \Omega_t} \1\{j \in \sigma\}, \sum_{j \in \Omega_t} \1\{j \in \sigma'\} \} > \tau$ then  $\mathrm{Sim}(\rho,\Omega_t)_i = \mathrm{Sim}(\rho',\Omega_t)_i$ for all $i \in [n]$.
On the other hand, if $W_\sigma(\Omega_t) \cap W_{\sigma'}(\Omega_t) \coloneqq \max\{ \sum_{j \in \Omega_t} \1\{j \in \sigma\}, \sum_{j \in \Omega_t} \1\{j \in \sigma'\} \} \leq \tau$ then $\mathrm{Sim}(\nu,\Omega_t)_i = \nu$ for $\nu \in \{\rho,\rho'\}$.
Thus, $W_{\sigma^\prime}(\Omega_t)$ is truthful under $\mathrm{Sim}(\rho,\Omega_t)_i$ and $W_\sigma(\Omega_t)$ is truthful under $\mathrm{Sim}(\rho^\prime,\Omega_t)$.
Using these observations, we can easily upper bound the KL divergence: 
% We will apply this lemma in the following way.
% Let $T$ be the time that the algorithm exists with $\widehat{S}$. 
% To calculate the KL divergence between simulators, we observe the following facts:\\
% \textbf{Fact 1}: $\mathrm{Sim}(\rho,\Omega)_i = \mathrm{Sim}(\rho',\Omega)_i$ for all $i \in [n]$ if $\max\{ \sum_{j \in \Omega} \1\{j \in \sigma\}, \sum_{j \in \Omega} \1\{j \in \sigma'\} \} > \tau$.\\
% \textbf{Fact 2}: $\mathrm{Sim}(\rho,\Omega)_i = \mathrm{Sim}(\rho',\Omega)_i$ if $i \not\in \sigma \cup \sigma'$.\\
% We clearly have
\begin{align*}
\max_{i_1,\dots,i_T \in [n]} \sum_{t=1}^T KL\left(\mathrm{Sim}( \rho,\{i_s\}_{s=1}^t), \mathrm{Sim}( \rho',\{i_s\}_{s=1}^t)\right) &\leq \max_{i \in \sigma} \tau KL(\rho_i, \rho_i') + \max_{j \in \sigma'} \tau KL(\rho_j, \rho_j') \\
% &= \max_{i = 1,\dots,|\sigma|} \tau KL(\rho_{\sigma(i)}, \rho_{\sigma'(i)}) + \max_{j = 1,\dots,|\sigma|} \tau KL(\rho_{\sigma'(j)}, \rho_{\sigma(j)}) \\
% &= \max_{i \in \sigma} \tau KL(\rho_i, \rho_{\sigma'(\sigma^{-1}(i))}) + \max_{j \in \sigma'} \tau KL(\rho_j, \rho_{\sigma(\sigma'{-1}(j))})2 \\
&=  \max_{i = 1,\dots,\ell} \tau (\mu_{\sigma(i)} - \mu_{\sigma'(i)})^2.
\end{align*}

As shown in \cite[Lemma 1]{simchowitz2017simulator} averaging over all permutations is equivalent to constructing a symmeterized version of the algorithm such that given any bandit instance, the algorithm randomly permutes the arms internally and then after making its set selection, returns the set inverted by the randomly chosen permutation.
This modified algorithm is symmetric in the sense that 
\begin{align*}
    \P_\rho( (i_1,\dots,i_T, s)=(I_1,\dots,I_T, \widehat{S})) = \P_{\pi(\rho)}( (i_1,\dots,i_T,, s)=(\pi(I_1),\dots,\pi(I_T), \pi(\widehat{S}))) .
\end{align*}
In what follows, we assume the algorithm is symmetric which, in particular, implies
\begin{align*}
\P_\rho(W_{\sigma'}^c) + \P_{\rho'}(W_{\sigma}^c) &= 2 \P_\rho(W_{\sigma'}^c).
\end{align*}
Putting all the pieces together we have
\begin{align*}
\P_{\rho}\left( \sum_{i \in \sigma'} T_i > \tau \right) = \P_{\rho}(W_{\sigma'}^c) 
&= \frac{1}{2}\left( \P_\rho(W_{\sigma'}^c) + \P_{\rho'}(W_{\sigma}^c) \right) \\
&\geq \frac{1}{2}\left( 1-\delta-e^{-1} - \sqrt{\tau \max_{i = 1,\dots,\ell}  (\mu_{\sigma(i)} - \mu_{\sigma'(i)})^2/2} \right) \\
&> \frac{1}{2}(1/8 - \delta)
\end{align*}
if $\tau = \frac{1}{2\max_{i = 1,\dots,\ell} (\mu_{\sigma(i)} - \mu_{\sigma'(i)})^2}$.
By Markov's inequality, $\E_\rho[ \sum_{i \in \sigma'} T_i ] \geq \tau \P_{\rho}\left( \sum_{i \in \sigma'} T_i > \tau \right)$.
Noting that $\sigma' \subset [n]\setminus [m]$ was arbitrary, we apply the above calculation for all connected subsets of size $\lceil m/k \rceil$
\begin{align*}
\E_{\rho}\left[ \sum_{i = m+1}^n T_i \right] &\geq \frac{1}{4}({1}/{8} - \delta) \sum_{r=1}^{(n-m)k/m} (\mu_{1} - \mu_{m+ r m/k})^{-2} \\
&\geq \frac{1}{4}({1}/{8} - \delta) \frac{k}{m} \sum_{i=m + m/k + 1}^{n} (\mu_{1} - \mu_{i})^{-2} \\
&\geq \frac{1}{4}({1}/{8} - \delta) \left[-(\mu_{1} - \mu_{m+1})^{-2} +  \frac{k}{m} \sum_{i=m + 1}^{n} (\mu_{1} - \mu_{i})^{-2} \right] \\
&\geq \frac{1}{64} \left[-(\mu_{1} - \mu_{m+1})^{-2} +  \frac{k}{m} \sum_{i=m + 1}^{n} (\mu_{1} - \mu_{i})^{-2} \right]
\end{align*}
where the last line follows since $\delta \in (0,\frac{1}{16})$.

\section{Additional Algorithms}
\label{sec:additional_algorithms}

In this section, we briefly introduce two additional algorithms that are very similar to the Algorithm \ref{fdr_alg} presented earlier but have stronger guarantees for the task of identifying means above a threshold. A FWER-TPR (family-wise error rate-true positive rate) guarantee outputs a set $\Q_t$ such that $\P(\exists t : \Q_t \cap \H_0 \neq \emptyset) \leq c \delta$ and $\E[|\Q_t \cap \H_1|] \geq (1-\delta)k$ for large enough $t$. A FWER-FWPD (family-wise error rate-family-wise probability of detection) guarantee is stronger since it requires that the outputted set $\calR_t$ satisfies $\P(\exists t : \calR_t \cap \H_0 \neq \emptyset) \leq c \delta$ and $| \calR_t \cap \H_1| \geq k$ for large enough $t$. For more formal examples of these guarantees, see Theorems  \ref{main_thm_fwer} and \ref{fwer_fdwp_thm}.

The algorithm suggests different sets depending on the objective. If FWER-TPR is desired, the algorithm maintains a set  $\mathcal{Q}_{t}$ and adds arms whose lower confidence bounds are above the threshold $\mu_0$ (Line~\ref{lin:fwer_tpr_set}). If FWER-FWPD is the goal, then an additional arm $J_t$ is pulled each time based on an upper confidence bound criterion and arms are accepted into the set $\calR_{t+1}$ (Line~\ref{lin:fwer_fwpd_set}) if their lower confidence bound is above the threshold $\mu_0$. 

\begin{algorithm}[t]
\caption{Infinite UCB Algorithm: FWER-TPR and FWER-FWPD}
\label{fwer_alg}
\footnotesize
\begin{algorithmic}[1]
\STATE $\delta_r = \frac{\delta}{r^2}$, $\delta^\prime_r = \frac{\delta_r}{6.4 \log(36/\delta_r)}$ $R_0 = 0$, $\ell = 0$, $\mathcal{S}_{0} = \emptyset$, $\mathcal{Q}_{0} = \emptyset$
\FOR{ $t =1,2, \ldots$}
\IF{$t \geq 2^\ell \ell$}
\STATE Draw a set $A_{\ell+1}$ uniformly at random from $\binom{[n]}{M_{\ell+1}}$, where $M_{\ell} := n \wedge 2^{\ell}$
\STATE $\ell = \ell + 1$
\ENDIF
\STATE $R_t = 1+  R_{t-1} \cdot \1\{R_{t-1} < \ell \} $
\IF{there exists $i \in A_{R_t} \setminus \mathcal{S}_{t}$ such that $T_{i,R_t}(t) = 0$}
\STATE  Pull an arm $I_t$ belonging to $\{i \in A_{R_t} \setminus \mathcal{S}_{t} : T_{i,R_t}(t) = 0 \}$
\ELSIF{FWER-TPR}
\STATE Pull arm $I_t = \text{argmax}_{i \in A_{R_t} \setminus \mathcal{Q}_{t}} \widehat{\mu}_{i,R_t, T_{i,R_t}(t)} + U(T_{i,R_t}(t), \delta)$
\STATE $\mathcal{Q}_{t+1} = \mathcal{Q}_{t} \cup  \{i \in A_{R_t} : \widehat{\mu}_{i,R_t,T_{i,R_t}(t)} - U(T_{i,R_t}(t), \frac{\delta}{|A_{R_t}| R_t^2} ) \geq \mu_0 \}$\hfill \texttt{\small\color{blue} \% FWER Thm.\ref{main_thm_fwer}}\label{lin:fwer_tpr_set}
\ELSIF{FWER-FWPD}
\STATE $\xi_{t,R_t} = \max\{2 |\S_t \cap A_{R_t} |, \frac{5}{3(1-4\delta_{R_t})} \log(1/\delta_{R_t})R_t^2 \}$
\STATE Pull arm $I_t = \text{argmax}_{i \in A_{R_t} \setminus \mathcal{S}_{t}} \widehat{\mu}_{i,R_t, T_{i,R_t}(t)} + U(T_{i,R_t}(t), \frac{\delta}{\xi_{t,R_t}})$
\STATE $s(p) = \{i \in A_{R_t} : \widehat{\mu}_{i,R_t, T_{i,R_t}(t)} - U(T_{i,R_t}(t), \frac{p}{|A_{R_t}| } \delta^\prime_{R_t} \geq \mu_0 \}$
\STATE $\mathcal{S}_{t+1} = \mathcal{S}_{t} \cup s(\widehat{p})$ where $\widehat{p} = \text{max}\{p \in [|A_{R_t}|] : |s(p)| \geq p \}$ 
\IF{$\S_t \cap A_{R_t} \neq \emptyset$} 
\STATE $\nu_{t,R_t} = \max(|\S_t \cap A_{R_t}|,1)$
\STATE Pull arm $J_t = \text{argmax}_{i \in \S_t \cap A_{R_t}  \setminus \calR_t}   \widehat{\mu}_{i,R_t,T_{i,R_t}(t)} + U(T_{i,R_t}(t), \frac{\delta_{R_t}}{\nu_{t,R_t}})$
\STATE $\chi_{t,R_t} = |A_{R_t}| - (1-2\delta^\prime_{R_t}(1+4\delta^\prime_{R_t}))|\S_t \cap A_{R_t}| +\frac{4(1+4\delta^\prime_{R_t})}{3} \log(5 \log_2(| A_{R_t}| /\delta^\prime_{R_t})/\delta^\prime_{R_t})$
\STATE $\calR_{t+1} = \calR_t \cup \{ i \in \S_t \cap A_{R_t} :  \widehat{\mu}_{i, R_t, T_{i,R_t}(t)} - U(T_{i,R_t}(t),\frac{\delta}{\chi_{t,R_t}}) \geq \mu_0 \}$\hfill \texttt{\small\color{blue} \% FWER Thm.\ref{fwer_fdwp_thm}}\label{lin:fwer_fwpd_set}
\ENDIF
\ENDIF
\ENDFOR
\end{algorithmic}
\end{algorithm}

\section{Proofs of Upper Bounds}
\label{sec:upper_bound_proofs}

The proofs for the FDR-TPR  result (the proof of Theorem~\ref{main_thm_fdr_log} in Section \ref{fdr_tpr_proof}) should be read first. Then, one can read the proofs for any of the other results. We introduce some notation that we use throughout the proofs. We use $c$ to denote a positive constant whose value may change from line to line.
Define
\begin{align*}
\rho_{i,r} &= \sup \{ \rho \in (0,1] : \cap_{t=1}^\infty \{|\widehat{\mu}_{i,r,t} - \mu_i| \leq U(t,\rho)\} \}.
\end{align*}
We note that $\{\rho_{i,r}\}_{i \in [n], r \in \N}$ are independent and $\P(\rho_{i,r} \leq \delta) \leq \delta$ since by definition of $U(\cdot, \cdot)$ for any bracket $r \in \mathbb{N}$ and $\alpha \in (0,1)$, $\P( \cap_{t=1}^\infty \{|\widehat{\mu}_{i,r,t} - \mu_i| \leq U(t,\alpha)) \geq 1-\alpha$. We define
\begin{align*}
\I_r = \{i \in \H_{1} \cap A_r : \rho_{i,r} \leq \delta \}.
\end{align*}
to be those arms in bracket $r$ whose empirical means concentrate well in the sense that $\rho_{i,r} \leq \delta$. We also define $U^{-1}(\gamma,\delta) = \min(t : U(t,\delta) \leq \gamma)$. It can be shown for a sufficiently large constant $c$ that $U^{-1}(\gamma,\delta) \leq c \gamma^{-2} \log(\log(\gamma^{-2})/\delta)$. Recall that we make that simplifying assumption that $\mu_0, \mu_1, \ldots, \mu_n \in [0,1]$ and that we define $\log(x) \coloneqq \max(\ln(x),1)$. 

We note that although all of our upper bounds apply to the expectation of a stopping time, it is possible to obtain high-probability bounds by arguing that with high probability there is an appropriately sized bracket with enough ``good" arms, e.g., an $\epsilon$-good arm. Unfortunately, this argument would lead to an upper bound that scales as $\log^2(1/\delta)$ and would lose the dependence on the individual gaps of the arms with mean greater than $\mu_1 - \epsilon$ or $\mu_0$. 

\subsection{Proof of FDR-TPR}
\label{fdr_tpr_proof}
Recall the relevant notation that $\Delta_{i,j} \coloneqq \mu_i - \mu_j$ and $\Delta_{j,0} \coloneqq \mu_j -\mu_0$. We restate Theorem \ref{thm_fdr_paper} from the main body of the paper with the doubly logarithmic terms. We only consider the gap-independent upper bound here; in the following section, we will prove a stronger result, which implies the the gap-dependent upper bound.

\begin{theorem}
\label{main_thm_fdr_log}
Let $\delta \leq (0,1/40)$. Let $k \in [|\H_1|]$.
 For all $j \in [m]$, define
\begin{align*}
\fdrHt(\mu_0;j) & \coloneqq  \frac{n}{j}k \Delta_{j,0}^{-2} \log\left( \log(\tfrac{n}{j}k) \log( \Delta_{j,0}^{-2})/\delta\right) .
\end{align*}
Let $(\mc{F}_t)_{t \in \N}$ be the filtration generated by playing Algorithm \ref{fdr_alg} on problem $\rho$. Then, Algorithm \ref{fdr_alg} has the property that for all $t \in \N$, $\E[ \frac{|\S_t \cap \H_0|}{|\S_t| \wedge 1}] \leq 2 \delta$ and there exists a stopping time $\tau_k$ wrt $(\mc{F})_{t \in \N}$ such that
\begin{align}
\E[\tau_k] & \leq c \min_{k \leq j \leq m}  \fdrHt(\mu_0;j) \log(\fdrHt(\mu_0;j)) \label{main_thm_res_ind} 
\end{align}
where $c$ is a universal constant and for all $t \geq \tau_k$, $\E[|\S_t \cap \H_1|] \geq (1-\delta)k$. 
\end{theorem}

We briefly sketch the proof. Let $j_0 \in \{k, \ldots, m \}$ minimize the upper bound \eqref{main_thm_res_ind}. Then, there exists a bracket $r_0$ with size  $\Theta(\frac{n}{j_0} k)$ such that with constant probability $A_{r_0}$ has at least $k$ arms in $[j_0]$ and the empirical means concentrate well enough (defined formally in Lemma \ref{bracket_lemma} as the event $E_{r_0} \coloneqq E_{r_0} \cap E_{0,r_0} \cap E_{1,r_0}$). The argument controls $\E[\tau_k]$ by partitioning the sample space according to which bracket $r_0+s$ is the first such that the good event $E_{r_0+s}$ occurs, i.e., according to $\{ E_{r_0}, E_{r_0}^c \cap E_{r_0+1}, E_{r_0}^c \cap E_{r_0+1}^c \cap  E_{r_0+2}, ... \}$. Lemma \ref{bracket_lemma} shows that $\E[ \1\{ E_{r_0} \}  \tau_{k}]$ has the same upper bound as \eqref{main_thm_res_ind} and that $\E[ \1\{ E_{r_0+s} \}  \tau_{k}]$ has an upper bound that is larger than line \eqref{main_thm_res_ind} by a factor exponential in $s$. On the other hand, because the brackets are independent and growing exponentially in size, the probability of $ E_{r_0+s} \cap(\cap_{r=0}^{s-1}E_{r_0+r}^c) $ decreases exponentially in $s$, enabling control of the exponential increase in $\E[ \1\{ E_{r_0+s} \}  \tau_{r_0+s,k}]$ and, by extension, $\E[\tau_k]$.

Lemma \ref{fdr_lemma} bounds the false discovery rate of Algorithm \ref{fdr_alg}.

\begin{lemma}
\label{fdr_lemma}
For all $t \in \N$, $\E[ \frac{|\S_t \cap \H_0|}{|\S_t| \wedge 1}] \leq 2 \delta$.
\end{lemma}

\begin{proof}
\begin{align*}
\E[ \frac{|\S_t \cap \H_0|}{|\S_t| \wedge 1}] & \leq \E[ \frac{\sum_{l=1}^{\infty} |\S_t \cap A_l \cap \H_0|}{  |\S_t| \wedge 1}] \\
& \leq  \sum_{l=1}^{\infty}  \E[ \frac{ |\S_t \cap A_l \cap \H_0|}{   |\S_t \cap A_l|\wedge 1}] \\
& \leq \delta \sum_{l=1}^\infty \frac{1}{l^2} \\
& = \delta \frac{\pi^2}{6}
\end{align*}
where we used Lemma 1 of \cite{jamieson2018bandit}.

\end{proof}

Lemma \ref{bracket_lemma}, below, is the key result for establishing Theorem \ref{main_thm_fdr_log}. For $k \in [|\H_1|]$ and $j_0 \in \{k,\ldots, |\H_1|\}$, it bounds the expected number of iterations that it takes a bracket $r$ (of size at least $2^r \geq k$) to add $k$ arms to the set $\S_t$ when the events $E_r \cap E_{0,r} \cap E_{1,r}$ occur where
\begin{align*}
E_r & = \{|[j_0]  \cap A_r| \geq k\}, \\
E_{0,r} & = \{ \sum_{i \in \H_0  \cap A_r} \Delta_{j_0,i}^{-2} \log(\frac{1}{\rho_{i,r} }) \leq 5 \sum_{i \in \H_0  \cap A_r} \Delta_{j_0,i}^{-2} \log(\frac{1}{\delta }) \}, \\
E_{1,r} & = \{ \sum_{i \in [j_0]  \cap A_r} \Delta_{i \vee j_0,0}^{-2} \log(\frac{1}{\rho_{i,r} }) \leq 5 \sum_{i \in [j_0]  \cap A_r} \Delta_{i \vee j_0,0}^{-2} \log(\frac{1}{\delta }) \}.
\end{align*}
Event $E_r $ says that there are at least $k$ arms in $A_r$ with $\mu_i \geq \mu_{j_0}$. The event $E_{0,r}$ says that the empirical means of the arms in $\H_0  \cap A_r$ concentrate well on the whole; event $E_{1,r}$ makes the analogous claim about $[j_0]  \cap A_r$. We remark that the the events $E_{0,r}$ and $E_{1,r}$ allow us to avoid using a union bound.

\begin{lemma} 
\label{bracket_lemma}
Fix $\delta \in (0,1/40)$, $k \in [|\H_1|]$, $j_0 \in \{k, \ldots, |\H_1|\}$, and $r \in \N$ such that $2^r \geq k$. 
Let $(\mc{F}_t)_{t \in \N}$ be the filtration generated by playing Algorithm \ref{fdr_alg} on problem $\rho$. Then, there exists a stopping time $\tau_k$ wrt $(\mc{F}_t)_{t \in \N}$ such that for all $t \geq \tau_k$, $\E[|\S_t \cap \H_1|] \geq (1-\delta)k$, and
\begin{align}
\E[ \1\{ E_r \cap E_{0,r} \cap E_{1,r} \}  \tau_{k}] & \leq c [2^{r-1}(r-1)+ |A_r|  \Delta_{j_0,0}^{-2} \log(  r \frac{\log(\Delta_{j_0,0}^{-2})}{\delta})  \log(|A_r|  \Delta_{j_0,0}^{-2} \log(  r \frac{\log(\Delta_{j_0,0}^{-2})}{\delta})) ] \label{bracket_lemma_ind}
\end{align}
where $c$ is a universal constant.
\end{lemma}

\begin{proof}
\textbf{Step 1: Define stopping time.} Define
\begin{align*}
\tau_k & = \min(t \in \N \cup \{\infty\} :|A_r \cap [j_0]| \geq k \text{ and } \I_r \cap A_r \cap \H_1 \subset \S_t). 
\end{align*}

Observe that for all $t \geq \tau_k$, $\E[|\S_t \cap \H_1|] \geq (1-\delta)k$ since for $t \geq \tau_k$ 
\begin{align*}
\E[|\S_t \cap \H_1|] \geq \E[|\I_r \cap A_r \cap \H_1|] \geq (1-\delta) |A_r \cap \H_1| \geq (1-\delta)k.
\end{align*}

\textbf{Step 2: Relate to bracket $r$.} 

In the interest of brevity, define $E \coloneqq  E_r \cap E_{0,r} \cap E_{1,r}$ and since we will only focus on bracket $r$, write $\widehat{\mu}_{i,t}$, $T_i(t)$, $\I$, and $\rho_i$  instead of $\widehat{\mu}_{i,r,t}$, $T_{i,r}(t)$, $\I_r$, and $\rho_{i,r}$. We will bound the number of rounds until $\I \cap A_r \cap \H_1 \subset \S_t$. Define
\begin{align*}
T & = |\{ t \in \N : \I \cap A_r \cap [j_0] \not \subset \S_t \text{ and } R_t = r\}|,
\end{align*}
i.e., the number of rounds that the algorithm works on the $r$th bracket and $\I \cap A_r \cap \H_1 \not \subset \S_t$. 

Next, we bound the number of brackets $r+s$ that are opened before $ \I \cap A_r \cap [j_0] \subset \S_t$. The $r+1$ bracket is opened after bracket $r$ is sampled $2^r$ times and similarly the $r+s$th bracket is opened after bracket $r$ is sampled 
% \begin{align*}
$\sum_{i=0}^{s-1} 2^{r+i} \geq 2^{r+s-1}$  
% \end{align*}
times. Thus,
\begin{align*}
 2^{r+s-1} & \leq T \implies 
 r+s-1 \leq \log(T).
 \end{align*}
So while $ \I \cap A_r \cap \H_1 \not \subset \S_t$, every time bracket $r$ is sampled, at most  $\log(T)$ total brackets are sampled. Thus, we have that once the algorithm starts working on bracket $r$, after 
\begin{align}
\log(T) T \label{fdr_log_bound}
\end{align}
additional rounds, we have that $ \I \cap A_r \cap [j_0] \subset \S_t$.

We note that after $2^{r-1} (r-1)$ rounds, the algorithm starts working on bracket $r$. Thus,
\begin{align}
\1\{ E \} \tau_{k} & \leq  [2^{r-1}(r-1)+ \1\{ E \} \log(T)  T] \nonumber\\
& =  [2^{r-1}(r-1)+  \log(\1\{ E \} T)  \1\{ E \} T] \label{no_expect_bound}
\end{align}

\textbf{Step 3: Bounding $\1\{ E \} T$.}
Note that we can write
\begin{align*}
 \1\{ E \} T  &= \1\{ E \}\sum_{t=1}^\infty \1 \{  [j_0] \cap \I \cap A_r \not \subset \S_{t}, R_t = r \} \\
 &= \1\{ E \}\sum_{t: R_t=r}^\infty \1 \{  [j_0] \cap \I \cap A_r \not \subset \S_{t} \} \\
 & \leq \1\{ E \}\sum_{t: R_t=r}^\infty \1 \{[j_0] \cap \I \cap A_r \not \subset \S_{t} , I_t \in \H_{0}\}\\
 &\hspace{1in} +  \1 \{[j_0] \cap \I \cap A_r \not \subset \S_{t} , I_t \in \H_{1} \cap [j_0]^c\} + \1\{ I_t \in [j_0]\} \\
  & \leq \1\{ E \}\sum_{t: R_t=r}^\infty \1 \{[j_0] \cap \I \cap A_r \not \subset \S_{t} , I_t \in \H_{0}\}\\
  &\hspace{1in} +  \1 \{[j_0] \cap \I \cap A_r \not \subset \S_{t} , I_t \in \H_{1} \cap [j_0]^c, \widehat{\mu}_{I_t, T_{I_t}(t)} < \mu_0 + \frac{\Delta_{j_0,0}}{2} \} \\
  &\hspace{1in} + \1 \{I_t \in \H_{1} \cap [j_0]^c, \widehat{\mu}_{I_t, T_{I_t}(t)} \geq \mu_0 + \frac{\Delta_{j_0,0}}{2} \} +  \1\{ I_t \in [j_0]\} 
\end{align*}
To begin, we bound the first sum. 

For any $l \in \I \cap [j_0] \cap A_r$ we have $\rho_l \geq \delta$ by definition, so
\begin{align*}
\widehat{\mu}_{l,T_l(t)} + U(T_l(t), \delta) \geq \mu_l - U(T_l(t), \rho_l) + U(T_l(t), \delta) \geq \mu_l \geq \mu_{j_0}.
\end{align*}
For any $i \in \H_0  \cap A_r$,
\begin{align*}
\widehat{\mu}_{i,T_i(t)} + U(T_i(t), \delta) \leq \mu_i + U(T_i(t), \rho_i)+ U(T_i(t), \delta) \leq  \mu_i + 2U(T_i(t), \rho_i \delta).
\end{align*}
Thus, $\widehat{\mu}_{i,T_i(t)} + U(T_i(t), \delta) \leq \mu_{j_0}$ if $T_i(t) \geq U^{-1}(\frac{ \Delta_{j_0,i}}{2}, \rho_i \delta)$, so that arm $i$ would not be pulled this many times as long as $[j_0] \cap \I \cap A_r \not \subset \S_{t}$. Thus, 
\begin{align}
 \1\{ E \}\sum_{t: R_t=r}^\infty \1 \{[j_0] \cap \I \cap A_r \not \subset \S_{t} , I_t \in \H_{0}\} & \leq \1\{ E \} \sum_{i \in \H_0  \cap A_r} U^{-1}(\frac{ \Delta_{j_0,i}}{2}, \rho_i \delta) \nonumber \\
& \leq  \1\{ E \} \sum_{i \in \H_0  \cap A_r} c \Delta_{j_0,i}^{-2} \log(\frac{\log( \Delta_{j_0,i}^{-2})}{\delta \rho_i}) \nonumber \\
& =  \1\{ E \} \sum_{i \in \H_0  \cap A_r} c \Delta_{j_0,i}^{-2} \log(\frac{\log( \Delta_{j_0,i}^{-2})}{\delta }) + c \Delta_{j_0,i}^{-2} \log(\frac{1}{\rho_i }) \nonumber \\
& \leq  \1\{ E \} \sum_{i \in \H_0  \cap A_r} c^\prime \Delta_{j_0,i}^{-2} \log(\frac{\log( \Delta_{j_0,i}^{-2})}{\delta }) \nonumber \\
& \leq \sum_{i \in \H_0  \cap A_r} c^\prime\Delta_{j_0,i}^{-2} \log(\frac{\log( \Delta_{j_0,i}^{-2})}{\delta }) \label{lilucb_arg}
\end{align}
where the second to last inequality follows from $E_{r} \subseteq E$.

Next, we consider the second sum.  If $[j_0] \cap \I \cap A_r \not \subset \S_{t} $, for any arm $i$ satisfying $\widehat{\mu}_{i, T_i(t)} < \mu_0 + \frac{\Delta_{j_0,0}}{2}$, we have that 
\begin{align*}
\widehat{\mu}_{i, T_i(t)} + U(T_i(t), \delta) <  \mu_0 + \frac{\Delta_{j_0,0}}{2}+ U(T_i(t), \delta) 
\end{align*}
so that if $T_i(t) \geq U^{-1}(\frac{\Delta_{j_0,0} }{2}, \delta)$, then $\widehat{\mu}_{i, T_i(t)} + U(T_i(t), \delta) < \mu_{j_0}$ and therefore arm $i$ is not pulled again until $[j_0] \cap \I \cap A_r  \subset \S_{t} $. Thus, 
\begin{align*}
\hspace{1in}&\hspace{-1in}\sum_{t: R_t=r}^\infty  \1 \{[j_0] \cap \I \cap A_r \not \subset \S_{t} ,  I_t \in \H_{1} \cap [j_0]^c \cap A_r, \widehat{\mu}_{I_t, T_{I_t}(t)} < \mu_0 + \frac{\Delta_{j_0,0}}{2} \} \\
& \leq  \sum_{i \in  \H_{1} \cap [j_0]^c \cap A_r} U^{-1}(\frac{\Delta_{j_0,0} }{2}, \delta) \\
& \leq c |\H_{1} \cap [j_0]^c \cap A_r| \Delta_{j_0,0}^{-2} \log(\frac{\log(\Delta_{j_0,0}^{-2})}{\delta}).
\end{align*}
Next, we bound the final summands
\begin{align*}
 \1\{ E \}\sum_{t: R_t=r}^\infty \1 \{I_t \in \H_{1} \cap [j_0]^c, \widehat{\mu}_{I_t, T_{I_t}(t)} \geq \mu_0 + \frac{\Delta_{j_0,0}}{2} \} +  \1\{ I_t \in [j_0]\} .
\end{align*}
Let $p \leq |A_r|$. If $j \in \H_{1} \cap [j_0]^c  \cap A_r$ and $\widehat{\mu}_{j, T_{j}(t)} \geq \mu_0 + \frac{\Delta_{j_0,0}}{2}$, then
\begin{align*}
\widehat{\mu}_{j, T_{j}(t)}  - U(T_j(t), \delta^\prime_r \frac{p}{|A_r| }) \geq \mu_0 + \frac{\Delta_{j_0,0}}{2}  - U(T_j(t), \delta^\prime_r \frac{p}{|A_r| })
\end{align*}
so that $\widehat{\mu}_{j, T_{j}(t)}  - U(T_j(t), \delta^\prime_r \frac{p}{|A_r| }) \geq \mu_0$ if $T_i(t) \geq U^{-1}(\frac{\Delta_{j_0,0}}{2}, \delta^\prime_r \frac{p}{|A_r| })$, which implies that $j \in s(p)$. 

Next, if $j \in [j_0]  \cap A_r$, then
\begin{align*}
\widehat{\mu}_{j, T_{j}(t)}  - U(T_j(t), \delta^\prime_r \frac{p}{|A_r| }) & \geq \mu_j  - U(T_j(t), \rho_j) - U(T_j(t), \delta^\prime_r \frac{p}{|A_r| }) \\
& \geq \mu_j - 2 U(T_j(t), \rho_j \delta^\prime_r \frac{p}{|A_r|}) 
\end{align*}
so that  $\widehat{\mu}_{j, T_{j}(t)}  - U(T_j(t), \delta^\prime_r \frac{p}{|A_r| }) \geq \mu_0$ if $T_i(t) \geq U^{-1}(\frac{\mu_j - \mu_0}{2}, \rho_j \delta^\prime_r \frac{p}{|A_r| })$, which implies that $j \in s(p)$. 

While there is \emph{some} $p$ associated with each arm when it is added to $s(p)$ and then consequently to $\S_t$, we don't know the order in or time at which particular arms are added.
However, in the worst case, the arms of $\H_1$ are added one at a time to $\S_t$ instead of in a big group so that the first reqires $p=1$, the second $p=2$, etc.  
Letting $\Gamma = \{ f : f: \H_1 \longrightarrow [|H_1|] \text{ is a bijection} \}$, 
\begin{align}
  \1\{ E\}  &\sum_{t:R_t = r}^\infty  \1 \{I_t \in \H_{1} \cap [j_0]^c, \widehat{\mu}_{I_t, T_{I_t}(t)} \geq \mu_0 + \frac{\Delta_{j_0,0}}{2} \}  +  \1\{ I_t \in [j_0]\}   \nonumber\\
& \leq  \1\{ E\} c \max_{\sigma \in \Gamma }  \bigg( \sum_{j \in \H_{1} \cap [j_0]^c \cap A_r} U^{-1}(\frac{\Delta_{j_0,0}}{2}, \delta^\prime_r \frac{\sigma(j)}{|A_r| }) + \sum_{j \in [j_0] \cap A_r}  U^{-1}(\frac{\mu_j - \mu_0}{2}, \rho_j \delta^\prime_r \frac{\sigma(j)}{|A_r| }) \bigg) \nonumber\\
& \leq  \1\{ E\} c \max_{\sigma \in \Gamma }  \bigg(\sum_{j \in \H_{1} \cap [j_0]^c \cap A_r} \Delta_{j_0,0}^{-2} \log( \frac{|A_r| }{\sigma(j)}  \frac{\log(\Delta_{j_0,0}^{-2})}{\delta^\prime_r}) + \sum_{j \in [j_0] \cap A_r}  \Delta_{j,0}^{-2} \log( \frac{|A_r| }{\sigma(j)}  \frac{\log(\Delta_{j,0}^{-2})}{\rho_j \delta^\prime_r}) \bigg)\nonumber\\
& =  \1\{ E\} c \max_{\sigma \in \Gamma } \bigg( \sum_{j \in \H_{1} \cap [j_0]^c \cap A_r}\Delta_{j_0,0}^{-2} \log( \frac{|A_r| }{\sigma(j)}  \frac{\log(\Delta_{j_0,0}^{-2})}{\delta^\prime_r})  \nonumber\\
&\hspace{1in}+ \sum_{j \in [j_0] \cap A_r}  \Delta_{j,0}^{-2} \log( \frac{|A_r| }{\sigma(j)}  \frac{\log(\Delta_{j,0}^{-2})}{ \delta^\prime_r}) + \sum_{j \in [j_0] \cap A_r}\Delta_{j,0}^{-2} \log(\frac{1}{\rho_j}) \bigg)\nonumber\\
& =  \1\{ E\} c \max_{\sigma \in \Gamma } \bigg( \sum_{j \in \H_{1} \cap [j_0]^c \cap A_r} \Delta_{j_0,0}^{-2} \log( \frac{|A_r| }{\sigma(j)}  \frac{\log(\Delta_{j_0,0}^{-2})}{\delta^\prime_r})  \nonumber\\
&\hspace{1in}+ \sum_{j \in [j_0] \cap A_r}  \Delta_{j,0}^{-2} \log( \frac{|A_r| }{\sigma(j)}  \frac{\log(\Delta_{j,0}^{-2})}{ \delta^\prime_r}) + 5  \sum_{j \in [j_0] \cap A_r}  \Delta_{j,0}^{-2} \log(\frac{1}{\delta}) \bigg) \nonumber\\
& \leq c^\prime \max_{\sigma \in \Gamma } \sum_{i \in \H_{1}  \cap A_r}  \Delta_{i \vee j_0,0}^{-2} \log( \frac{|A_r| }{\sigma(i)} r^2 \frac{\log(\Delta_{i \vee j_0,0}^{-2})}{\delta}) \label{gap_dep_bound} \\
& \leq c^\prime  \sum_{i =1}^{|\H_{1}  \cap A_r|}  \Delta_{ j_0,0}^{-2} \log( \frac{|A_r| }{i} r^2 \frac{\log(\Delta_{ j_0,0}^{-2})}{\delta}) \nonumber\\
& \leq c^{\prime \prime} |A_r|  \Delta_{ j_0,0}^{-2} \log(  r \frac{\log(\Delta_{ j_0,0}^{-2})}{\delta}) \label{gap_indep_bound}
\end{align}
where the last line follows from the fact that for any $p \leq |A_r|$, 
$\sum_{i =1}^p  \log( \frac{|A_r| }{i} ) \leq |A_r|$. 

\textbf{Step 4: finishing bound \eqref{bracket_lemma_ind}.}
Using lines \eqref{gap_indep_bound} and \eqref{no_expect_bound},
\begin{align*}
\1\{ E \} \tau_{k} \leq c^\prime [2^{r-1}(r-1)+  \log(|A_r|  \Delta_{j_0,0}^{-2} \log(  r \frac{\log(\Delta_{j_0,0}^{-2})}{\delta}))  |A_r|  \Delta_{j_0,0}^{-2} \log(  r \frac{\log(\Delta_{j_0,0}^{-2})}{\delta})]
\end{align*}
deterministically, which yields line \eqref{bracket_lemma_ind}.

\end{proof}

\begin{proof}[Proof of Theorem \ref{main_thm_fdr_log}]
As in the proof of Lemma \ref{bracket_lemma}, define
\begin{align*}
\tau_k & = \min(t \in \N \cup \{\infty\} : \exists s \text{ such that } |A_s \cap [j_0]| \geq k \text{ and } \I_s \cap A_s \cap \H_1 \subset \S_t), \\
\tau_{k}^{(r)} & = \min(t \in \N \cup \{\infty\} :  |A_r \cap [j_0]| \geq k \text{ and } \I_r \cap A_r \cap \H_1 \subset \S_t)
\end{align*}
As was argued in Step 1 of the proof of Lemma \ref{bracket_lemma}, for all $t \geq \tau_k$, $\E[|\S_t \cap \H_1|] \geq (1-\delta)k$.

\textbf{Step 1: A lower bound on the probability of a good event.}
Let $j_0 \in \{k,\ldots,m\}$ minimize \eqref{main_thm_res_ind}. Define $E_r = E_r \cap E_{0,r} \cap E_{1,r}$. We note that since $\{\rho_{i,r}\}_{i \in [n], r \in \N}$ and the brackets $\{A_r\}_{r \in \N}$ are independent, $\{E_r\}_{r \in \N}$ are independent events.  Let $r_0$ be the smallest integer such that
\begin{align*}
\min(40 \frac{n}{j_0}k,n) \leq 2^{r_0} \leq 80 \frac{n}{j_0}k,
\end{align*}
Note that if $2^{r_0} \geq n$, then the bracket $r_0$ has $n$ arms.

Next, we bound $\P(E_{r_0}^c)$. If $2^{r_0} \geq n$, then $\P(E_{r_0}^c) = 0$, so assume that $2^{r_0} < n$. Note that since the elements of $A_{r_0}$ are chosen uniformly from $[n]$ and $|A_{r_0}| = 2^{r_0} \geq 40 \frac{n}{j_0}k$ we have that
\begin{align*}
\E[|[j_0]  \cap A_{r_0}|] &= \frac{j_0}{n}|A_{r_0}| \\
&\geq 40 k. 
\end{align*}
Then, by a Chernoff bound for hypergeometric random variables,
\begin{align*}
\P(|[j_0] \cap A_{r_0}| \leq  20  k ) \leq \exp(- \frac{1}{8} 40 k ) \leq \exp(-5). 
\end{align*}
Thus, $E_{r_0}$ occurs with probability at least $1-\exp(-5)$. 
Furthermore, we note that for any $r \geq r_0$, $\P(E_r^c) \leq \exp(-5)$.

Furthermore, by Lemma 8 of \cite{jamieson2018bandit}, for any $r \in \N$ and $i =0,1$,
\begin{align*}
\P(E_{i,r}^c) = \E [\P(E_{i,r}^c|A_r)] \leq \delta. 
\end{align*}
Finally, note that for every $r \geq r_0$ and any $\delta \in (0,1/40)$ we have 
\begin{align*}
\P(E_r^c) \leq  \exp(-5) + 2\delta \leq \frac{1}{16}.
\end{align*}
Furthermore, we claim that $\P(\cap_{l=r_0}^{\infty} E_l^c) = 0$. Let $s \geq r_0$; then, using the independence between brackets,
\begin{align*}
\P(\cap_{l=r_0}^{\infty} E_l^c) \leq \P(\cap_{l=r_0}^{s} E_l^c) = \frac{1}{16^s} \longrightarrow 0
\end{align*}
as $s \longrightarrow \infty$, proving the claim.

\textbf{Step 2: Gap-Independent bound on the number of samples.}
For the sake of brevity, write $\tau$ instead of $\tau_k$ and $\tau^{(r)}$ instead of $\tau_{k}^{(r)}$. Then, by the independence between brackets, the fact that $\cup_{r=r_0}^{\infty} E_r \cap (\cap_{r_0\leq l< r} E_l^c)$ occurs with probability $1$, and line \ref{bracket_lemma_ind} of Lemma \ref{bracket_lemma},
\begin{align*}
\E[\tau] & =    \E[\tau \1\{\cup_{r=r_0}^{\infty} E_r \cap (\cap_{r_0\leq l< r} E_l^c)\}] \\
& \leq \sum_{r=r_0}^{\infty} \E[\tau \1\{ E_r \cap (\cap_{r_0\leq l<r} E_l^c) \}]  \\
& \leq \sum_{r=r_0}^{\infty} \E[\tau^{(r)} \1\{ E_r \cap (\cap_{r_0\leq l<r} E_l^c) \}]  \\
& =  \sum_{r=r_0}^{\infty}\E[\tau^{(r)} \1\{ E_r \} ] \P( \cap_{r_0\leq l<r} E_l^c) \\
& \leq  \sum_{r=r_0}^{\infty} [2^{r-1}(r-1)+  \log(|A_r|  \Delta_{j_0,0}^{-2} \log(  r \frac{\log(\Delta_{j_0,0}^{-2})}{\delta}))  |A_r|  \Delta_{j_0,0}^{-2} \log(  r \frac{\log(\Delta_{j_0,0}^{-2})}{\delta}))]\frac{1}{16^{r-r_0}} \\
& \leq \sum_{s=0}^{\infty} [2^{r_0-1}\cdot 2^s(r_0+s-1)\\
& \hspace{.2in} +  \log(2^s|A_{r_0}|  \Delta_{j_0,0}^{-2} \log(  (r_0+s) \frac{\log(\Delta_{j_0,0}^{-2})}{\delta}))  2^s|A_{r_0}|  \Delta_{j_0,0}^{-2} \log(  (r_0+s) \frac{\log(\Delta_{j_0,0}^{-2})}{\delta}))]\frac{1}{16^{s}} .
\end{align*}
We bound the first term as follows:
\begin{align}
\sum_{s=0}^{\infty} \frac{2^{r_0-1} \cdot 2^s(r_0+s-1)}{16^s} & = 2^{r_0-1}\sum_{s=0}^{\infty} \frac{ (r_0+s-1)}{8^s}  \\
& \leq c 2^{r_0} r_0 \label{eq:first_term_main_pf_bd} \\
& \leq c^\prime \frac{n}{j_0} k \log(\frac{n}{j_0} k) \nonumber \\
& \leq c^{\prime \prime} \fdrHt(\mu_0; j_0)\log(\fdrHt(\mu_0; j_0)). \nonumber
\end{align}
where the last inequality follows since $\Delta_{i,j}^{-2} \geq 1$ for all $i<j \in [n] \cup \{0\}$ since $\mu_0, \mu_1,\ldots, \mu_n \in [0,1]$.

We note that 
\begin{align*}
\log(  (r_0+s) \frac{\log(\Delta_{j_0,0}^{-2})}{\delta})) & \leq c[\log(  r_0 \frac{\log(\Delta_{j_0,0}^{-2})}{\delta}))+ \log(  s \frac{\log(\Delta_{j_0,0}^{-2})}{\delta}))] \\
& \leq c^\prime \log(  r_0 \frac{\log(\Delta_{j_0,0}^{-2})}{\delta}))+ c \log(  s )
\end{align*}
and
\begin{align*}
\log(2^s|A_{r_0}|  \Delta_{j_0,0}^{-2} \log(  (r_0+s) \frac{\log(\Delta_{j_0,0}^{-2})}{\delta})) & = \log(|A_{r_0}|  \Delta_{j_0,0}^{-2} \log(  (r_0+s) \frac{\log(\Delta_{j_0,0}^{-2})}{\delta})) + s \\
& \leq \log(|A_{r_0}|  \Delta_{j_0,0}^{-2} c^\prime \log(  r_0 \frac{\log(\Delta_{j_0,0}^{-2})}{\delta})+ c \log(  s )) + s \\
& \leq c^{\prime \prime} \log(|A_{r_0}|  \Delta_{j_0,0}^{-2}  \log(  r_0 \frac{\log(\Delta_{j_0,0}^{-2})}{\delta})))+ c^{\prime \prime \prime} \log(\log(  s )) + s \\
& \leq c^{\prime \prime} \log(|A_{r_0}|  \Delta_{j_0,0}^{-2}  \log(  r_0 \frac{\log(\Delta_{j_0,0}^{-2})}{\delta})))+ c^{\prime \prime \prime \prime}  s \\
\end{align*}
Then,
\begin{align*}
&  \sum_{s=0}^{\infty}   \log(2^s|A_{r_0}|  \Delta_{j_0,0}^{-2} \log(  (r_0+s) \frac{\log(\Delta_{j_0,0}^{-2})}{\delta}))  2^s|A_{r_0}|  \Delta_{j_0,0}^{-2} \log(  (r_0+s) \frac{\log(\Delta_{j_0,0}^{-2})}{\delta}))\frac{1}{16^{s}} \\
 & \leq   \sum_{s=0}^{\infty}  [c^{\prime \prime} \log(|A_{r_0}|  \Delta_{j_0,0}^{-2}  \log(  r_0 \frac{\log(\Delta_{j_0,0}^{-2})}{\delta})))+ c^{\prime \prime \prime \prime}  s]  |A_{r_0}|  \Delta_{j_0,0}^{-2} [c^\prime \log(  r_0 \frac{\log(\Delta_{j_0,0}^{-2})}{\delta}))+ c \log(  s )]\frac{1}{8^{s}} \\
 & \leq c^{\prime \prime} c^\prime  \log(|A_{r_0}|  \Delta_{j_0,0}^{-2}  \log(  r_0 \frac{\log(\Delta_{j_0,0}^{-2})}{\delta}))) |A_{r_0}|  \Delta_{j_0,0}^{-2} \log(  r_0 \frac{\log(\Delta_{j_0,0}^{-2})}{\delta})) \\
 & + c^{\prime \prime \prime \prime \prime }  \log(|A_{r_0}|  \Delta_{j_0,0}^{-2}  \log(  r_0 \frac{\log(\Delta_{j_0,0}^{-2})}{\delta}))) |A_{r_0}|  \Delta_{j_0,0}^{-2} \\
 & + c^{\prime \prime \prime \prime \prime \prime}  |A_{r_0}|  \Delta_{j_0,0}^{-2} \log(  r_0 \frac{\log(\Delta_{j_0,0}^{-2})}{\delta})) + + c^{\prime \prime \prime \prime \prime \prime \prime}  |A_{r_0}|  \Delta_{j_0,0}^{-2} \\
 & \leq  c^{\prime \prime \prime \prime \prime \prime \prime \prime}  \log(|A_{r_0}|  \Delta_{j_0,0}^{-2}  \log(  r_0 \frac{\log(\Delta_{j_0,0}^{-2})}{\delta}))) |A_{r_0}|  \Delta_{j_0,0}^{-2} \log(  r_0 \frac{\log(\Delta_{j_0,0}^{-2})}{\delta})) 
\end{align*}
Plugging in $|A_{r_0}|$ and $r_0$ yields the gap independent bound. 

\end{proof}

\subsection{Proof of FWER-TPR}

In this section, we prove an upper bound for the FWER-TPR version of our Algorithm (see Algorithm \ref{fwer_alg}). We note that the gap-dependent upper bound in Theorem \ref{thm_fdr_paper} follows as a corollary since whenever the FWER-TPR version of our Algorithm \ref{fwer_alg} accepts an arm, the FDR-TPR version of our Algorithm \ref{fdr_alg} accepts the same arm.

\begin{theorem}
\label{main_thm_fwer}
Let $\delta \in (0,1/40)5$. Let $k \in [|\H_1|]$. For all $j \in \{k, \ldots, |\H_1|\}$ define 
\begin{align*}
\fwerH(\mu_0;j)  & \coloneqq \frac{k}{j}\left( \underbrace{\left\{\sum_{i=1}^m \Delta_{i \vee j, 0}^{-2}\right\}}_{\text{top arms}}\log(\frac{nk}{j\delta}\log(\Delta_{i \vee j, 0}^{-2})) + \underbrace{\sum_{i=m+1}^{n}\Delta_{j,i}^{-2}}_{\text{bottom arms}}\log(\frac{\log(\Delta_{j,i}^{-2})}{\delta})  \right).
\end{align*}
Let $(\mc{F}_t)_{t \in \N}$ be the filtration generated by playing Algorithm \ref{fwer_alg} on problem $\rho$. Then, Algorithm \ref{fwer_alg} has the property that $\P(\exists t : \Q_t \cap \H_0 \neq \emptyset) \leq 2 \delta$ and there exists a stopping time $\tau_k$ wrt $(\mc{F})_{t \in \N}$ such that
\begin{align}
\E[\tau_k] & \leq c \min_{k \leq j \leq m } \fwerH(\mu_0;j)  \log(\fwerH(\mu_0;j) +  \Delta_{j,0}^{-2} \log(\tfrac{n}{j}k \log( \Delta_{j,0}^{-2})/\delta) )  \label{main_thm_FWER_res}
\end{align}
% \begin{align}
% \E[\tau_k] & \lesssim  \min_{\epsilon >0 : |\H_{1,\epsilon}| \geq k } {S}_k(\epsilon) \log[{S}_k(\epsilon) +  \epsilon^{-2} ] \label{main_thm_FWER_res} \\
% & \leq  \widetilde{S}_k \log[\widetilde{S}_k +  \Delta^{-2}  ]\label{main_thm_FWER_res_simple} 
% \end{align}
and for all $t \geq \tau_k$, $\E[|\Q_t \cap \H_1|] \geq (1-\delta)k$. 
\end{theorem}

$E_r$, $E_{0,r}$, and $E_{1,r}$ are defined as in Section \ref{fdr_tpr_proof}.

\begin{lemma} 
\label{bracket_lemma_fwer}
Fix $\delta \in (0,1/40)$, $k \in [|\H_1|]$, $j_0 \in \{k,\ldots,m\}$, and $r \in \N$ such that $2^r \geq k$. Define
\begin{align*}
U_r \coloneqq	\frac{ \min(2^r,n)}{n}[\sum_{i \in \H_1}  \Delta_{i \vee j_0,0}^{-2}  \log(\min(2^r,n) r  \frac{\log(\Delta_{i \vee j_0,0}^{-2})}{\delta})) + \sum_{i \in \H_0}  \Delta_{j_0, i}^{-2} \log(\frac{\log( \Delta_{j_0, i}^{-2})}{\delta }) ].
\end{align*}
Let $(\mc{F}_t)_{t \in \N}$ be the filtration generated by playing Algorithm \ref{fwer_alg} on problem $\rho$. Then, there exists a stopping time $\tau_k$ wrt $(\mc{F}_t)_{t \in \N}$ such that for all $t \geq \tau_k$, $\E[|\mathcal{Q}_t \cap \H_1|] \geq (1-\delta)k$, and
\begin{align}
\E[ \1\{ E_r \cap E_{0,r} \cap E_{1,r} \}  \tau_{k}]& \leq  c [2^{r-1}(r-1)+ U_r \log(U_r + \Delta_{j_0,0}^{-2}  \log(\min(2^r,n) r  \frac{\log(\Delta_{j_0,0}^{-2})}{\delta})))] \label{bracket_lemma_fwer_dep}  
\end{align}
where $c$ is a universal constant.
\end{lemma}

\begin{remark}
Note that $\fwerH(\mu_0;j) \approx U_{\lceil \log_2(\frac{nk}{j}) \rceil}$.
\end{remark}

\begin{proof}
\textbf{Step 1: Define stopping time.} Define
\begin{align*}
\tau_k & = \min(t \in \N \cup \{\infty\} : |A_r \cap [j_0]| \geq k \text{ and } \I_r \cap A_r \cap \H_1 \subset \Q_t). 
\end{align*}

Observe that for all $t \geq \tau_k$, $\E[|\Q_t \cap \H_1|] \geq (1-\delta)k$ since for $t \geq \tau_k$ 
\begin{align*}
\E[|\Q_t \cap \H_1|] \geq \E[|\I_s \cap A_s \cap \H_1|] \geq (1-\delta) |A_s \cap \H_1| \geq (1-\delta)k.
\end{align*}

Let $r \in \N$. Define
\begin{align*}
T & = |\{ t \in \N : \I \cap A_r \cap [j_0] \not \subset \Q_t \text{ and } R_t = r\}|,
\end{align*}
By the same argument used in Lemma \ref{bracket_lemma} to obtain line \eqref{no_expect_bound},
\begin{align}
\1\{ E \} \tau_{k} & \leq    [2^{r-1}(r-1)+  \log(\1\{ E \} T)  \1\{ E \} T]. \label{no_expect_bound_fwer}
\end{align}

We can use the same argument that was used to obtain line \eqref{lilucb_arg} and line \eqref{gap_dep_bound} in Lemma \ref{bracket_lemma} and the lower bounds $1 \leq \sigma(i) $ and $p \geq 1$ to obtain
\begin{align}
\1\{ E \}T & \leq  c \bigg(\sum_{i \in \H_0  \cap A_r}  \Delta_{j_0,i}^{-2} \log(\frac{\log( \Delta_{j_0,i}^{-2})}{\delta }) \\
&\hspace{.5in} + |\H_{1} \cap [j_0]^c \cap A_r| \Delta_{j_0,0}^{-2} \log(\frac{\log(\Delta_{j_0,0}^{-2})}{\delta}) + \sum_{i \in \H_{1}  \cap A_r}  \Delta_{i \vee j_0,0}^{-2} \log(|A_r| r^2 \frac{\log(\Delta_{i \vee j_0,0}^{-2})}{\delta})  \bigg) \nonumber\\
& \leq c^\prime [\sum_{i \in \H_0  \cap A_r}  \Delta_{j_0,i}^{-2} \log(\frac{\log( \Delta_{j_0,i}^{-2})}{\delta })  +  \sum_{i \in \H_{1}  \cap A_r}  \Delta_{i \vee j_0,0}^{-2} \log(|A_r| r \frac{\log(\Delta_{i \vee j_0,0}^{-2})}{\delta}) ] \label{gap_dep_bound_main} \\
& \coloneqq c^\prime S_r
\end{align}
where the second inequality follows from the fact that $\Delta_{i \vee j_0, 0} \geq \Delta_{j_0,0}$ so the third term absorbs the second.

Using lines \eqref{gap_dep_bound_main} and \eqref{no_expect_bound_fwer},
\begin{align*}
\1\{ E \} \tau_{k} & \leq c [2^{r-1}(r-1)+  \log(S_r)  S_r] 
\end{align*} 
but note that now the bound depends on the particular random elements of $A_r \cap \H_0$ and $A_r \cap \H_1$.

\textbf{Step 2: Bounding $\E[ \log(S_r)  S_r]$.}  Next, taking the expectation of both sides and focusing on the expectation of the second term,
\begin{align*}
\E[ \log(S_r)  S_r] & = \sum_{i \in \H_0  }  \Delta_{j_0,i}^{-2} \log(\frac{\log(  \Delta_{j_0,i}^{-2})}{\delta }) \E[ \1 \{i \in A_r \} \log(S_r  )]  \\
&\hspace{1in} +  \sum_{i \in \H_{1}  }  \Delta_{i \vee j_0, 0}^{-2} \log(|A_r| r^2 \frac{\log(\Delta_{i \vee j_0, 0}^{-2})}{\delta}) \E[ \1 \{i \in A_r \} \log(S_r )] . 
\end{align*}
It suffices to bound the first sum since the argument for the second is the same. 
\begin{align}
\E[ \1 \{j \in A_r \} \log(S_r)]  & = \E[ \log(S_r) | j \in A_r] \frac{\min(2^r,n)}{n} \label{expect_lemma_tot_exp} \\
& \leq \log(\E[ S_r | j \in A_r]) \frac{\min(2^r,n)}{n} \label{expect_lemma_jensen}\\
& = \log( \frac{ \min(2^r-1,n-1)}{n-1}[\sum_{i \in \H_1} \Delta_{i \vee j_0, 0}^{-2}  \log(\min(2^r,n) r  \frac{\log(\Delta_{i \vee j_0, 0}^{-2})}{\delta})) \nonumber\\
&\hspace{.5in} + \sum_{i \in \H_0 \setminus j}   \Delta_{j_0,i}^{-2} \log(\frac{\log(  \Delta_{j_0,i}^{-2})}{\delta }) ] +  \Delta_{j_0,j}^{-2} \log(\frac{\log(  \Delta_{j_0,j}^{-2})}{\delta }) )  \frac{\min(2^r,n)}{n} \\
& \leq \log(S_r +  \Delta_{j_0,j}^{-2} \log(\frac{\log(  \Delta_{j_0,i}^{-2})}{\delta })) \frac{\min(2^r,n)}{n}, \label{expect_lemma_fraction}
\end{align}
where line \eqref{expect_lemma_tot_exp} follows by the law of total expectation, line \eqref{expect_lemma_jensen} follows by Jensen's inequality, and line \eqref{expect_lemma_fraction} follows since $\frac{a}{b} \leq \frac{a+1}{b+1}$ if $a \leq b$. Thus, collecting terms, 
\begin{align*}
\E[ \1\{ E_r \cap E_{0,r} \cap E_{1,r} \} \tau_{r,k} ]  \leq U_r \log(U_r + \Delta_{j_0,0}^{-2}  \log(\min(2^r,n) r  \frac{\log(\Delta_{j_0,0}^{-2})}{\delta}))
\end{align*}
yielding line \eqref{bracket_lemma_fwer_dep}.
\end{proof}

\begin{proof}[Proof of Theorem \ref{main_thm_fwer}]
\textbf{Step 1: Showing $\P(\exists t : \Q_t \cap \H_0 \neq \emptyset) \leq 2 \delta$.} First, we show that $\P(\exists t : \Q_t \cap \H_0 \neq \emptyset) \leq 2 \delta$. 
\begin{align*}
\P(\exists t : \Q_t \cap \H_0 \neq \emptyset) & \leq \sum_{r=1}^\infty \P( \exists t : \Q_t \cap A_r \cap \H_0 \neq \emptyset) \\
& \leq \sum_{r=1}^\infty \P( \exists t \in \N \text{ and } i \in \H_0 \cap A_r : \widehat{\mu}_{i,r,T_{i,r}(t)} - U(T_{i,r}(t), \frac{\delta}{|A_{r}| r^2} ) \geq \mu_0 ) \\
& \leq \sum_{r=1}^\infty \P( \exists t \in \N \text{ and } i \in \H_0 \cap A_r : \widehat{\mu}_{i,r,T_{i,r}(t)} - U(T_{i,r}(t), \frac{\delta}{|A_{r}| r^2} ) \geq \mu_i ) \\
& \leq \sum_{r=1}^\infty |A_r \cap \H_0| \frac{\delta}{ |A_r| r^2}  \\
& \leq \sum_{r=1}^\infty  \frac{\delta}{ r^2}  \\
& \leq \delta \frac{\pi^2}{6}
\end{align*}

\textbf{Step 2: Defining the stopping time.}  As in the proof of Lemma \ref{bracket_lemma_fwer}, define
\begin{align*}
\tau_k & = \min(t \in \N \cup \{\infty\} : \exists s \text{ such that } |A_s \cap [j_0]| \geq k \text{ and } \I_s \cap A_s \cap \H_1 \subset \Q_t), \\
\tau_{k}^{(r)} & = \min(t \in \N \cup \{\infty\} :  |A_r \cap [j_0]| \geq k \text{ and } \I_r \cap A_r \cap \H_1 \subset \S_t). 
\end{align*}
As was argued in Step 1 of the proof of Lemma \ref{bracket_lemma_fwer},  for all $t \geq \tau_k$, $\E[|\Q_t \cap \H_1|] \geq (1-\delta)k$ since for $t \geq \tau_k$.

\textbf{Step 3: A lower bound on the probability of a good event.} Let $j_0 \in \{k,\ldots,m\}$ minimize \eqref{main_thm_FWER_res}. Define $E_r = E_r \cap E_{0,r} \cap E_{1,r}$. We note that since $\{\rho_{i,r}\}_{i \in [n], r \in \N}$ and the brackets $\{A_r\}_{r \in \N}$ are independent, $\{E_r\}_{r \in \N}$ are independent events. Let $r_0$ be the smallest integer such that
\begin{align*}
\min(40 \frac{n}{j_0}k,n) \leq 2^{r_0} \leq 80 \frac{n}{j_0}k,
\end{align*}
Note that if $2^{r_0} \geq n$, then the bracket $r_0$ has $n$ arms.

As was argued in the proof of Theorem \ref{main_thm_fdr_log} we have that
\begin{align*}
\P(E_r^c) \leq  \exp(-5) + 2\delta \leq \frac{1}{16}.
\end{align*}
and that $\P(\cap_{l=r_0}^{\infty} E_l^c) = 0$. 

\textbf{Step 4: Gap-Dependent bound on the number of samples.} For the sake of brevity, write $\tau$ instead of $\tau_k$ and $\tau^{(r)}$ instead of $\tau^{(r)}_k$. Then, by the independence between brackets, the fact that $\cup_{r=r_0}^{\infty} E_r \cap (\cap_{r_0\leq l< r} E_l^c)$ occurs with probability $1$, and Lemma \ref{bracket_lemma_fwer},
\begin{align*}
\E[\tau] & = \E[\tau \1\{\cup_{r=r_0}^{\infty} E_r \cap (\cap_{r_0\leq l< r} E_l^c)\}] \\
& \leq \sum_{r=r_0}^{\infty} \E[\tau \1\{ E_r \cap (\cap_{r_0\leq l<r} E_l^c) \}]  \\
& \leq \sum_{r=r_0}^{\infty} \E[\tau^{(r)} \1\{ E_r \cap (\cap_{r_0\leq l<r} E_l^c) \}]  \\
& =  \sum_{r=r_0}^{\infty}\E[\tau^{(r)} \1\{ E_r \} ] \P( \cap_{r_0\leq l<r} E_l^c) \\
& \leq   \sum_{r=r_0}^{\infty}c [2^{r-1}(r-1)+ U_r \log(U_r + \Delta_{j_0,0}^{-2}  \log(\min(2^r,n) r  \frac{\log(\Delta_{j_0,0}^{-2})}{\delta})))]\frac{1}{16^{r-r_0}} \\
& \leq  \sum_{r=r_0}^{\infty}c [2^{r-r_0} \cdot 2^{r_0-1}(r-1)+ 4^{r-r_0} U_{r_0} \log(4^{r-r_0} U_{r_0} + \Delta_{j_0,0}^{-2}  \log(2^{r_0} \cdot 2^{r-r_0} r  \frac{\log(\Delta_{j_0,0}^{-2})}{\delta})))]\frac{1}{16^{r-r_0}} 
\end{align*}
where we used Lemma \ref{bracket_lemma} and the fact that $4^s U_r \geq U_{r+s}$ for any $s \geq 1$, which holds by the following argument
\begin{align*}
4^s U_r & =4^s	\frac{ \min(2^r,n)}{n}[\sum_{i \in \H_1}  \Delta_{i \vee j_0, 0}^{-2}  \log(\min(2^r,n) r  \frac{\log(\Delta_{i \vee j_0, 0}^{-2})}{\delta})) + \sum_{i \in \H_0}  \Delta_{j_0,i}^{-2} \log(\frac{\log( \Delta_{j_0,i}^{-2})}{\delta }) ] \\
& \geq 	\frac{ \min(2^{r+s},n)}{n}[\sum_{i \in \H_1} \Delta_{i \vee j_0, 0}^{-2}  \log(\min(2^{r 2^s},n) r  \frac{\log(\Delta_{i \vee j_0, 0}^{-2})}{\delta})) + \sum_{i \in \H_0}  \Delta_{j_0,i}^{-2} \log(\frac{\log(  \Delta_{j_0,i}^{-2})}{\delta }) ] \\
& \geq 	\frac{ \min(2^{r+s},n)}{n}[\sum_{i \in \H_1}  \Delta_{i \vee j_0, 0}^{-2}  \log(\min(2^{r +s},n) r  \frac{\log(\Delta_{i \vee j_0, 0}^{-2})}{\delta})) + \sum_{i \in \H_0}   \Delta_{j_0,i}^{-2} \log(\frac{\log(  \Delta_{j_0,i}^{-2})}{\delta }) ] \\
& = U_{r+s}.
\end{align*}

Next, we can bound the first term using the same argument in line \eqref{eq:first_term_main_pf_bd}:
\begin{align*}
\sum_{r=r_0}^{\infty}2^{r-r_0} \cdot 2^{r_0-1}(r-1)  \frac{1}{16^{r-r_0}}& \leq   c  2^{r_0}r_0 \\
& \leq c^\prime \frac{n}{j_0} k \log(\frac{n}{j_0} k)  \\
& \leq c^{\prime \prime} \fwerH(\mu_0; j_0)\log(\fwerH(\mu_0; j_0)).  
\end{align*}
where the last inequality follows since $\Delta_{i,j}^{-2} \geq 1$ for all $i<j \in [n] \cup \{0\}$ since $\mu_0, \mu_1,\ldots, \mu_n \in [0,1]$.

Next, we bound the second term.
\begin{align*}
&\sum_{r=r_0}^{\infty} \frac{1}{4^{r-r_0}}  U_{r_0} \log(4^{r-r_0} U_{r_0} + \Delta_{j_0,0}^{-2}  \log(2^{r_0} \cdot 2^{r-r_0} (s+r_0)  \frac{\log(\Delta_{j_0,0}^{-2})}{\delta}))) \\
& = \sum_{s=0}^{\infty} \frac{1}{4^{s}}  U_{r_0} \log(4^{s} U_{r_0} + \Delta_{j_0,0}^{-2}  \log(2^{r_0} \cdot 2^{s} (s+r_0)  \frac{\log(\Delta_{j_0,0}^{-2})}{\delta}))) \\
& \leq  \sum_{s=0}^{\infty} \frac{1}{4^{s}}  U_{r_0} \log(4^{s} U_{r_0} + c^\prime \Delta_{j_0,0}^{-2}  \log(2^{r_0} r_0  \frac{\log(\Delta_{j_0,0}^{-2})}{\delta})) + c^\prime s) \\
& \leq U_{r_0} \sum_{s=0}^{\infty} \frac{1}{4^{s}}[ c^{\prime \prime}  \log(4^{s} U_{r_0} + c^\prime \Delta_{j_0,0}^{-2}  \log(2^{r_0} r_0  \frac{\log(\Delta_{j_0,0}^{-2})}{\delta}))) + c^{\prime \prime \prime} \log( s)] \\
& \leq U_{r_0} \sum_{s=0}^{\infty} \frac{1}{4^{s}}[ c^{\prime \prime}  \log( U_{r_0} + c^\prime \Delta_{j_0,0}^{-2}  \log(2^{r_0} r_0  \frac{\log(\Delta_{j_0,0}^{-2})}{\delta}))) + c^{\prime \prime}  \log(4^{s}) + c^{\prime \prime \prime} \log( s)] \\
& \leq c^{\prime \prime \prime \prime} U_{r_0} \log( U_{r_0} + \Delta_{j_0,0}^{-2}  \log(2^{r_0} r_0  \frac{\log(\Delta_{j_0,0}^{-2})}{\delta}))) \\
& \leq  c^{\prime \prime \prime \prime \prime} U_{r_0} \log( U_{r_0} + \Delta_{j_0,0}^{-2}  \log(2^{r_0}  \frac{\log(\Delta_{j_0,0}^{-2})}{\delta}))) \\
& \leq  c^{\prime \prime \prime \prime \prime \prime} \fwerH(\mu_0; j_0) \log( \fwerH(\mu_0; j_0) + \Delta_{j_0,0}^{-2}  \log(2^{r_0}  \frac{\log(\Delta_{j_0,0}^{-2})}{\delta}))) 
\end{align*}
where we used $U_{r_0} \leq c \fwerH(\mu_0; j_0)$ and
\begin{align*}
\log(2^{r_0} \cdot 2^{s} (s+r_0)  \frac{\log(\Delta_{j_0,0}^{-2})}{\delta})) & = \log(2^{r_0}  (s+r_0)  \frac{\log(\Delta_{j_0,0}^{-2})}{\delta})) +c s \\
& \leq c^\prime \log(2^{r_0} r_0  \frac{\log(\Delta_{j_0,0}^{-2})}{\delta})) + c^\prime \log(2^{r_0} s  \frac{\log(\Delta_{j_0,0}^{-2})}{\delta})) +c s \\
& \leq  c^{\prime \prime} \log(2^{r_0} r_0  \frac{\log(\Delta_{j_0,0}^{-2})}{\delta})) +c^{\prime \prime \prime \prime} s 
\end{align*}

\end{proof}

\subsection{Proof of $\epsilon$-Good Arm Identification}

We restate Theorem \ref{eps_good_theorem_paper} with the doubly logarithmic terms.
\begin{theorem}
\label{eps_good_theorem}
Let $\epsilon >0$ and $\delta \in (0,1)$. Define $m = |\{i : \mu_i > \mu_1 - \epsilon \}|$.	 For all $j \in [m]$ define
\begin{align*}
\bestH(\epsilon;j)  & \coloneqq \frac{1}{j}\left( \underbrace{\sum_{i=1}^m \Delta_{i \vee j,m+1}^{-2} }_{\text{top arms}}\log(\frac{n}{j\delta} \log(\Delta_{i \vee j,m+1}^{-2})) + \underbrace{\sum_{i=m+1}^{n}\Delta_{j,i}^{-2}}_{\text{bottom arms}}\log (\frac{\Delta_{j,i}^{-2}}{\delta})  \right).
\end{align*}
Let $(\mc{F}_t)_{t \in \N}$ be the filtration generated by playing Algorithm \ref{fdr_alg} on problem $\rho$. Then, there exists a stopping time $\tau$ wrt $(\mc{F})_{t \in \N}$ such that
\begin{align}
\E[\tau] & \leq c \min_{j \in [m]} \bestH(\epsilon;j) \log(\bestH(\epsilon;j) + \Delta_{j,m+1}^{-2} \log(\frac{n}{j\delta} \log(\Delta_{ j,m+1}^{-2})))  \label{eps_good_theorem_res} 
\end{align}
and $\P( \exists s \geq \tau: \mu_{O_s} \leq \mu_1 - \epsilon) \leq  2\delta$. 
\end{theorem}

Lemma \ref{eps_good_lemma} is the key intermediate result in the proof of Theorem \ref{eps_good_theorem}; its role is similar to that of Lemma \ref{bracket_lemma} in the proof of Theorem \ref{main_thm_fdr_log} and the proof is technically similar to the proof of Lemma \ref{bracket_lemma}. For any $r \in \mathbb{N}$ and $j \in [m]$ define the events
\begin{align*}
F_{r,1}^{(j)} & = \{A_r \cap [j] \neq \emptyset \} \\
F_{r,2}^{(j)} & = \{ \sum_{i \in  A_r :  \mu_{i} <  \frac{\mu_{j} +\mu_{m+1}}{2}} \Delta_{j,i}^{-2} \log(\frac{1}{\rho_{i,r}  }) \leq 5 \sum_{i \in  A_r : \mu_{i} <  \frac{\mu_{j} +\mu_{m+1}}{2}} \Delta_{j,i}^{-2} \log(\frac{1}{\delta }) \}, \\
F_{r,3}^{(j)} & = \{\exists i_0 \in A_r \cap [j] \text{ s.t.} \forall t \in \N : |\widehat{\mu}_{{i_0},r,t} - \mu_{i_0} | \leq U(t, \delta)  \}.
\end{align*}
$F_{r,1}^{(j)}$ says that there is at least one arm in bracket $r$ with mean at least $\mu_j \geq \mu_m$. $F_{r,2}^{(j)}$ allows us to avoid a union bound and says that most of the arms in bracket $r$ with mean at most $\frac{\mu_{j} +\mu_{m+1}}{2}$ have large $\rho_{i,r}$. Finally, $F_{r,3}^{(j)}$ says that at least one of the arms in the $r$th bracket with mean at least $\mu_j \geq \mu_m$ that concentrates well in the sense that $\rho_{i,r} \geq \delta$. 
\begin{lemma}
\label{eps_good_lemma}
Let $\epsilon > 0$, $m = |\{i : \mu_i > \mu_1 - \epsilon \}|$, $j_0 \in [m]$, and $r \in \N$. Define
\begin{align*}
Y_r = \frac{\min(2^r,n)}{n} [\sum_{i=1}^m  \Delta_{i \vee j_0,m+1}^{-2} \log(|A_r| r \frac{\log(\Delta_{i \vee j_0,m+1}^{-2})}{\delta})+\sum_{i=m+1}^n  \Delta_{j_0, i}^{-2} \log(\frac{\log( \Delta_{j_0, i}^{-2})}{\delta })]
\end{align*}
Let $(\mc{F}_t)_{t \in \N}$ be the filtration generated by playing Algorithm \ref{fdr_alg} on problem $\rho$. Then, there exists a stopping time $\tau$ wrt $(\mc{F}_t)_{t \in \N}$ such that $\P( \exists s \geq \tau: \mu_{O_s} \leq \mu_1 - \epsilon) \leq  2\delta$, and
\begin{align}
\E[\1\{F_{r,1}^{(j_0)}  \cap F_{r,2}^{(j_0)} \cap F_{r,3}^{(j_0)}   \} \tau ] \leq c[2^{r-1} (r-1) + Y_r \log(Y_r + \Delta_{j_0,m+1}^{-2}\log(|A_r| r \frac{\log(\Delta_{j_0,m+1}^{-2})}{\delta})]. \label{eps_good_lemma_bound}
\end{align}
\end{lemma}

\begin{remark}
Note that $\bestH(\epsilon;j) \approx Y_{\lceil \log_2(\frac{n}{j}) \rceil}$.
\end{remark}

\begin{proof}
\textbf{Step 1: Define stopping time.} Our strategy is to define a stopping time $\tau$ that says that some arm $i$ that is $\epsilon$-good has been sampled enough times so that its confidence bound is sufficiently small and then to show that with high probability for all $t\geq\tau$, \emph{(i)} the lower confidence bound of arm $i$ is above $\mu_{m+1}$ and \emph{(ii)} the algorithm always outputs an $\epsilon$-good arm. To this end, define 
\begin{align*}
\tau =  \min\{  t \in \N \cup \{\infty\} :\exists s \in \N \text{ and } \exists i  \in A_s \text{ s.t. } \mu_i \geq \frac{\mu_{j_0} +\mu_{m+1}}{2} \text{ and } T_{i,s}(t) \geq U^{-1}(\frac{\Delta_{i \vee j_0, m+1}}{4}, \frac{\delta}{|A_{s}| s^2}) \}.
\end{align*}
We claim that $\P( \exists t \geq \tau: \mu_{O_t} < \mu_1 - \epsilon) \leq  2\delta$. Define the event
\begin{align*}
F & = \{\forall t \in \N, s \in \N, \text{ and } i \in A_s : |\widehat{\mu}_{i,s,t} - \mu_i| \leq  U(t, \frac{\delta}{|A_{s}| s^2} ) \}.
\end{align*}
By a union bound, $F$ occurs with probability at least $1-2\delta$. Suppose $F$ occurs and let $t \geq \tau$. Then, since $t \geq \tau$, there exists a bracket $s$ and an arm $i \in A_s$ such that $\mu_i \geq \frac{\mu_{j_0}+\mu_{m+1}}{2}$ and $T_{i,s}(t) \geq U^{-1}(\frac{\Delta_{i \vee j_0, m+1}}{4}, \frac{\delta}{|A_{s}| s^2})$. Then by event $F$,
\begin{align*}
\widehat{\mu}_{i, s,T_{i,s}(t)} - U(T_{i,s}(t), \frac{\delta}{|A_{s}| s^2} ) & \geq \mu_i - 2U(T_{i,s}(t), \frac{\delta}{|A_{s}| s^2} ) \\
& > \mu_i - \frac{ \Delta_{i \vee j_0,m+1}}{2} \\
& \geq \mu_{m+1}
\end{align*}
where the last inequality follows by considering separately the cases \emph{(i)} $\mu_i \geq \mu_{j_0}$ and \emph{(ii)} $\mu_i < \mu_{j_0}$. Towards a contradiction, suppose that there exists a bracket $s_0 \in \N$ and another arm $j \in A_{s_0}$ ($j \neq i$) such that $\mu_j \leq \mu_1 - \epsilon$ and the algorithm outputs $j$ at time $t$. Then, by event $F$,
\begin{align*}
\mu_j \geq \widehat{\mu}_{j,s_0,T_{j,s_0}(t)} - U(T_{j,s_0}(t), \frac{\delta}{|A_{s_0}| s_0^2} ) \geq \widehat{\mu}_{i,s,T_{i,s}(t)} - U(T_{i,s}(t), \frac{\delta}{|A_{s}| s^2} ) > \mu_{m+1} \geq \mu_j,
\end{align*}
which is a contradiction. Thus, $\P( \exists t \geq \tau: \mu_{O_t} < \mu_1 - \epsilon) \leq 2 \delta$.

\textbf{Step 2: Relating $\tau$ to bracket $r$.}  Next, we bound $\E[\1\{F_{r,1}^{(j_0)}  \cap F_{r,2}^{(j_0)}  \cap F_{r,3}^{(j_0)}    \} \tau ] $. For the sake of brevity, we write $F_{r,i} $ instead of $F_{r,i}^{(j_0)} $ and define $F_r \coloneqq F_{r,1} \cap F_{r,2} \cap F_{r,3} $ and since we will only focus on bracket $r$, write $\widehat{\mu}_{i,t}$, $T_i(t)$, and $\rho_i$  instead of $\widehat{\mu}_{i,r,t}$, $T_{i,r}(t)$, and $\rho_{i,r}$.  Define
\begin{align*}
T & = |\{ t \in \N : R_t = r \text{ and } \nexists i  \in A_r \text{ s.t. } \mu_i \geq \frac{\mu_{j_0} +\mu_{m+1}}{2} \text{ and } T_{i}(t) \geq U^{-1}(\frac{\Delta_{i \vee j_0,m+1}}{4}, \frac{\delta}{|A_{r}| r^2})\}|,
\end{align*}
i.e., the number of rounds that the algorithm works on the $r$th bracket and there does not exist $i  \in A_r \text{ s.t. } \mu_i \geq \frac{\mu_{j_0} +\mu_{m+1}}{2}$ and $T_{i}(t) \geq U^{-1}(\frac{\Delta_{i \vee j_0,m+1}}{2}, \frac{\delta}{|A_{r}| r^2})$. By the same argument given in line \eqref{no_expect_bound} in Lemma \ref{bracket_lemma}, we have that 
\begin{align*}
\1\{F_r\} \tau \leq c[2^{r-1} (r-1) + \log(T \1\{F_r\}) T\1\{F_r\}].
\end{align*}
\textbf{Step 3: Bounding $T\1\{F_r\}$.}  In the interest of brevity, define $F(t) = \{\nexists i  \in A_r \text{ s.t. } \mu_i \geq \frac{\mu_{j_0} +\mu_{m+1}}{2} \text{ and } T_{i}(t) \geq U^{-1}(\frac{\Delta_{i \vee j_0,m+1}}{4}, \frac{\delta}{|A_{r}| r^2}) \}$. Then,
\begin{align*}
\1\{F_r\}T & \leq \1\{F_r\}\sum_{t=1}^\infty \1 \{ R_t = r, F(t) \} \\
 & \leq \1\{ F_r \} \sum_{t: R_t = r}^\infty \1 \{ \mu_{I_t} <  \frac{\mu_{j_0} +\mu_{m+1}}{2} \} + \1\{ \mu_{I_t} \geq \frac{\mu_{j_0} +\mu_{m+1}}{2},  F(t) \} 
\end{align*}
We bound each sum separately. Note that by $F_{r,3}$ there exists an $i_0 \in A_r \cap G_\gamma$ such that
\begin{align}
 \widehat{\mu}_{{i_0},T_{i_0}(t)} + U(T_{i_0}(t), \delta) \geq  \mu_{i_0}  \geq \mu_{j_0}. \label{eps_good_lemma_good_arm}
\end{align}
Let $j$ such that $\mu_{j} <  \frac{\mu_{j_0} +\mu_{m+1}}{2} $. Then, 
\begin{align*}
 \widehat{\mu}_{j,T_j(t)} + U(T_j(t), \delta) \leq \mu_j + U(T_j(t), \rho_j) + U(T_j(t), \delta) \leq \mu_j  + 2  U(T_j(t), \rho_j \delta).
\end{align*}
Thus, line \eqref{eps_good_lemma_good_arm} implies that if $T_j(t) \geq U^{-1}(\frac{\Delta_{j_0,j}}{4},\rho_j \delta)$, arm $j$ is not pulled since in that case
\begin{align*}
 \widehat{\mu}_{j,T_j(t)} + U(T_j(t), \delta) \leq \mu_j + 2U(T_j(t), \rho_j \delta) \leq \mu_j + \frac{\Delta_{j_0,j}}{2} \leq \mu_{j_0}.
\end{align*}

Thus, by arguments made throughout this paper (e.g., line \eqref{lilucb_arg} of the proof of Lemma \ref{bracket_lemma}) and the event $F_{r,2}$,
\begin{align*}
\sum_{t: R_t = r}^\infty \1 \{ \mu_{I_t} \leq \mu_1 - \epsilon \} \leq c\sum_{j \in A_r :\mu_{j} <  \frac{\mu_{j_0} +\mu_{m+1}}{2} }  \Delta_{j_0,j}^{-2} \log(\frac{\log( \Delta_{j_0,j}^{-2})}{\delta })
\end{align*}
Finally, by event $F$ we clearly have
\begin{align*}
\sum_{t:R_t = r}^\infty  \1\{ \mu_{I_t} \geq \frac{\mu_{j_0} +\mu_{m+1}}{2},  F(t) \}& \leq  c \sum_{j  \in A_r: \mu_j \geq  \frac{\mu_{j_0} +\mu_{m+1}}{2}}  \Delta_{j_0 \vee j,m+1}^{-2} \log(|A_r| r \frac{\log(\Delta_{j_0 \vee j,m+1}^{-2})}{\delta})
\end{align*}
Thus, 
\begin{align*}
\1\{F_r\}T & \leq c[\sum_{j \in A_r : \mu_{j} <  \frac{\mu_{j_0} +\mu_{m+1}}{2} }  \Delta_{j_0,j}^{-2} \log(\frac{\log( \Delta_{j_0,j}^{-2})}{\delta }) +\sum_{j  \in A_r: \mu_j \geq  \frac{\mu_{j_0} +\mu_{m+1}}{2}}  \Delta_{j \vee j_0, m+1}^{-2} \log(|A_r| r \frac{\log(\Delta_{j \vee j_0, m+1}^{-2})}{\delta})] \\
& \leq c[\sum_{j \in A_r : \mu_{j} \leq \mu_1 -\epsilon }  \Delta_{j_0,j}^{-2} \log(\frac{\log( \Delta_{j_0,j}^{-2})}{\delta }) +\sum_{j  \in A_r: \mu_j > \mu_1 -\epsilon}  \Delta_{j \vee j_0, m+1}^{-2} \log(|A_r| r \frac{\log(\Delta_{j \vee j_0, m+1}^{-2})}{\delta})] \\
& \coloneqq c X_r
\end{align*}
where we used the fact that for $j$ satisfying $\mu_{j} <  \frac{\mu_{j_0} +\mu_{m+1}}{2}$, it follows that
\begin{align*}
\Delta_{j_0,j} & = \mu_{j_0} - \mu_j \geq \frac{\mu_{j_0} - \mu_{m+1}}{2} = \frac{\Delta_{j_0,m+1}}{2}.
\end{align*}
Then, using the same argument from lines \eqref{expect_lemma_tot_exp}-\eqref{expect_lemma_fraction}, we have that
\begin{align*}
\E X_r \log(X_r) \leq c Y_r \log(Y_r + \Delta_{j_0,m+1}^{-2} \log(|A_r| r \frac{\log(\Delta_{j_1,m+1}^{-2})}{\delta})]
\end{align*}
Thus, putting it together,
\begin{align*}
\E[\1\{ \1\{F_r\} \tau ] \leq c[2^{r-1} (r-1) + Y_r \log(Y_r + \Delta_{j_0,m+1}^{-2} \log(|A_r| r \frac{\log(\Delta_{j_0,m+1}^{-2})}{\delta})]
\end{align*}

\end{proof}

\begin{proof}[Proof of Theorem \ref{eps_good_theorem}]
Let $j_0 \in [m]$ minimize the optimization problem in line \eqref{eps_good_theorem_res}. Let $r_0$ such that be the smallest integer such that
\begin{align*}
\min(40 \frac{n}{j_0},n) \leq 2^{r_0} \leq 80 \frac{n}{j_0}.
\end{align*}
For the sake of brevity, we write $F_{r_0,i} $ instead of $F_{r_0,i}^{(j_0)} $. We bound $\P((F_{r_0,1} \cap F_{r_0,2} \cap F_{r_0,3} )^c)$. By a union bound and the law of total probability,
\begin{align*}
\P((F_{r_0,1} \cap F_{r_0,2} \cap F_{r_0,3} )^c) & \leq \P(F_{r_0,1}^c \cap F_{r_0,3}^c) + \P(F_{r_0,2}^c) \\
& \leq \P(F_{r_0,1}^c) +  \P(F_{r_0,3}^c| F_{r,1}) + \P(F_{r_0,2}^c) \\
& \leq 2\delta + \P(F_{r_0,1}^c) \\
& \leq  2\delta + \exp(-5) \\
& \leq \frac{1}{16}
\end{align*}
The rest of the proof proceeds as the proof of Theorem \ref{main_thm_fdr_log} starting at step 2.
\end{proof}

\subsection{Proof of FWER-FWPD}

Finally, we present a Theorem for the FWER-FWPD version of Algorithm \ref{fwer_alg}. Although it is possible  to use the ideas from the other upper bound proofs to establish a result that depends on the distribution of the arms in $\H_1$, for simplicity our upper bound is in terms of $\Delta = \min_{i \in {\H_1}} \mu_i - \mu_0$ and $m \coloneqq |\{i : \mu_i > \mu_0 \}$. 

\begin{theorem}
\label{fwer_fdwp_thm}
Let $\delta \in (0,\frac{1}{600})$. Let $k \in [|\H_1|]$. Define
\begin{align*}
\widetilde{V}_k & \coloneqq (\frac{n}{m}k - k) \Delta^{-2} \log(\max(k, \log\log( \frac{n}{m}k \frac{1}{\delta}) )\log(\Delta^{-2})\log(\frac{n}{m}k) /\delta) \\
& + k \log(\max(\frac{n}{m}k - (1-2\delta(1+4\delta))k,\log\log(\frac{n}{m}k \frac{1}{\delta})) \log(\Delta^{-2})\log(\frac{n}{m}k)/\delta) ] \\
& \lesssim (\frac{n}{m}k - k) \Delta^{-2} \log(k /\delta)  + k \log(\frac{\frac{n}{m}k - (1-2\delta(1+4\delta))k}{\delta}) 
\end{align*}
Furthermore, define
\begin{align*}
\lambda_k & = \min(t \in \N : |\calR_t \cap \H_1| \geq k). 
\end{align*}
Then, Algorithm \ref{fwer_alg} has the property that $\P(\exists t \in \N: \calR_t \cap \H_0 \neq \emptyset) \leq 10 \delta$ and
\begin{align*}
\E[\lambda_k] \leq c \log(\widetilde{V}_k) \widetilde{V}_k.
\end{align*}

\end{theorem}

\begin{lemma}
Let $\delta \in (0,.01)$. Let $k \in [|\H_1|]$. Let $r \in \N$ such that $2^r \geq k$. Define
\begin{align*}
\lambda_r & = \min(t \in \N : |\calR_t \cap A_r \cap \H_1| \geq k). 
\end{align*}
Define
\begin{align*}
V_r & \coloneqq (2^r - \min(|\H_1|,  \frac{|\H_1|}{n} 2^r)) \Delta^{-2} \log(\max(\min(|\H_1|,  \frac{|\H_1|}{n} 2^r), \log\log(r 2^r/\delta) )\log(\Delta^{-2})r /\delta) \\
& + \min(|\H_1|,  \frac{|\H_1|}{n} 2^r) \log(\max(2^r - (1-2\delta(1+4\delta))\min(|\H_1|,  \frac{|\H_1|}{n} 2^r),\log\log(\frac{r 2^r}{\delta})) \log(\Delta^{-2})r/\delta) ]
\end{align*}
Then with probability at least $1 -6\delta - 2\exp(-2^{r-3}) - \P( |A_r \cap \H_1| <  k)$,
\begin{align*}
\lambda_r \leq c(2^{r-1} (r-1) + \log(V_r) V_r). 
\end{align*}

\end{lemma}

\begin{proof}
\textbf{Step 1: Definitions and events.} Recall $R_t$ is the bracket chosen at time $t$ and define
\begin{align*}
T & = |\{ t \in \N :  A_r \cap \H_1 \not \subset \calR_t \text{ and } R_t = r\}|, 
\end{align*}
i.e., the number of rounds that the algorithm works on the $r$th bracket and $A_r \cap \H_1 \not \subset \calR_t$. 
Define the events 
\begin{align*}
\Sigma_{r,1} & = \{ |A_r \cap \H_1| \geq  k \} \\
\Sigma_{r,2} & = \{ |A_r \cap \H_1| \leq \min(|\H_1|,  \tfrac{|\H_1|}{n} 2^{r+1})) \} \\
\Sigma_{r,3} & = \{ |A_r \cap \H_1| \geq \min(|\H_1|,  \tfrac{|\H_1|}{n} 2^{r-1})) \} 
\end{align*}
If $2^{r+1} \geq n$, then $|A_r \cap \H_1| \leq |\H_1|$ implies $\P(\Sigma_{r,2}^c) = 0$. Therefore, suppose $2^{r+1}< n$. Then, by multiplicative Chernoff for hypergeometric random variables,
\begin{align*}
\P(\Sigma_{r,2}^c) & = \P( |A_r \cap \H_1| >  \frac{|\H_1|}{n} 2^{r+1})  \leq \exp(-\frac{|\H_1|}{n} 2^{r-2}) \leq \exp(-2^{r-2})
\end{align*}
Similarly, if $2^r \geq n$, then $|A_r| = n$ and $\P(\Sigma_{r,2}^c) = 0$. Therefore, suppose $2^r < n$.
\begin{align*}
\P(\Sigma_{r,3}^c) & = \P(|A_r \cap \H_1| < \frac{|\H_1|}{n} 2^{r-1}) ) \leq \exp(-\frac{|\H_1|}{n} 2^{r-3}) \leq \exp(-2^{r-3})
\end{align*}

Since the algorithm essentially runs the FWER-FWDP version of the algorithm from \cite{jamieson2018bandit} on each bracket $r$ with confidence $\delta/r^2$, we can apply Theorem 4 of \cite{jamieson2018bandit} directly to obtain that there exists an event $\Sigma_{r,4}$, which only depends on the samples of the arms in bracket $r$, such that $\P(\Sigma_{r,4}^c) \leq 6 \delta$ and on $\Sigma_{r,4}$
\begin{align*}
T & \leq c[(|A_r| - |A_r \cap \H_1|) \Delta^{-2} \log(\max(|A_r \cap \H_1|, \log\log(|A_r|/\delta_r) )\log(\Delta^{-2})/\delta_r) \\
& + |A_r \cap \H_1| \Delta^{-2} \log(\max(|A_r| - (1-2\delta_r(1+4\delta_r)|A_r \cap \H_1|,\log\log(\frac{|A_r|}{\delta_r})) \log(\Delta^{-2})/\delta_r)].
\end{align*}
This roughly says 
\begin{align*}
T \lesssim (|A_r| - |A_r \cap \H_1|) \Delta^{-2} \log(|A_r \cap \H_1|/\delta) + |A_r \cap \H_1| \Delta^{-2} \log( (|A_r| - |A_r \cap \H_1|)/\delta).
\end{align*}

\textbf{Step 2: Bounding $\lambda_r$.}
In what follows, assume $\Sigma_{r,1} \cap \Sigma_{r,2} \cap \Sigma_{r,3} \cap \Sigma_{r,4}$ occurs, which happens with probability at least
\begin{align*}
1 -6\delta - 2\exp(-2^{r-3}) - \P(\Sigma_{r,1}^c).
\end{align*}

By the same argument given in lines \eqref{fdr_log_bound} and \eqref{no_expect_bound}, event $\Sigma_{r,1} $ implies that
\begin{align*}
\lambda_r \leq c(2^{r-1} (r-1) + \log(T) T). 
\end{align*}
Furthermore, using $\Sigma_{r,2} \cap \Sigma_{r,3} \cap \Sigma_{r,4}$, 
\begin{align*}
T  \leq& c [(|A_r| - |A_r \cap \H_1|) \Delta^{-2} \log(\max(|A_r \cap \H_1|, \log\log(|A_r|/\delta_r) )\log(\Delta^{-2})/\delta_r) \\
& + |A_r \cap \H_1| \Delta^{-2} \log(\max(|A_r| - (1-2\delta_r(1+4\delta_r))|A_r \cap \H_1|,\log\log(\frac{|A_r|}{\delta_r})) \log(\Delta^{-2})/\delta_r) ]\\
 \leq& c^\prime [(|A_r| - |A_r \cap \H_1|) \Delta^{-2} \log(\max(|A_r \cap \H_1|, \log\log(r |A_r|/\delta) )\log(\Delta^{-2})r /\delta) \\
& + |A_r \cap \H_1| \Delta^{-2} \log(\max(|A_r| - (1-2\delta(1+4\delta))|A_r \cap \H_1|,\log\log(\frac{r |A_r|}{\delta})) \log(\Delta^{-2})r/\delta) ]\\
\leq&  c^{\prime \prime} [(2^r - \min(|\H_1|,  \frac{|\H_1|}{n} 2^r)) \Delta^{-2} \log(\max(\min(|\H_1|,  \frac{|\H_1|}{n} 2^r), \log\log(r 2^r/\delta) )\log(\Delta^{-2})r /\delta) \\
& + \min(|\H_1|,  \frac{|\H_1|}{n} 2^r) \Delta^{-2} \log(\max(2^r - (1-2\delta(1+4\delta))\min(|\H_1|,  \frac{|\H_1|}{n} 2^r),\log\log(\frac{r 2^r}{\delta})) \log(\Delta^{-2})r/\delta) ]
\end{align*}

\end{proof}

\begin{proof}[Proof of Theorem \ref{fwer_fdwp_thm}]
We note that the algorithm essentially runs the FWER-FWDP version of the algorithm from \cite{jamieson2018bandit} on each bracket $r$ with confidence $\delta/r^2$. Therefore, by Theorem 4 from  \cite{jamieson2018bandit}, 
\begin{align*}
\P(\exists t \in \N: A_r \cap \calR_t \cap \H_0 \neq \emptyset) \leq 6 \frac{\delta}{r^2}
\end{align*}
Thus,
\begin{align*}
\P(\exists t \in \N: \calR_t \cap \H_0 \neq \emptyset) &\leq \P(\exists t \in \N, r \in \N: A_r \cap \calR_t \cap \H_0 \neq \emptyset) \\
& \leq  \sum_{ r \in \N} \P(\exists t \in \N: A_r \cap \calR_t \cap \H_0 \neq \emptyset) \\
& \leq \sum_{ r \in \N} 6 \frac{\delta}{r^2} \\
& \leq 10 \delta.
\end{align*}

Let $r_0 \in \N$ be the smallest integer such that $r_0 \geq 6$ and 
\begin{align*}
\min(40 \frac{n}{m}k,n) \leq 2^{r_0} \leq 80 \frac{n}{m}k .
\end{align*}
If $2^{r_0} \geq n$, then $ \P( |A_r \cap \H_1| <  k) = 0$. Otherwise,by multiplicative Chernoff for hypergeometric random variables,
\begin{align*}
\P( |A_r \cap \H_1| <  k) \leq \exp(-5).
\end{align*}
In the interest of brevity, define $\Sigma_r = \Sigma_{r,1} \cap \Sigma_{r,2} \cap \Sigma_{r,3} \cap \Sigma_{r,4}$. Observe that $\{\Sigma_r \}_{r \in \N}$ are mutually independent.
Further, using $\delta \in (0,\frac{1}{600})$, for all brackets $r \geq r_0$, the events  occur which happens with probability at least
\begin{align*}
\P(\Sigma_r^c) \leq 6\delta + 2\exp(-2^{r-3}) + \P(\Sigma_{r,1}^c) \leq \frac{1}{16}
\end{align*}
The rest of the proof proceeds as in Step 2 of the proof of Theorem \ref{main_thm_fdr_log}.
\end{proof}

\section{Best of both Worlds Algorithm for $\epsilon$-Good Arm Identification}
\label{sec:best_of_both}

One practical concern about the SimplePAC setting is that it is not clear when to stop the algorithm. To address this concern we propose Algorithm \ref{alg:BOBW}, which combines Algorithm \ref{fdr_alg} and LUCB from \cite{DBLP:conf/icml/KalyanakrishnanTAS12} to achieve the best of both worlds of PAC and SimplePAC. Let LUCB($\epsilon$) denote the LUCB algorithm that terminates once it finds an $\epsilon$-good arm. Let $\beta(t,\delta)$ denote the confidence bound used in \cite{DBLP:conf/icml/KalyanakrishnanTAS12}; although, it is possible to tighten these confidence bounds, for the sake of simplicity and brevity we use theirs so that we can appeal to their sample complexity results. Algorithm \ref{alg:BOBW} takes a desired tolerance $\epsilon > 0$ as input, runs LUCB($\epsilon$) and the $\epsilon$-good arm identification version of Algorithm \ref{fdr_alg} in parallel without sharing samples between the algorithms,\footnote{Samples should be shared in practice.} and outputs an arm $\widehat{i}_t$ at every iteration. This arm $\widehat{i}_t$ is the arm $O_t$ suggested by Algorithm \ref{fdr_alg} for every iteration until the termination condition of LUCB($\epsilon$) obtains at which point algorithm \ref{alg:BOBW} decides whether to output $O_t$ or the arm suggested by LUCB($\epsilon$). Let $\widehat{\mu}_{i,t}$ denote the empirical mean at time $t$ of arm $i$ based on the samples collected by LUCB($\epsilon$) and $T_{i,t}$ denote the number of pulls of arm $i$ at time $t$ by LUCB($\epsilon$).

\begin{algorithm}[t]
\caption{Best of both Worlds Algorithm: $\epsilon$-Good Arm Identification}
\label{alg:BOBW}
\footnotesize
\begin{algorithmic}[1]
\STATE \textbf{Input:} $\epsilon > 0$
\FOR{ $t =1,2, \ldots$}
\STATE Pull arm according to sampling rule given by the $\epsilon$-good arm identification version of Algorithm \ref{fdr_alg}
\STATE Pull arm according to sampling rule given by LUCB($\epsilon$)
\STATE Let $O_t$ be the arm returned by the $\epsilon$-good arm identification version of Algorithm \ref{fdr_alg}
\IF{LUCB($\epsilon$) terminates}
\STATE Let $\widehat{j}$ denote the arm returned by LUCB($\epsilon$)
\STATE $r_0 = \text{argmax}_{r \in \N} \widehat{\mu}_{O_t,r,T_{O_t,r}(t)} - U(T_{O_t,r}(t), \frac{\delta}{|A_{r}| r^2} )$ 
\IF{$\widehat{\mu}_{O_t,r_0,T_{i,r}(t)} - U(T_{O_t,r_0}(t), \frac{\delta}{|A_{r_0}| r^2} ) \geq \widehat{\mu}_{\widehat{j},T_{\widehat{j}}(t)} - \beta(T_{\widehat{j}}(t), \delta)$}
\STATE   Set $\widehat{i}_t = O_t$
\ELSE
\STATE  Set $\widehat{i}_t = \widehat{j}$
\ENDIF
\STATE Output $\widehat{i}_t$ and terminate.
\ELSE
\STATE Set $\widehat{i}_t = O_t$
\STATE Output $\widehat{i}_t$
\ENDIF
\ENDFOR
\end{algorithmic}
\end{algorithm}

\begin{theorem}
\label{thm:bobw}
Let $\rho$ be a problem instance and let $\delta \in (0,1/40)$ and $\epsilon_1, \epsilon_2 > 0$. Let $(\mc{F}_t)_{t \in \N}$ be the filtration generated by running Algorithm \ref{alg:BOBW} with input $\epsilon_1$ on $\rho$. There is a stopping time $\tau_{simple}$ wrt $(\mc{F}_t)_{t \in \N}$ such that 
\begin{align}
\E[\tau_{simple}] & \lesssim \min_{\gamma \in(0,\epsilon_2)}  U_{\epsilon_2}(\gamma)\log(U_{\epsilon_2}(\gamma) + \Delta_{m, \epsilon_2, \gamma}^{-2} ) \label{eq:bobw_simple_bound}
\end{align}
and $\P( \exists s \geq \tau_{simple}: \mu_{\widehat{i}_s} \leq \mu_1 - \epsilon_2) \leq  2\delta$. Furthermore, there exists a stopping time $\tau_{PAC}$ wrt $(\mc{F}_t)_{t \in \N}$ such that
\begin{align}
\E[\tau_{PAC}] \lesssim H^{\epsilon/2} \log(\frac{H^{\epsilon/2}}{\delta}) \label{eq:bobw_pac_bound}
\end{align}
where $H^\gamma = \sum_{i \in [n]} \max(\mu_1 - \mu_i, \gamma)^{-2}$ and at time $\tau_{PAC}$ the Algorithm \ref{alg:BOBW} terminates and returns an arm $\widehat{i}_{\tau_{PAC}}$ such that $\P(\mu_{\widehat{i}_{\tau_{PAC}}} \leq \mu_1 - \min(\epsilon_1, \epsilon_2)) \leq 3 \delta$. 
\end{theorem}

To interpret the Theorem \ref{thm:bobw}, suppose that $\epsilon_1 > \epsilon_2>0$ are such that $\E[\tau_{simple}] \leq \E[\tau_{PAC}]$. Then, Theorem \ref{thm:bobw} says that Algorithm \ref{alg:BOBW} with input $\epsilon_1$ starts outputting an $\epsilon_2$-good arm in nearly optimal time \emph{and} certifies that it is an $\epsilon_1$-good arm in nearly optimal optimal. Thus, Algorithm \ref{alg:BOBW} achieves the best of both worlds.

\begin{proof}[Proof of Theorem \ref{thm:bobw}]
Theorem 6 of \cite{DBLP:conf/icml/KalyanakrishnanTAS12} implies that there exists a stopping time $\tau_{PAC}$ wrt $(\mc{F}_t)_{t \in \N}$ such that at time $\tau_{PAC}$ the Algorithm \ref{alg:BOBW} terminates and \eqref{eq:bobw_pac_bound} holds. Theorem \ref{eps_good_theorem_paper} implies the existence of stopping time $\tau_{simple}$ wrt $(\mc{F}_t)_{t \in \N}$ such that \eqref{eq:bobw_simple_bound} holds and $\P( \exists s \geq \tau_{simple}: \mu_{O_s} \leq \mu_1 - \epsilon_2) \leq  2\delta$. 

It remains to show that when the Algorithm \ref{alg:BOBW} terminates at $t = \tau_{PAC}$,  $\P(\mu_{\widehat{i}_{\tau_{PAC}}} \leq \mu_1 - \min(\epsilon_1,\epsilon_2)) \leq 3 \delta$.  Define the event
\begin{align*}
F & = \{\forall t \in \N, s \in \N, \text{ and } i \in A_s : |\widehat{\mu}_{i,s,t} - \mu_i| \leq  U(t, \frac{\delta}{|A_{s}| s^2} ) \}.
\end{align*}
By a union bound, $F$ occurs with probability at least $1-2\delta$. By the argument in Step 1 of the proof of Lemma \ref{eps_good_lemma}, on $F$, for all $t \geq \tau_{simple}$
\begin{align*}
\text{max}_{r \in \N} \widehat{\mu}_{O_t,r,T_{O_t,r}(t)} - U(T_{O_t,r}(t), \frac{\delta}{|A_{r}| r^2} ) > \max_{ i : \mu_i \leq \mu_1 - \epsilon_2} \mu_i.
\end{align*}
Next, define the event 
\begin{align*}
E = \{\forall t \in \N \text{ and } \forall i \in [n] : |\widehat{\mu}_{i,t} - \mu_i| \leq \beta(t,\delta) \}
\end{align*} 

By Theorem 1 of \cite{DBLP:conf/icml/KalyanakrishnanTAS12}, $\P(E) \geq 1 - \delta$ and on $E$,
\begin{align*}
\widehat{\mu}_{\widehat{j},T_{\widehat{j}}(\tau_{PAC})} - \beta(T_{\widehat{j}}(\tau_{PAC}), \delta) > \mu_1 - \epsilon_1
\end{align*}
Suppose $F$ and $E$ occur, which by a union bound occur with probability at least $1-3\delta$. Either $\widehat{i}_{\tau_{PAC}} = \widehat{j}$ or $\widehat{i}_{\tau_{PAC}} = O_{\tau_{PAC}}$. Suppose $\widehat{i}_{\tau_{PAC}}  = \widehat{j}$. Then,
\begin{align*}
\mu_{\widehat{i}_{\tau_{PAC}}} & = \mu_{\widehat{j}} \\
& \geq \widehat{\mu}_{\widehat{j},T_{\widehat{j}}(\tau_{PAC})} - \beta(T_{\widehat{j}}(\tau_{PAC}), \delta) \\
& > \text{max}_{r \in \N} \widehat{\mu}_{O_t,r,T_{O_t,r}(t)} - U(T_{O_t,r}(t), \frac{\delta}{|A_{r}| r^2} ) \\
& \geq \max_{ i : \mu_i \leq \mu_1 - \epsilon_2} \mu_i,
\end{align*}
which implies that $\mu_{\widehat{i}_{\tau_{PAC}}} \geq \mu_1 - \min(\epsilon_1, \epsilon_2)$. A similar argument proves the case $\widehat{i}_{\tau_{PAC}} = O_{\tau_{PAC}}$.

\end{proof}

\section{Experiment Details} \label{sec:experiment_details}
We used two publicly available datasets to base our simulated experiments on. 

\subsection{$\epsilon$-good arm identification}
For the $\epsilon$-good arm identification experiment, we used the \emph{New Yorker Magazine} Caption Contest data available at \url{https://github.com/nextml/caption-contest-data}.
Specifically, we used contest 641 conducted the first week of December of 2018. 
Briefly, visitors to the site \url{nextml.org/captioncontest} are shown a fixed image and one of $n$ captions that they rate as either \texttt{Unfunny}, \texttt{Somewhat funny}, or \texttt{Funny}.  When they make their selection, the image stays the same but one of $n$ other captions are shown (uniformly at random for this contest). 
Contest 641 has $n=9061$ arms and each one was shown about $155$ times. 
For the $i$th caption we define $\widehat{\mu}_{i,T_i}$ as the proportion of times \texttt{Somewhat funny} of \texttt{Funny} was clicked relative to the total number of times it was rated denoted $T_i$. 
These empirical means $\widehat{\mu}_{i,T_i}$ were treated as ground truth so that in our experiments a pull of the $i$th arm was an iid draw from a Bernoulli distribution with mean $\widehat{\mu}_{i,T_i}$.
Figure~\ref{fig:new_yorker} shows the histogram $\widehat{\mu}_{i,T_i}$ and $T_i$ for all $n=9061$
 arms.

 \begin{figure}
 \includegraphics[width=.5\textwidth]{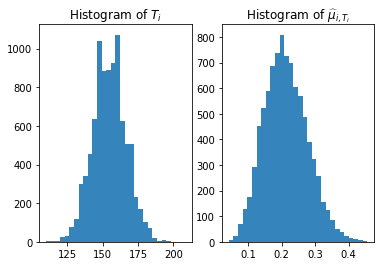}
 \caption{Empirical means and counts from the \emph{New Yorker Magazine} caption contest 641. There were $n=9061$ arms.}
 \label{fig:new_yorker}
 \end{figure}

\subsection{Identifying arms above a threshold}
This dataset is from \cite{hao2008drosophila}. 
The study was interested in identifying genes in Drosophila that inhibit virus replication. 
Essentially, for each individual gene $i \in [n]$ for $n=13071$ the researchers used RNAi to ``knock-out'' the gene from a population of cells, infected the cells with a virus connected to a florescing tag, and then measured the amount of florescence after a period of time.  
The idea is that if a lot of florescence was measured when the $i$th gene was knocked out, that means that gene was very influential for inhibiting virus replication because more virus was present. 
A control or baseline amount of florescence $\mu_0$ (and its variance) was established by infecting cells without any genes knocked out. 
Using these controls, each measurement (pull) from the $i$th gene (arm) is reported as a $Z$-score such that under the null (gene $i$ has no impact on virus replication) an observation is normally distributed with mean $\mu_i=\mu_0$ with variance $1$.
We make the simplifying assumption that if the gene did have non-negligible influence so that $\mu_i > 0$, then the variance was still equal to $1$.

As described in \cite{hao2008drosophila}, the researchers measured each of the $n=13071$ genes twice and eliminated all but the $1000$ most extreme observations, and then measured each of these $1000$ genes $12$ times. Finally, they reported the $100$ genes that were statistically significant of these $1000$ genes measured $12$ times. 
To generate the data for our experiments, we average just the two initial measurements from all $n=13071$ measurements. 
Two averaged $Z$-scores of the $i$th gene, denoted $\widehat{\mu}_i$, have a variance of $1/2$ which more or less buries any signal in noise.  
If we adopt the model $\widehat{\mu}_i \sim \mathcal{N}(\mu_i,1/2)$ then we can perform a a maximum likelihood estimate (MLE) of the original distribution of underlying $\{ \mu_i \}_{i=1}^n$ using a fine grid on $[-4,4]$, the range of the observations.
The normalized histogram of $\{\widehat{\mu}_i\}_i$ as well as the MLE of the $\{\mu_i\}_i$ are shown in the first panel of Figure~\ref{fig:drosophila}.
Reassuringly, there is a spike with mass of about $.97$ at $0$ indicating that the vast majority of genes have no influence on inhibiting virus proliferation. 
The majority of the remaining mass lies in a spike around $1$. 
To encourage the distribution of the means not at $0$ to have a bit more shape, we use a small amount of entropic regularization without increasing negative log likelihood too much. 
For our experiments we used $\lambda=1e^{-4}$.
\begin{figure}
\includegraphics[width=.32\textwidth]{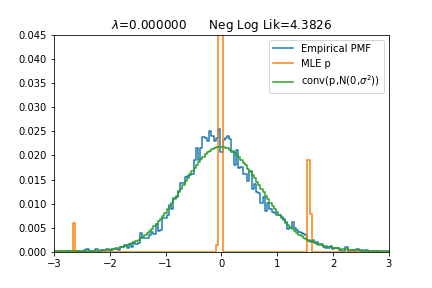}
\includegraphics[width=.32\textwidth]{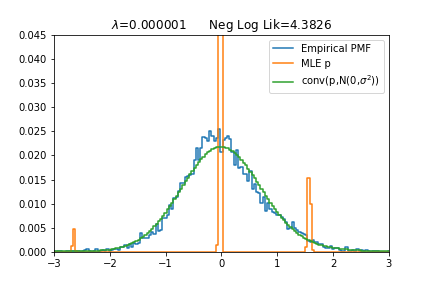}
\includegraphics[width=.32\textwidth]{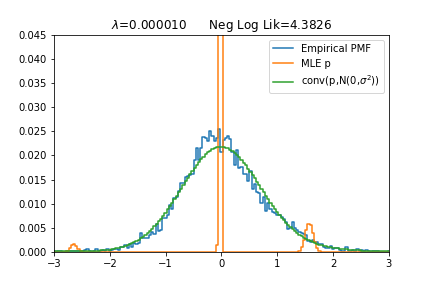}\\
\includegraphics[width=.32\textwidth]{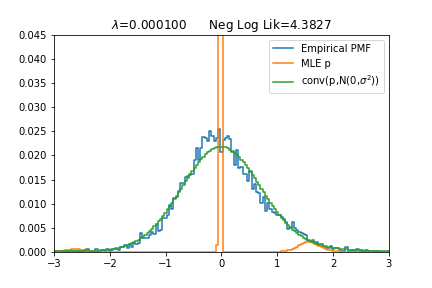}
\includegraphics[width=.32\textwidth]{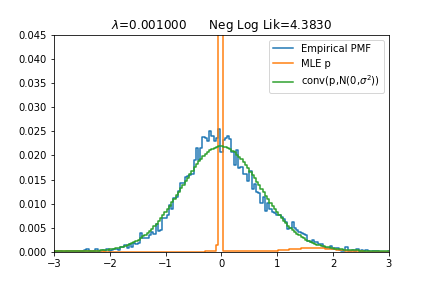}
\includegraphics[width=.32\textwidth]{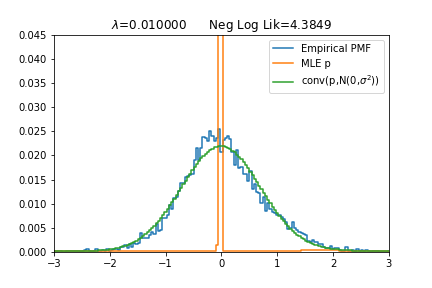}
\caption{Drosophila data.}
\label{fig:drosophila}
\end{figure}

\subsection{Algorithm Details}

We use $\delta = 0.05$ for all of the algorithms. For the implementation of our algorithms,  we chose the starting bracket to have size $2^6$. We share samples between the brackets and stop opening brackets after a bracket of size $n$ is opened. For the $\epsilon$-good arm identification experiment, we use a heuristic where we remove a bracket if its maximum lower confidence bound is less than the maximum lower confidence bound of a larger bracket. 

For the experiment concerning  the dataset of  \cite{hao2008drosophila} we used the FDR-TPR versions of our algorithm and the algorithm of \cite{jamieson2018bandit}. Following the advice of \cite{jamieson2018bandit}, we use the Benjamini-Hochberg procedure developed for multi-armed bandits at level $\delta$ instead of $O(\delta/\log(1/\delta))$. We used the following two heuristics for our algorithm. First, we give each bracket a point if it pulls an accepted arm more than any of the other brackets. Then, we remove a bracket if its score is less than the score of a larger bracket. Second, we estimate the number of pulls required for each bracket to accept $5$ additional arms and choose the bracket with lowest estimate $90\%$ of the time and otherwise cycle through the brackets.\footnote{We note that we only get slightly worse performance if we pick the bracket with lowest estimate $50\%$ of the time. See Figure \ref{fig:alt_heuristic}.} We calculate this estimate as follows. For each bracket, we take the 5 arms with the largest empirical means and estimate the remaining number of times that they need to be pulled by
\begin{align*}
\widehat{\mu}_{i,T_i(t)}^{-2} \log[\text{ size of the bracket } \cdot \text{ number of total brackets to open }/\delta] - T_i(t).
\end{align*}
For the other arms, we estimate the number of times that they need to be pulled before accepting 5 arms with the largest empirical means in the following way. Let $\lambda$ denote the value of the fifth smallest mean multiplied by a factor of $2$, which estimates roughly the value of its upper confidence bound at the point at which it is accepted. Then, the estimate is 
\begin{align*}
(\lambda - \widehat{\mu}_{i,T_i(t)})^{-2}\log[\text{ number of total brackets to open }/\delta] - T_i(t).
\end{align*}

We note that while the above heuristics for removing brackets break the sample complexity guarantees of the algorithms because they may remove a good bracket, the algorithms are still correct in the sense that the confidence bounds hold with high probability. We ran each experiment for 100 trials. We also plot $95\%$ confidence intervals.

\begin{figure*}[t]
    \centering
    \includegraphics[width=\textwidth]{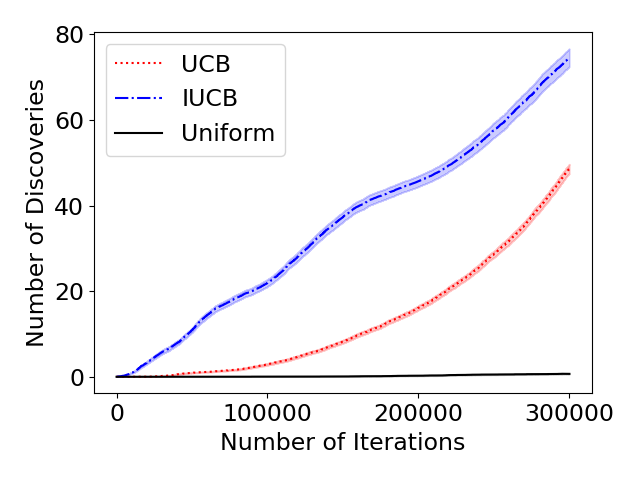}
    \vspace*{-5mm}
    \caption{Identifying means above a threshold: pick estimated best bracket $50\%$ of the time.}
  \label{fig:alt_heuristic}
\end{figure*}

\clearpage
\bibliography{refs}

\begin{thebibliography}{10}

\bibitem{balcan2010true}
Maria-Florina Balcan, Steve Hanneke, and Jennifer~Wortman Vaughan.
\newblock The true sample complexity of active learning.
\newblock {\em Machine learning}, 80(2-3):111--139, 2010.

\bibitem{chaudhuri2019pac}
Arghya~Roy Chaudhuri and Shivaram Kalyanakrishnan.
\newblock Pac identification of many good arms in stochastic multi-armed
  bandits.
\newblock {\em arXiv preprint arXiv:1901.08386}, 2019.

\bibitem{kaufmann2016complexity}
Emilie Kaufmann, Olivier Capp{\'e}, and Aur{\'e}lien Garivier.
\newblock On the complexity of best-arm identification in multi-armed bandit
  models.
\newblock {\em The Journal of Machine Learning Research}, 17(1):1--42, 2016.

\bibitem{Mannor04thesample}
Shie Mannor, John~N. Tsitsiklis, Kristin Bennett, and Nicolò Cesa-bianchi.
\newblock The sample complexity of exploration in the multi-armed bandit
  problem.
\newblock {\em Journal of Machine Learning Research}, 5:2004, 2004.

\bibitem{DBLP:conf/icml/KalyanakrishnanTAS12}
Shivaram Kalyanakrishnan, Ambuj Tewari, Peter Auer, and Peter Stone.
\newblock {PAC} subset selection in stochastic multi-armed bandits.
\newblock In {\em Proceedings of the 29th International Conference on Machine
  Learning, {ICML} 2012, Edinburgh, Scotland, UK, June 26 - July 1, 2012},
  2012.

\bibitem{10.1109/CISS.2014.6814096}
K.~Jamieson and R~Nowak.
\newblock Best-arm identification algorithms for multi-armed bandits in the
  fixed confidence setting.
\newblock {\em Information Sciences and Systems (CISS)}, pages 1--6, 2014.

\bibitem{jamieson2018bandit}
Kevin Jamieson and Lalit Jain.
\newblock A bandit approach to multiple testing with false discovery control.
\newblock In {\em Advances in Neural Information Processing Systems}, 2018.

\bibitem{hao2008drosophila}
Linhui Hao, Akira Sakurai, Tokiko Watanabe, Ericka Sorensen, Chairul~A Nidom,
  Michael~A Newton, Paul Ahlquist, and Yoshihiro Kawaoka.
\newblock Drosophila rnai screen identifies host genes important for influenza
  virus replication.
\newblock {\em Nature}, 454(7206):890, 2008.

\bibitem{hao2013limited}
Linhui Hao, Qiuling He, Zhishi Wang, Mark Craven, Michael~A Newton, and Paul
  Ahlquist.
\newblock Limited agreement of independent rnai screens for virus-required host
  genes owes more to false-negative than false-positive factors.
\newblock {\em PLoS computational biology}, 9(9):e1003235, 2013.

\bibitem{degenne2019pure}
R{\'e}my Degenne and Wouter~M Koolen.
\newblock Pure exploration with multiple correct answers.
\newblock {\em arXiv preprint arXiv:1902.03475}, 2019.

\bibitem{garivier2019non}
Aur{\'e}lien Garivier and Emilie Kaufmann.
\newblock Non-asymptotic sequential tests for overlapping hypotheses and
  application to near optimal arm identification in bandit models.
\newblock {\em arXiv preprint arXiv:1905.03495}, 2019.

\bibitem{Bubeckal11}
S.~Bubeck, R.~Munos, and G.~Stoltz.
\newblock {Pure Exploration in Finitely Armed and Continuous Armed Bandits}.
\newblock {\em Theoretical Computer Science 412, 1832-1852}, 412:1832--1852,
  2011.

\bibitem{chaudhuri2017pac}
Arghya~Roy Chaudhuri and Shivaram Kalyanakrishnan.
\newblock Pac identification of a bandit arm relative to a reward quantile.
\newblock In {\em AAAI}, pages 1777--1783, 2017.

\bibitem{aziz2018pure}
Maryam Aziz, Jesse Anderton, Emilie Kaufmann, and Javed Aslam.
\newblock Pure exploration in infinitely-armed bandit models with
  fixed-confidence.
\newblock In {\em ALT 2018-Algorithmic Learning Theory}, 2018.

\bibitem{EvenDaral06}
E.~Even-Dar, S.~Mannor, and Y.~Mansour.
\newblock {Action Elimination and Stopping Conditions for the Multi-Armed
  Bandit and Reinforcement Learning Problems}.
\newblock {\em Journal of Machine Learning Research}, 7:1079--1105, 2006.

\bibitem{NIPS2012_4640}
Victor Gabillon, Mohammad Ghavamzadeh, and Alessandro Lazaric.
\newblock Best arm identification: A unified approach to fixed budget and fixed
  confidence.
\newblock In F.~Pereira, C.J.C. Burges, L.~Bottou, and K.Q. Weinberger,
  editors, {\em Advances in Neural Information Processing Systems 25}, pages
  3212--3220. Curran Associates, Inc., 2012.

\bibitem{COLT13}
E.~Kaufmann and S.~Kalyanakrishnan.
\newblock {Information complexity in bandit subset selection}.
\newblock In {\em {Proceeding of the 26th Conference On Learning Theory.}},
  2013.

\bibitem{icml2013_karnin13}
Zohar Karnin, Tomer Koren, and Oren Somekh.
\newblock Almost optimal exploration in multi-armed bandits.
\newblock In Sanjoy Dasgupta and David Mcallester, editors, {\em Proceedings of
  the 30th International Conference on Machine Learning (ICML-13)}, pages
  1238--1246. JMLR Workshop and Conference Proceedings, May 2013.

\bibitem{simchowitz2017simulator}
Max Simchowitz, Kevin Jamieson, and Benjamin Recht.
\newblock The simulator: Understanding adaptive sampling in the
  moderate-confidence regime.
\newblock In {\em Conference on Learning Theory}, pages 1794--1834, 2017.

\bibitem{li2017hyperband}
Lisha Li, Kevin~G Jamieson, Giulia DeSalvo, Afshin Rostamizadeh, and Ameet
  Talwalkar.
\newblock Hyperband: A novel bandit-based approach to hyperparameter
  optimization.
\newblock {\em Journal of Machine Learning Research}, 18:185--1, 2017.

\bibitem{berry1997}
Donald~A. Berry, Robert~W. Chen, Alan Zame, David~C. Heath, and Larry~A. Shepp.
\newblock Bandit problems with infinitely many arms.
\newblock {\em Ann. Statist.}, 25(5):2103--2116, 10 1997.

\bibitem{wang2008}
Yizao Wang, Jean yves Audibert, and R\'{e}mi Munos.
\newblock Algorithms for infinitely many-armed bandits.
\newblock In D.~Koller, D.~Schuurmans, Y.~Bengio, and L.~Bottou, editors, {\em
  Advances in Neural Information Processing Systems 21}, pages 1729--1736.
  Curran Associates, Inc., 2009.

\bibitem{CarpentierV15}
Alexandra Carpentier and Michal Valko.
\newblock Simple regret for infinitely many armed bandits.
\newblock {\em CoRR}, abs/1505.04627, 2015.

\bibitem{chandrasekaran2014finding}
Karthekeyan Chandrasekaran and Richard Karp.
\newblock Finding a most biased coin with fewest flips.
\newblock In {\em Conference on Learning Theory}, pages 394--407, 2014.

\bibitem{jamieson2016power}
Kevin~G Jamieson, Daniel Haas, and Benjamin Recht.
\newblock The power of adaptivity in identifying statistical alternatives.
\newblock In {\em Advances in Neural Information Processing Systems}, pages
  775--783, 2016.

\bibitem{locatelli2016optimal}
Andrea Locatelli, Maurilio Gutzeit, and Alexandra Carpentier.
\newblock An optimal algorithm for the thresholding bandit problem.
\newblock In {\em International Conference on Machine Learning}, pages
  1690--1698, 2016.

\bibitem{mukherjee2017thresholding}
Subhojyoti Mukherjee, Naveen~Kolar Purushothama, Nandan Sudarsanam, and
  Balaraman Ravindran.
\newblock Thresholding bandits with augmented ucb.
\newblock In {\em Proceedings of the 26th International Joint Conference on
  Artificial Intelligence}, pages 2515--2521. AAAI Press, 2017.

\bibitem{DBLP:conf/nips/ChenLKLC14}
Shouyuan Chen, Tian Lin, Irwin King, Michael~R. Lyu, and Wei Chen.
\newblock Combinatorial pure exploration of multi-armed bandits.
\newblock In {\em Advances in Neural Information Processing Systems 27: Annual
  Conference on Neural Information Processing Systems 2014, December 8-13 2014,
  Montreal, Quebec, Canada}, pages 379--387, 2014.

\bibitem{chen2017nearly}
Lijie Chen, Jian Li, and Mingda Qiao.
\newblock Nearly instance optimal sample complexity bounds for top-k arm
  selection.
\newblock In {\em Artificial Intelligence and Statistics}, pages 101--110,
  2017.

\bibitem{jamieson2014lil}
Kevin Jamieson, Matthew Malloy, Robert Nowak, and S{\'e}bastien Bubeck.
\newblock lil’ucb: An optimal exploration algorithm for multi-armed bandits.
\newblock In {\em Conference on Learning Theory}, pages 423--439, 2014.

\end{thebibliography}
\clearpage

\end{document}